\documentclass{mylatex2e}

\usepackage{cite}
\usepackage{graphicx}
\usepackage{amsfonts}
\usepackage{stmaryrd}
\usepackage{mathdots}
\usepackage{epsf}
\usepackage{theorem}
\usepackage{amsmath}
\usepackage{amssymb}
\usepackage{textcomp}
\usepackage{bm}

\theoremstyle{plain}
\newtheorem{theorem}{Theorem}

\newenvironment{proof}{{\bf Proof.}}{\hspace{\stretch{1}}$\square$\\}

\begin{document}

\renewcommand{\baselinestretch}{1.4}
\small\normalsize

\runauthor{Z.-H. Zhou et al.}
\begin{frontmatter}

\title{Multi-Instance Multi-Label Learning}\vspace{-2mm}

\renewcommand{\thefootnote}{\fnsymbol{footnote}}

\large{Zhi-Hua Zhou\footnote{Corresponding author. E-mail: zhouzh@lamda.nju.edu.cn},
Min-Ling Zhang, Sheng-Jun Huang, Yu-Feng Li}\vspace{+5mm}

{\small \textit{National Key Laboratory for Novel Software Technology,\\Nanjing
University, Nanjing 210046, China
}}\vspace{+5mm}

\renewcommand{\thefootnote}{\arabic{footnote}}
\setcounter{footnote}{0}

\begin{abstract}
In this paper, we propose the MIML (\textit{Multi-Instance Multi-Label learning})
framework where an example is described by multiple instances and associated with
multiple class labels. Compared to traditional learning frameworks, the MIML
framework is more convenient and natural for representing complicated objects which
have multiple semantic meanings. To learn from MIML examples, we propose the
\textsc{MimlBoost} and \textsc{MimlSvm} algorithms based on a simple degeneration
strategy, and experiments show that solving problems involving complicated objects with
multiple semantic meanings in the MIML framework can lead to good performance.
Considering that the degeneration process may lose information, we propose the
\textsc{D-MimlSvm} algorithm which tackles MIML problems directly in a regularization
framework. Moreover, we show that even when we do not have access to the real objects
and thus cannot capture more information from real objects by using the MIML
representation, MIML is still useful. We propose the \textsc{InsDif} and
\textsc{SubCod} algorithms. \textsc{InsDif} works by transforming single-instances into
the MIML representation for learning, while \textsc{SubCod} works by transforming
single-label examples into the MIML representation for learning. Experiments show that
in some tasks they are able to achieve better performance than learning the
single-instances or single-label examples directly.
\end{abstract}


\end{frontmatter}


\section{Introduction}\vspace{-4mm}

In \textit{traditional supervised learning}, an object is represented by an instance,
i.e., a feature vector, and associated with a class label. Formally, let $\mathcal{X}$
denote the instance space (or feature space) and $\mathcal{Y}$ the set of class labels.
The task is to learn a function $f: \mathcal{X} \to \mathcal{Y}$ from a given data set
$\{(\bm{x}_1, y_1), (\bm{x}_2, y_2), \cdots, (\bm{x}_m, y_m)\}$, where $\bm{x}_i \in
\mathcal{X}$ is an instance and $y_i \in \mathcal{Y}$ is the known label of $\bm{x}_i$.
Although this formalization is prevailing and successful, there are many real-world
problems which do not fit in this framework well. In particular, each object in this
framework belongs to only one concept and therefore the corresponding instance is
associated with a single class label. However, many real-world objects are complicated,
which may belong to multiple concepts simultaneously. For example, an image can belong
to several classes simultaneously, e.g., \textit{grasslands}, \textit{lions},
\textit{Africa}, etc.; a text document can be classified to several categories if it is
viewed from different aspects, e.g., \textit{scientific novel}, \textit{Jules Verne's
writing} or even \textit{books on traveling}; a web page can be recognized as
\textit{news page}, \textit{sports page}, \textit{soccer page}, etc. In a specific real
task, maybe only one of the multiple concepts is the right semantic meaning. For
example, in image retrieval when a user is interested in an image with lions, s/he may
be only interested in the concept \textit{lions} instead of the other concepts
\textit{grasslands} and \textit{Africa} associated with that image. The difficulty here
is caused by those objects that involve multiple concepts.
 To choose the right semantic meaning for such objects for a specific scenario is the fundamental difficulty of many tasks. In contrast to starting from a large universe of all possible concepts involved in the task, it may be helpful to get the subset of concepts associated with the concerned object at first, and then make a choice in the small subset later. However, getting the subset of concepts, that is, assigning proper class labels to such objects, is still a challenging task.

We notice that as an alternative to representing an object by a single instance, in
many cases it is possible to represent a complicated object using a set of instances.
For example, multiple patches can be extracted from an image where each patch is
described by an instance, and thus the image can be represented by a set of instances;
multiple sections can be extracted from a document where each section is described by
an instance, and thus the document can be represented by a set of instances; multiple
links can be extracted from a web page where each link is described by an instance, and
thus the web page can be represented by a set of instances. Using multiple instances to
represent those complicated objects may be helpful because some inherent patterns which
are closely related to some labels may become explicit and clearer.
In this paper, we propose the MIML (\textit{Multi-Instance Multi-Label learning})
framework, where an example is described by multiple instances and associated with
multiple class labels.

Compared to traditional learning frameworks, the MIML framework is more convenient
and natural for representing complicated objects. To exploit the advantages of the MIML
representation, new learning algorithms are needed. We propose the \textsc{MimlBoost}
algorithm and the \textsc{MimlSvm} algorithm based on a simple degeneration strategy,
and experiments show that solving problems involving complicated objects with multiple
semantic meanings under the MIML framework can lead to good performance. Considering
that the degeneration process may lose information, we also propose the
\textsc{D-MimlSvm} (i.e., Direct \textsc{MimlSvm}) algorithm which tackles MIML
problems directly in a regularization framework. Experiments show that this ``direct''
algorithm outperforms the ``indirect'' \textsc{MimlSvm} algorithm.

In some practical tasks we do not have access to the real objects themselves such as
the real images and the real web pages; instead, we are given observational data where
each real object has already been represented by a single instance. Thus, in such cases
we cannot capture more information from the real objects using the MIML representation.
Even in this situation, however, MIML is still useful. We propose the \textsc{InsDif}
(i.e., INStance DIFferentiation) algorithm which transforms single-instances into MIML
examples for learning. This algorithm is able to achieve a better performance than
learning the single-instances directly in some tasks. This is not strange because for
an object associated with multiple class labels, if it is described by only a single
instance, the information corresponding to these labels are mixed and thus difficult
for learning; if we can transform the single-instance into a set of instances in some
proper ways, the mixed information might be detached to some extent and thus less
difficult for learning.

MIML can also be helpful for learning single-label objects. We propose the
\textsc{SubCod} (i.e., SUB-COncept Discovery) algorithm which works by discovering
sub-concepts of the target concept at first and then transforming the data into MIML
examples for learning. This algorithm is able to achieve a better performance than
learning the single-label examples directly in some tasks. This is also not strange
because for a label corresponding to a high-level complicated concept, it may be quite
difficult to learn this concept directly since many different lower-level concepts are
mixed; if we can transform the single-label into a set of labels corresponding to some
sub-concepts, which are relatively clearer and easier for learning, we can learn these
labels at first and then derive the high-level complicated label based on them with a
less difficulty.

The rest of this paper is organized as follows. In Section~\ref{sec:survey}, we review
some related work. In Section~\ref{sec:MIML}, we propose the MIML framework. In
Section~\ref{sec:MIMLalgos} we propose the \textsc{MimlBoost} and \textsc{MimlSvm}
algorithms, and apply them to tasks where the objects are represented as MIML examples.
In Section~\ref{sec:Dmiml} we present the \textsc{D-MimlSvm} algorithm and compare it
with the ``indirect'' \textsc{MimlSvm} algorithm. In Sections~\ref{sec:MLLtoMIML} and
\ref{sec:MILtoMIML}, we study the usefulness of MIML when we do not have access to real
objects. Concretely, in Section~\ref{sec:MLLtoMIML}, we propose the \textsc{InsDif}
algorithm and show that using MIML can be better than learning single-instances
directly; in Section~\ref{sec:MILtoMIML} we propose the \textsc{SubCod} algorithm and
show that using MIML can be better than learning single-label examples directly.
Finally, we conclude the paper in Section~\ref{sec:conclusion}.

\section{Related Work}\label{sec:survey}\vspace{-4mm}

Much work has been devoted to the learning of multi-label examples under the umbrella
of \textit{multi-label learning}. Note that multi-label learning studies the problem
where a real-world object described by one instance is associated with a number of
class labels\footnote{Most work on multi-label learning assumes that an instance can be
associated with multiple valid labels, but there is also some work assuming that only
one of the labels among those associated with an instance is correct
\cite{Jin:Ghahramani2003}.}, which is different from multi-class learning or multi-task
learning \cite{Evgeniou:Micchelli:Pontil2005}. In multi-class learning each object is
only associated with a single label; while in multi-task learning different tasks may
involve different domains and different data sets. Actually, traditional two-class and
multi-class problems can both be cast into multi-label problems by restricting that
each instance has only one label. The generality of multi-label problems, however,
inevitably makes it more difficult to address.

One famous approach to solving multi-label problems is Schapire and Singer's
\textsc{AdaBoost.MH} \cite{Schapire:Singer2000}, which is an extension of
\textsc{AdaBoost} and is the core of a successful multi-label learning system
\textsc{BoosTexter} \cite{Schapire:Singer2000}. This approach maintains a set of
weights over both training examples and their labels in the training phase, where
training examples and their corresponding labels that are hard (easy) to predict get
incrementally higher (lower) weights. Later, De Comit\'{e} et al.
\cite{DeComite:Gilleron:Tommasi2003} used alternating decision trees
\cite{Freund:Mason1999} which are more powerful than decision stumps used in
\textsc{BoosTexter} to handle multi-label data and thus obtained the
\textsc{AdtBoost.MH} algorithm. Probabilistic generative models have been found useful
in multi-label learning. McCallum \cite{McCallum1999} proposed a Bayesian approach for
multi-label document classification, where a mixture probabilistic model (one mixture
component per category) is assumed to generate each document and an EM algorithm is
employed to learn the mixture weights and the word distributions in each mixture
component. Ueda and Saito \cite{Ueda:Saito2003} presented another generative approach,
which assumes that the multi-label text has a mixture of characteristic words appearing
in single-label text belonging to each of the multi-labels. It is noteworthy that the
generative models used in \cite{McCallum1999} and \cite{Ueda:Saito2003} are both based
on learning text frequencies in documents, and are thus specific to text applications.

Many other multi-label learning algorithms have been developed, such as decision trees,
neural networks, $k$-nearest neighbor classifiers, support vector machines, etc. Clare
and King \cite{Clare:King2001} developed a multi-label version of C4.5 decision trees
through modifying the definition of entropy. Zhang and Zhou \cite{Zhang:Zhou2006tkde}
presented multi-label neural network \textsc{Bp-Mll}, which is derived from the
Backpropagation algorithm by employing an error function to capture the fact that the
labels belonging to an instance should be ranked higher than those not belonging to
that instance. Zhang and Zhou \cite{Zhang:Zhou2007pr} also proposed the
\textsc{Ml-$k$nn} algorithm, which identifies the $k$ nearest neighbors of the
concerned instance and then assigns labels according to the maximum a posteriori
principle. Elisseeff and Weston \cite{Elisseeff:Weston2002} proposed the {\sc RankSvm}
algorithm for multi-label learning by defining a specific cost function and the
corresponding margin for multi-label models. Other kinds of multi-label \textsc{Svm}s
have been developed by Boutell et al. \cite{Boutell:Luo:Shen:Brown2004} and Godbole and
Sarawagi \cite{Godbole:Sarawagi2004}. In particular, by hierarchically approximating
the Bayes optimal classifier for the H-loss, Cesa-Bianchi et al.
\cite{CesaBianchi:Gentile:Zaniboni2006} proposed an algorithm which outperforms simple
hierarchical \textsc{Svm}s. Recently, non-negative matrix factorization has also been
applied to multi-label learning \cite{Liu:Jin:Yang2006}, and multi-label dimensionality
reduction methods have been developed \cite{Yu:Yu:Tresp2005,Zhang:Zhou2010tkdd}.

Roughly speaking, earlier approaches to multi-label learning attempt to divide
multi-label learning to a number of two-class classification problems
\cite{Joachims1998,Yang1999} or transform it into a label ranking problem
\cite{Schapire:Singer2000,Elisseeff:Weston2002}, while some later approaches try to
exploit the correlation between the labels
\cite{Ueda:Saito2003,Liu:Jin:Yang2006,Zhang:Zhou2010tkdd}.

Most studies on multi-label learning focus on text categorization
\cite{Schapire:Singer2000,McCallum1999,Ueda:Saito2003,DeComite:Gilleron:Tommasi2003,Godbole:Sarawagi2004,Kazawa:Izumitani:Taira:Maeda2005,Yu:Yu:Tresp2005},
and several studies aim to improve the performance of text categorization systems by
exploiting additional information given by the hierarchical structure of classes
\cite{Cai:Hofmann2004,Rousu:Saunders:Szedmak:ShaweTaylor2005,CesaBianchi:Gentile:Zaniboni2006}
or unlabeled data \cite{Liu:Jin:Yang2006}. In addition to text categorization,
multi-label learning has also been found useful in many other tasks such as scene
classification \cite{Boutell:Luo:Shen:Brown2004}, image and video annotation
\cite{Kang:Jin:Sukthankar2006,Qi:Hua:Rui:Tang:Mei:Zhang2007}, bioinformatics
\cite{Clare:King2001,Elisseeff:Weston2002,Brinker:Furnkranz:Hullermeler2006,Barutcuoglu:Schapire:Troyanskaya2006,Brinker:Hullermeler2007},
and even association rule mining
\cite{Thabtah:Cowling:Peng2004,Rak:Kurgan:Reformat2005}.

There is a lot of research on {\it multi-instance learning}, which studies the problem
where a real-world object described by a number of instances is associated with a
single class label. Here the training set is composed of many {\it bags} each
containing multiple instances; a bag is labeled positively if it contains at least one
positive instance and negatively otherwise. The goal is to label unseen bags correctly.
Note that although the training bags are labeled, the labels of their instances are
unknown. This learning framework was formalized by Dietterich et al.
\cite{Dietterich:Lathrop:Lozano1997} when they were investigating drug activity
prediction.

Long and Tan \cite{Long:Tan1998} studied the {\sc Pac}-learnability of multi-instance
learning and showed that if the instances in the bags are independently drawn from
product distribution, the {\sc Apr} (Axis-Parallel Rectangle) proposed by Dietterich et
al. \cite{Dietterich:Lathrop:Lozano1997} is {\sc Pac}-learnable. Auer et al.
\cite{Auer:Long:Srinivasan1998} showed that if the instances in the bags are not
independent then {\sc Apr} learning under the multi-instance learning framework is
NP-hard. Moreover, they presented a theoretical algorithm that does not require product
distribution, which was transformed into a practical algorithm named {\sc Multinst}
\cite{Auer1997}. Blum and Kalai \cite{Blum:Kalai1998} described a reduction from {\sc
Pac}-learning under the multi-instance learning framework to {\sc Pac}-learning with
one-sided random classification noise. They also presented an algorithm with smaller
sample complexity than that of the algorithm of Auer et al.
\cite{Auer:Long:Srinivasan1998}.

Many multi-instance learning algorithms have been developed during the past decade. To
name a few, {\sc Diverse Density} \cite{Maron:Lozano1998} and {\sc Em-dd}
\cite{Zhang:Goldman2002}, $k$-nearest neighbor algorithms {\sc Citation-$k$nn} and {\sc
Bayesian-$k$nn} \cite{Wang:Zucker2000}, decision tree algorithms {\sc Relic}
\cite{Ruffo2000} and \textsc{Miti} \cite{Blockeel:Page:Srinivasan2005}, neural network
algorithms {\sc Bp-mip} and extensions \cite{Zhou:Zhang2002,Zhang:Zhou2004} and
\textsc{Rbf-mip} \cite{Zhang:Zhou2006}, rule learning algorithm {\sc Ripper-mi}
\cite{Chevaleyre:Zucker2001}, support vector machines and kernel methods
\textsc{mi-Svm} and \textsc{Mi-Svm} \cite{Andrews:Tsochantaridis:Hofmann2003},
\textsc{Dd-Svm} \cite{Chen:Wang2004}, \textsc{MissSvm} \cite{Zhou:Xu2007}, {\sc
Mi-Kernel} \cite{Gartner:Flach:Kowalczyk2002}, \textsc{Bag-Instance Kernel}
\cite{Cheung:Kwok2006}, \textsc{Marginalized Mi-Kernel} \cite{Kwok:Cheung2007} and
convex-hull method \textsc{Ch-Fd} \cite{Fung:Dundar:Krishnappuram2007nips06}, ensemble
algorithms {\sc Mi-Ensemble} \cite{Zhou:Zhang2003}, \textsc{MiBoosting}
\cite{Xu:Frank2004} and \textsc{MilBoosting} \cite{Auer:Ortner2004}, logistic
regression algorithm \textsc{Mi-lr} \cite{Ray:Craven2005}, etc. Actually almost all
popular machine learning algorithms have their multi-instance versions. Most
algorithms attempt to adapt single-instance supervised learning algorithms to the
multi-instance representation, by shifting their focus from discrimination
on instances to discrimination on bags \cite{Zhou:Zhang2003}. Recently
there is some proposal on adapting the multi-instance representation to single-instance
algorithms by representation transformation \cite{Zhou:Zhang2007}.

It is worth mentioning that standard multi-instance learning
\cite{Dietterich:Lathrop:Lozano1997} assumes that if a bag contains a positive instance
then the bag is positive; this implies that there exists a \textit{key instance} in a
positive bag. Many algorithms were designed based on this assumption. For example, the
point with maximal diverse density identified by the {\sc Diverse Density} algorithm
\cite{Maron:Lozano1998} actually corresponds to a key instance; many \textsc{Svm}
algorithms defined the margin of a positive bag by the margin of its \textit{most}
positive instance \cite{Andrews:Tsochantaridis:Hofmann2003,Cheung:Kwok2006}. As the
research of multi-instance learning goes on, however, some other assumptions have been
introduced \cite{Foulds:Frank2010}. For example, in contrast to assuming that there is
a key instance, some work has assumed that there is no key instance and every instance
contributes to the bag label \cite{Xu:Frank2004,Chen:Bi:Wang2006}. There is also an
argument that the instances in the bags should not be treated independently
\cite{Zhou:Xu2007}. All those assumptions have been put under the umbrella of
multi-instance learning, and generally, in tackling real tasks it is difficult to know
which assumption is the fittest. In other words, in different tasks multi-instance
learning algorithms based on different assumptions may have different superiorities.

In the early years of the research of multi-instance learning, most work considered
multi-instance classification with discrete-valued outputs. Later, multi-instance
regression with real-valued outputs was studied
\cite{Amar:Dooly:Goldman2001,Ray:Page2001}, and different versions of generalized
multi-instance learning have been defined
\cite{Weidmann:Frank:Pfahringer2003,Scott:Zhang:Brown2003}. The main difference between
standard multi-instance learning and generalized multi-instance learning is that in
standard multi-instance learning there is a single concept, and a bag is positive if it
has an instance satisfying this concept; while in generalized multi-instance learning
\cite{Weidmann:Frank:Pfahringer2003,Scott:Zhang:Brown2003} there are multiple concepts,
and a bag is positive only when all concepts are satisfied (i.e., the bag contains
instances from every concept). Recently, research on multi-instance clustering
\cite{Zhang:Zhou2009apin}, multi-instance semi-supervised learning
\cite{Rahmani:Goldman2006} and multi-instance active learning
\cite{Settles:Craven:Ray2008nips07} have also been reported.

Multi-instance learning has also attracted the attention of the {\sc Ilp} community. It
has been suggested that multi-instance problems could be regarded as a bias on
inductive logic programming, and the multi-instance paradigm could be the key between
the propositional and relational representations, being more expressive than the
former, and much easier to learn than the latter \cite{DeRaedt1998}. Alphonse and
Matwin \cite{Alphonse:Matwin2004} approximated a relational learning problem by a
multi-instance problem, fed the resulting data to feature selection techniques adapted
from propositional representations, and then transformed the filtered data back to
relational representation for a relational learner. Thus, the expressive power of
relational representation and the ease of feature selection on propositional
representation are gracefully combined. This work confirms that multi-instance learning
can really act as a bridge between propositional and relational learning.

Multi-instance learning techniques have already been applied to diverse applications
including image categorization \cite{Chen:Bi:Wang2006,Chen:Wang2004}, image retrieval
\cite{Yang:Lozano2000,Zhang:Yu:Goldman2002}, text categorization
\cite{Andrews:Tsochantaridis:Hofmann2003,Settles:Craven:Ray2008nips07}, web mining
\cite{Zhou:Jiang:Li2005}, spam detection \cite{Jorgensen:Zhou:Inge2008}, computer
security \cite{Ruffo2000}, face detection
\cite{Viola:Platt:Zhang2006,Zhang:Viola2008nips07}, computer-aided medical diagnosis
\cite{Fung:Dundar:Krishnappuram2007nips06}, etc.

\section{The MIML Framework}\label{sec:MIML}\vspace{-4mm}

Let $\mathcal{X}$ denote the instance space and $\mathcal{Y}$ the set of class labels.
Then, formally, the MIML task is defined as:

\begin{itemize}

\item \textbf{MIML} (multi-instance multi-label learning): To learn a function $f: 2^\mathcal{X} \to
2^\mathcal{Y}$ from a given data set $\{(X_1, Y_1), (X_2, Y_2), \cdots, (X_m, Y_m)\}$,
where $X_i \subseteq \mathcal{X}$ is a set of instances $\{{\bm x_{i1}}, {\bm x_{i2}},
\cdots, {\bm x_{i,n_i}}\}$, ${\bm x_{ij}} \in \mathcal{X}$ $(j=1,2,\cdots,n_i)$, and
$Y_i \subseteq \mathcal{Y}$ is a set of labels $\{y_{i1}, y_{i2}, \cdots, y_{i,l_i}\}$,
$y_{ik} \in \mathcal{Y}$ $(k=1,2,\cdots,l_i)$. Here $n_i$ denotes the number of
instances in $X_i$ and $l_i$ the number of labels in $Y_i$.

\end{itemize}


\begin{figure}[!t]
\centering
\begin{minipage}[c]{2.8in}
\centering
\includegraphics[width = 2.5in]{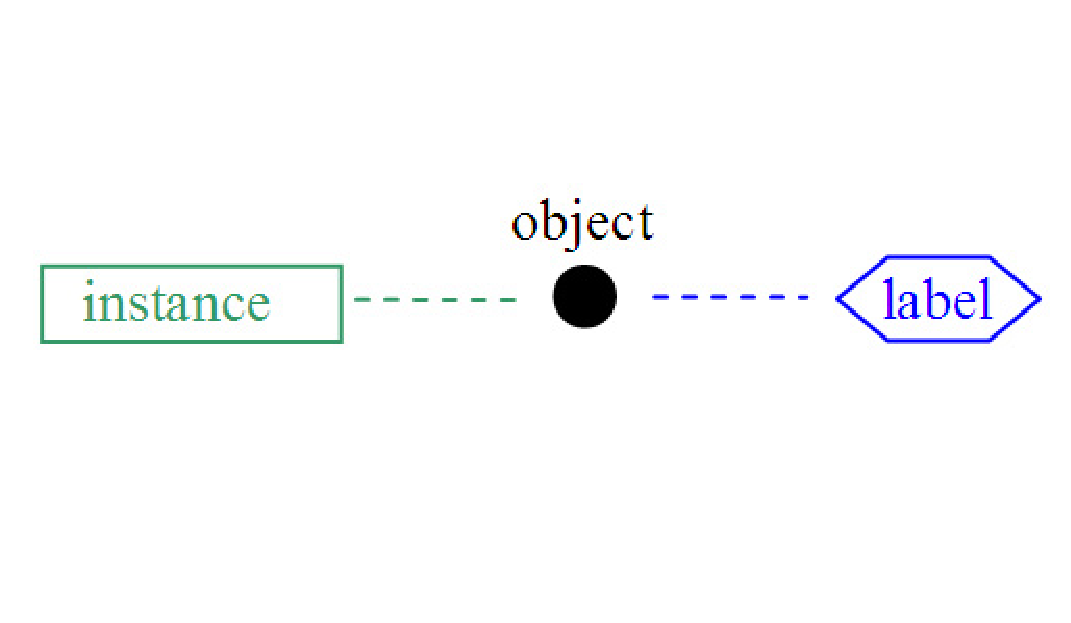}
\end{minipage}%
\begin{minipage}[c]{2.8in}
\centering
\includegraphics[width = 2.5in]{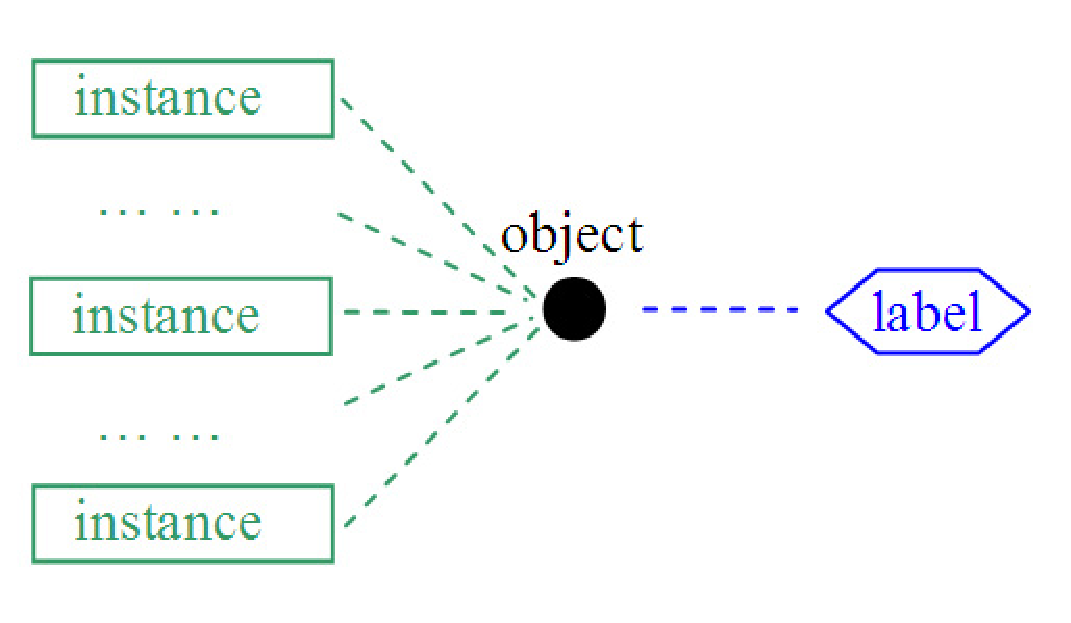}
\end{minipage}\\[+5pt]
\centering
\begin{minipage}[c]{2.8in}
\centering \mbox{\small (a) Traditional supervised learning}
\end{minipage}%
\begin{minipage}[c]{2.8in}
\centering \mbox{\small (b) Multi-instance learning}
\end{minipage}\\[+8pt]
\centering
\begin{minipage}[c]{2.8in}
\centering
\includegraphics[width = 2.5in]{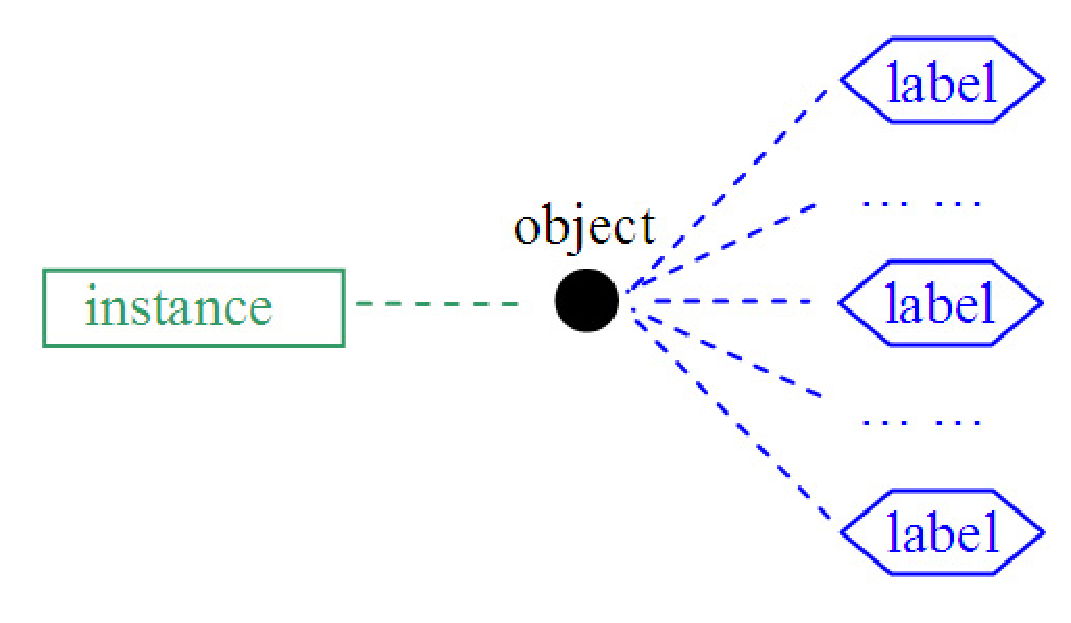}
\end{minipage}%
\begin{minipage}[c]{2.8in}
\centering
\includegraphics[width = 2.5in]{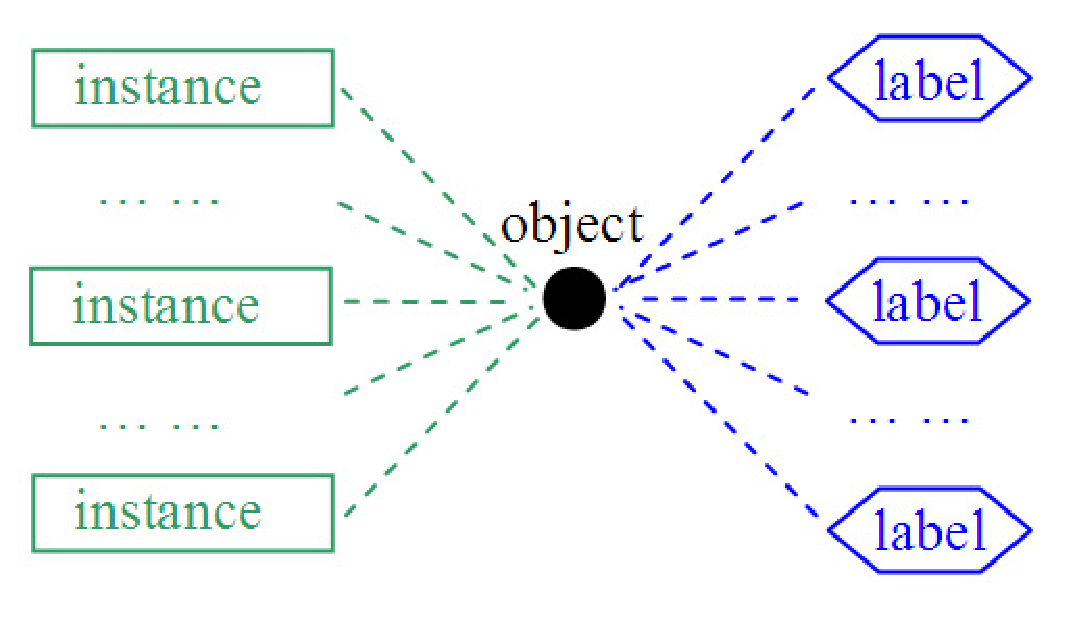}
\end{minipage}\\[+5pt]
\centering
\begin{minipage}[c]{2.8in}
\centering \mbox{\small (c) Multi-label learning}
\end{minipage}%
\begin{minipage}[c]{2.8in}
\centering \mbox{\small (d) Multi-instance multi-label learning}
\end{minipage}\\[+5pt]
\caption{Four different learning frameworks}\label{fig:miml}\bigskip
\end{figure}

It is interesting to compare MIML with the existing frameworks of traditional
supervised learning, multi-instance learning, and multi-label learning.

\begin{itemize}

\item {\bf Traditional supervised learning} (single-instance single-label learning): To learn a function $f:
\mathcal{X} \to \mathcal{Y}$ from a given data set $\{({\bm x_1}, y_1), ({\bm x_2},
y_2), \cdots, ({\bm x_m}, y_m)\}$, where ${\bm x_i} \in \mathcal{X}$ is an instance and
$y_i \in \mathcal{Y}$ is the known label of ${\bm x_i}$.

\item {\bf Multi-instance learning} (multi-instance single-label learning): To learn a function $f: 2^\mathcal{X}
\to \mathcal{Y}$ from a given data set $\{(X_1, y_1), (X_2, y_2), \cdots, (X_m,
y_m)\}$, where $X_i \subseteq \mathcal{X}$ is a set of instances $\{{\bm x_{i1}}, {\bm
x_{i2}}, \cdots, {\bm x_{i,n_i}}\}$, ${\bm x_{ij}} \in \mathcal{X}$
$(j=1,2,\cdots,n_i)$, and $y_i \in \mathcal{Y}$ is the label of
$X_i$.\footnote{According to notions used in multi-instance learning, $(X_i, y_i)$ is a
labeled {\it bag} while $X_i$ an unlabeled bag.} Here $n_i$ denotes the number of
instances in $X_i$.

\item {\bf Multi-label learning} (single-instance multi-label learning): To learn a function $f: \mathcal{X} \to
2^\mathcal{Y}$ from a given data set $\{({\bm x_1}, Y_1), ({\bm x_2}, Y_2), \cdots,
({\bm x_m}, Y_m)\}$, where ${\bm x_{i}} \in \mathcal{X}$ is an instance and $Y_i
\subseteq \mathcal{Y}$ is a set of labels $\{y_{i1}, y_{i2}, \cdots, y_{i,l_i}\}$,
$y_{ik} \in \mathcal{Y}$ $(k=1,2,\cdots,l_i)$. Here $l_i$ denotes the number of labels
in $Y_i$.

\end{itemize}

From Fig.~\ref{fig:miml} we can see the differences among these learning frameworks. In
fact, the {\it multi-} learning frameworks are resulted from the ambiguities in
representing real-world objects. Multi-instance learning studies the ambiguity in the
input space (or instance space), where an object has many alternative input
descriptions, i.e., instances; multi-label learning studies the ambiguity in the output
space (or label space), where an object has many alternative output descriptions, i.e.,
labels; while MIML considers the ambiguities in both the input and output spaces
simultaneously. In solving real-world problems, having a good representation is often
more important than having a strong learning algorithm, because a good representation
may capture more meaningful information and make the learning task easier to tackle.
Since many real objects are inherited with input ambiguity as well as output ambiguity,
MIML is more natural and convenient for tasks involving such objects.


\begin{figure}[!ht]
\centering
\begin{minipage}[c]{6in}
\centering
\includegraphics[width = 5in]{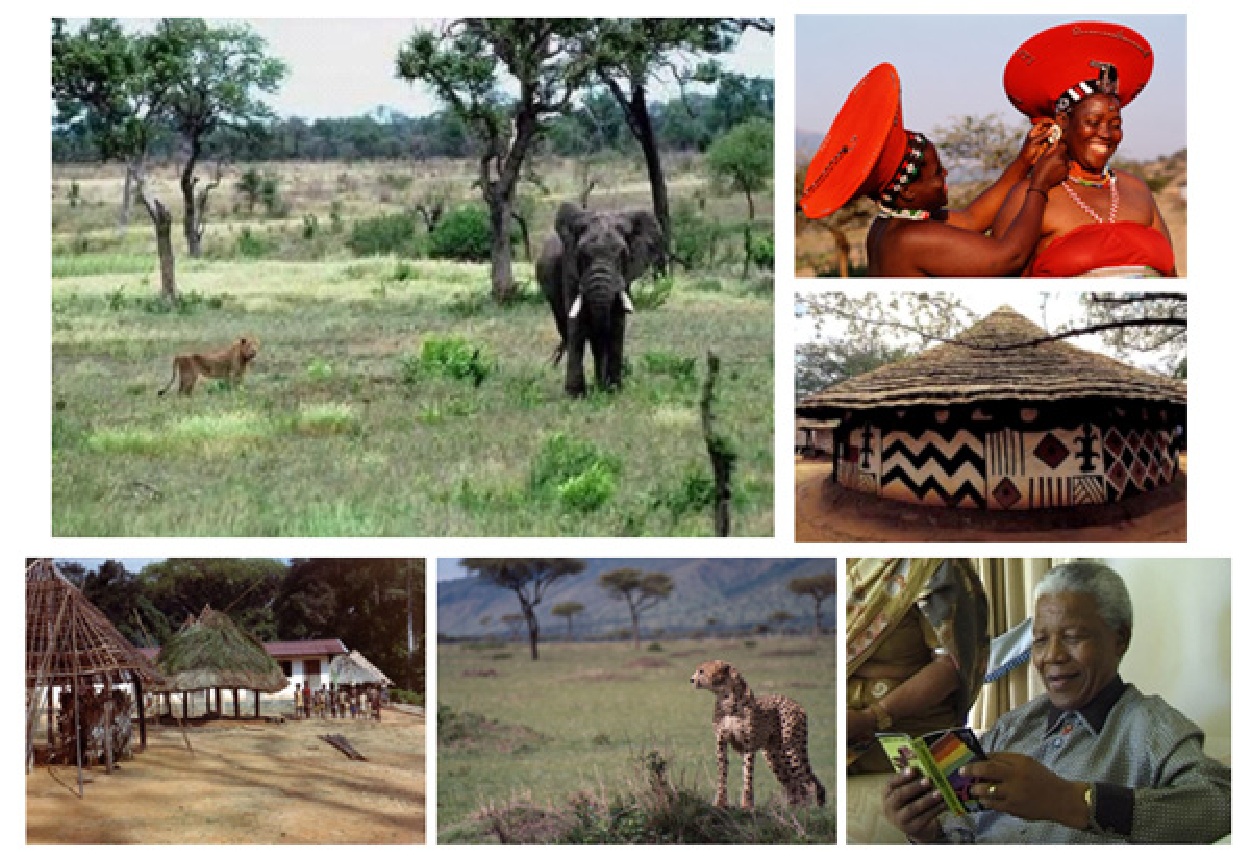}
\centering \mbox{\small (a) \textit{Africa} is a complicated high-level concept}
\end{minipage}\\[+5pt]
\begin{minipage}[c]{6in}
\centering
\includegraphics[width = 5in]{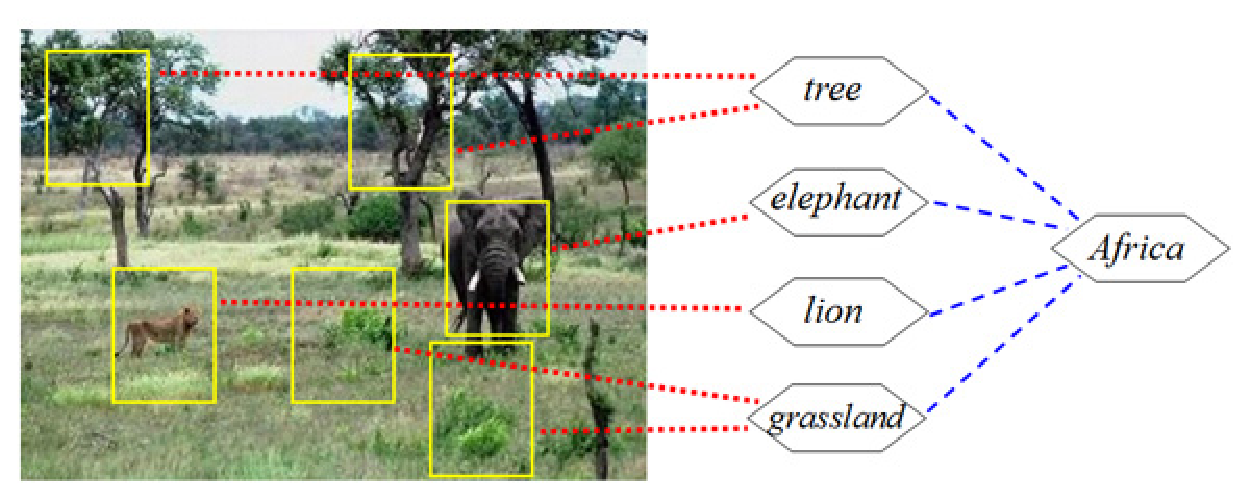}
\centering \mbox{\small (b) The concept \textit{Africa} may become easier to learn
through exploiting some sub-concepts}
\end{minipage}\\[+5pt]
\caption{MIML can be helpful in learning single-label examples involving complicated
high-level concepts}\label{fig:Africa}\bigskip
\end{figure}

It is worth mentioning that MIML is more reasonable than (single-instance) multi-label
learning in many cases. Suppose a multi-label object is described by one
instance but associated with $l$ number of class labels, namely label$_1$, label$_2$, $\ldots$, label$_l$.
If we represent the multi-label object using a set of $n$ instances, namely instance$_1$, instance$_2$, $\ldots$, instance$_n$, the underlying information in
a single instance may become easier to exploit, and for each label the number of
training instances can be significantly increased. So, transforming multi-label
examples to MIML examples for learning may be beneficial in some tasks, which will be
shown in Section~\ref{sec:MLLtoMIML}. Moreover, when representing the multi-label
object using a set of instances, the relation between the input patterns and the
semantic meanings may become more easily discoverable. Note that in some cases,
understanding why a particular object has a certain class label is even more important
than simply making an accurate prediction, while MIML offers a possibility for this
purpose. For example, under the MIML representation, we may discover that one object has label$_1$ because it
contains instance$_n$; it has label$_l$ because it contains instance$_i$; while the
occurrence of both instance$_1$ and instance$_i$ triggers label$_j$.

MIML can also be helpful for learning single-label examples involving complicated
high-level concepts. For example, as Fig.~\ref{fig:Africa}(a) shows, the concept
\textit{Africa} has a broad connotation and the images belonging to \textit{Africa}
have great variance, thus it is not easy to classify the top-left image in
Fig.~\ref{fig:Africa}(a) into the \textit{Africa} class correctly. However, if we can
exploit some low-level sub-concepts that are less ambiguous and easier to learn, such
as \textit{tree}, \textit{lions}, \textit{elephant} and \textit{grassland} shown in
Fig.~\ref{fig:Africa}(b), it is possible to induce the concept \textit{Africa} much
easier than learning the concept \textit{Africa} directly. The usefulness of MIML in
this process will be shown in Section~\ref{sec:MILtoMIML}.

\section{Solving MIML Problems by Degeneration}\label{sec:MIMLalgos}\vspace{-4mm}

It is evident that traditional supervised learning is a degenerated version of
multi-instance learning as well as a degenerated version of multi-label learning, while
traditional supervised learning, multi-instance learning and multi-label learning are
all degenerated versions of MIML. So, a simple idea to tackle MIML is to identify its
equivalence in the traditional supervised learning framework, using multi-instance
learning or multi-label learning as the bridge, as shown in Fig.~\ref{fig:twosolution}.

\begin{figure}[!t]
\centering
\includegraphics[width=10cm]{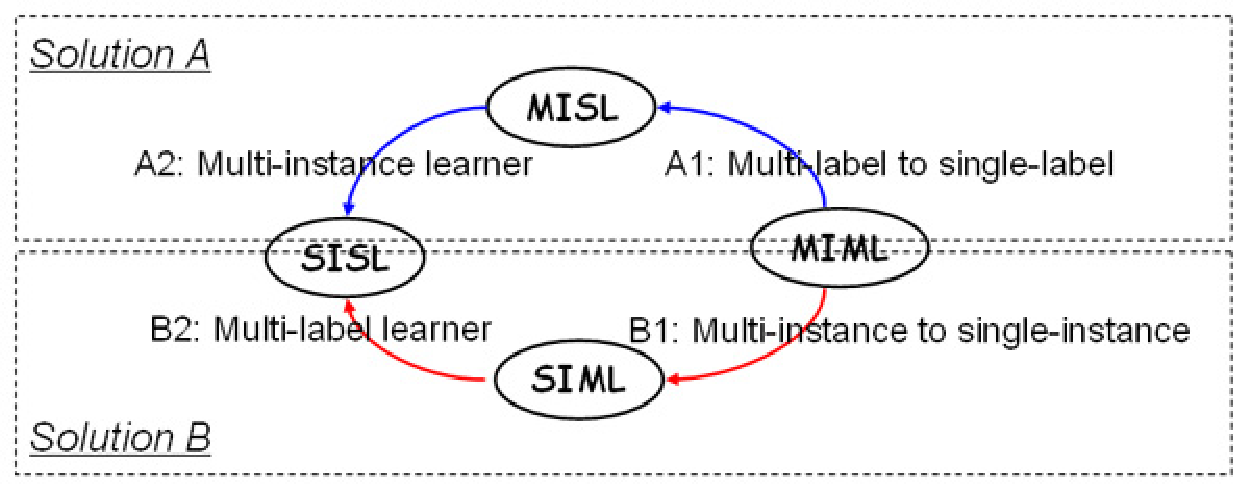}
\caption{The two general degeneration solutions. }\bigskip\label{fig:twosolution}
\end{figure}

\begin{itemize}

\item {\bf Solution A}: Using multi-instance learning as the bridge:

The MIML learning task, i.e., to learn a function $f: 2^\mathcal{X} \to 2^\mathcal{Y}$,
can be transformed into a multi-instance learning task, i.e., to learn a function
$f_{MIL}: 2^\mathcal{X} \times \mathcal{Y} \to \{-1, +1\}$. For any $y \in
\mathcal{Y}$, $f_{MIL}(X_i, y) = +1$ if $y \in Y_i$ and $-1$ otherwise. The proper
labels for a new example $X^*$ can be determined according to $Y^* = \{ y |
sign[f_{MIL}(X^*, y)] = +1\}$. This multi-instance learning task can be further
transformed into a traditional supervised learning task, i.e., to learn a function
$f_{SISL}: \mathcal{X} \times \mathcal{Y} \to \{-1, +1\}$, under a constraint
specifying how to derive $f_{MIL}(X_i, y)$ from $f_{SISL}({\bm x_{ij}}, y)$ $(j =
1,2,\cdots,n_i)$. For any $y \in \mathcal{Y}$, $f_{SISL}({\bm x_{ij}}, y) = +1$ if $y
\in Y_i$ and $-1$ otherwise. Here the constraint can be $f_{MIL}(X_i, y) = sign
[\sum_{j=1}^{n_i} f_{SISL}({\bm x_{ij}}, y)]$ which has been used by Xu and Frank
\cite{Xu:Frank2004} in transforming multi-instance learning tasks into traditional
supervised learning tasks. Note that other kinds of constraint
can also be used here.\\

\item {\bf Solution B}: Using multi-label learning as the bridge:

The MIML learning task, i.e., to learn a function $f: 2^\mathcal{X} \to 2^\mathcal{Y}$,
can be transformed into a multi-label learning task, i.e., to learn a function
$f_{MLL}: \mathcal{Z} \to 2^\mathcal{Y}$. For any ${\bm z}_i \in \mathcal{Z}$,
$f_{MLL}({\bm z}_i) = f_{MIML}(X_i)$ if ${\bm z}_i = \phi(X_i)$, $\phi: 2^\mathcal{X}
\to \mathcal{Z}$. The proper labels for a new example $X^*$ can be determined according
to $Y^* = f_{MLL}(\phi(X^*))$. This multi-label learning task can be further
transformed into a traditional supervised learning task, i.e., to learn a function
$f_{SISL}: \mathcal{Z} \times \mathcal{Y} \to \{-1, +1\}$. For any $y \in \mathcal{Y}$,
$f_{SISL}({\bm z}_i, y) = +1$ if $y \in Y_i$ and $-1$ otherwise. That is, $f_{MLL}({\bm
z}_i) = \{ y | f_{SISL}({\bm z}_i, y) = +1\}$. Here the mapping $\phi$ can be
implemented with {\it constructive clustering} which was proposed by Zhou and Zhang
\cite{Zhou:Zhang2007} in transforming multi-instance bags into traditional
single-instances. Note that other kinds of mappings can also be used here.

\end{itemize}

In the rest of this section we will propose two MIML algorithms, \textsc{MimlBoost} and
\textsc{MimlSvm}. \textsc{MimlBoost} is an illustration of Solution A, which uses
\textit{category-wise decomposition} for the A1 step in Fig.~\ref{fig:twosolution} and
\textsc{MiBoosting} for A2; \textsc{MimlSvm} is an illustration of Solution B, which
uses \textit{clustering-based representation transformation} for the B1 step and
\textsc{MlSvm} for B2. Other MIML algorithms can be developed by taking alternative
options. Both \textsc{MimlBoost} and \textsc{MimlSvm} are quite simple. We will see that for dealing with complicated objects with multiple
semantic meanings, good performance can be obtained under the MIML framework even by
using such simple algorithms. This demonstrates that the MIML framework is very
promising, and we expect better performance can be achieved in the future if researchers
put forward more powerful MIML algorithms.

\

\subsection{\textsc{MimlBoost}}\vspace{-5mm}

Now we propose the {\sc MimlBoost} algorithm according to the first solution mentioned
above, that is, identifying the equivalence in the traditional supervised learning
framework using multi-instance learning as the bridge. Note that this strategy can also
be used to derive other kinds of MIML algorithms.

Given any set $\Omega$, let $|\Omega|$ denote its size, i.e., the number of elements in
$\Omega$; given any predicate $\pi$, let $[\kern-0.15em[ \pi ]\kern-0.15em]$ be 1 if
$\pi$ holds and 0 otherwise; given $(X_i, Y_i)$, for any $y \in \mathcal{Y}$, let
$\Psi(X_i, y) = +1$ if $y \in Y_i$ and $-1$ otherwise, where $\Psi$ is a function
$\Psi: 2^\mathcal{X} \times \mathcal{Y} \to \{-1, +1\}$ which judges whether a label
$y$ is a proper label of $X_i$ or not. The basic assumption of \textsc{MimlBoost} is
that the labels are independent so that the MIML task can be decomposed into a series
of multi-instance learning tasks to solve, by treating each label as a task. The
pseudo-code of {\sc MimlBoost} is summarized in Appendix A (Table~\ref{table:mimlboost}).

In the first step of {\sc MimlBoost}, each MIML example $(X_u,Y_u)$ $(u =
1,2,\cdots,m)$ is transformed into a set of $|\mathcal{Y}|$ number of multi-instance
bags, i.e., $\{ [(X_u, y_1), \Psi(X_u, y_1)],$ $[(X_u, y_2), \Psi(X_u, y_2)],$
$\cdots,$ $[(X_u, y_{|\mathcal{Y}|}),$ $\Psi(X_u, y_{|\mathcal{Y}|})]\}$. Note that
$[(X_u, y_v), \Psi(X_u, y_v)]$ $(v = 1,2,\cdots,|\mathcal{Y}|)$ is a labeled
multi-instance bag where $(X_u, y_v)$ is a bag containing $n_u$ number of instances,
i.e., $\{({\bm x_{u1}}, y_v),$ $({\bm x_{u2}}, y_v),$ $\cdots,$ $({\bm x_{u,n_u}},
y_v)\}$, and $\Psi(X_u, y_v) \in \{-1, +1\}$ is the label of this bag.

Thus, the original MIML data set is transformed into a multi-instance data set
containing $m \times |\mathcal{Y}|$ number of bags. We order them as $[(X_1, y_1),
\Psi(X_1, y_1)],$ $\cdots,$ $[(X_1, y_{|\mathcal{Y}|}), \Psi(X_1, y_{|\mathcal{Y}|})],$
$[(X_2, y_1), \Psi(X_2, y_1)],$ $\cdots,$ $[(X_m, y_{|\mathcal{Y}|}),$ $\Psi(X_m,
y_{|\mathcal{Y}|})]$, and let $[(X^{(i)},y^{(i)}),\Psi(X^{(i)},y^{(i)})]$ denote the
$i$-th of these $m \times |\mathcal{Y}|$ number of bags which contains $n_i$ number of
instances.

Then, from the data set a multi-instance learning function $f_{MIL}$ can be learned,
which can accomplish the desired MIML function because $f_{MIML}(X^*) = \{y | sign$
$[f_{MIL}(X^*,y)] = +1\}$. In this paper, the {\sc MiBoosting} algorithm
\cite{Xu:Frank2004} is used to implement $f_{MIL}$. Note that by using {\sc
MiBoosting}, the \textsc{MimlBoost} algorithm assumes that all instances in a bag
contribute independently in an equal way to the label of that bag.

For convenience, let $(B,g)$ denote the bag $[(X,y),\Psi(X,y)]$, $B \in \mathcal{B}$,
$g \in \mathcal{G}$, and $E$ denotes the expectation. Then, here the goal is to learn a
function $\mathcal{F}(B)$ minimizing the bag-level exponential loss $E_\mathcal{B}
E_{\mathcal{G}|\mathcal{B}}[\exp(-g \mathcal{F}(B))]$, which ultimately estimates the
bag-level log-odds function $\frac{1}{2}\log\frac{Pr(g=1|B)}{Pr(g=-1|B)}$ on the
training set. In each boosting round, the aim is to expand $\mathcal{F}(B)$ into
$\mathcal{F}(B)+cf(B)$, i.e., adding a new weak classifier, so that the exponential
loss is minimized. Assuming that all instances in a bag contribute equally and
independently to the bag's label, $f(B) = {1 \over n_B} \sum_{j}h({\bm b}_j)$ can be
derived, where $h({\bm b}_j) \in \{-1, +1\}$ is the prediction of the instance-level
classifier $h(\cdot)$ for the $j$-th instance of the bag $B$, and $n_B$ is the number of
instances in $B$.

It has been shown by \cite{Xu:Frank2004} that the best $f(B)$ to be added can be
achieved by seeking $h(\cdot)$ which maximizes $\sum_i\sum_{j=1}^{n_i}[\frac{1}{n_i}
W^{(i)} g^{(i)} h({\bm b}_{j}^{(i)})]$, given the bag-level weights $W=\exp(-g
\mathcal{F}(B))$. By assigning each instance the label of its bag and the corresponding
weight ${W^{(i)} / {n_i}}$, $h(\cdot)$ can be learned by minimizing the weighted
instance-level classification error. This actually corresponds to the Step 3a of {\sc
MimlBoost}. When $f(B)$ is found, the best multiplier $c > 0$ can be got by directly
optimizing the exponential loss:
\begin{eqnarray}
  E_\mathcal{B}E_{\mathcal{G}|\mathcal{B}}[\exp(-g \mathcal{F}(B)+c(-g f(B)))] &=& {\sum}_i W^{(i)}
  \exp\left[c\left(-\frac{g^{(i)}\sum_j h({\bm b}_{j}^{(i)})}{n_i}\right)\right] \nonumber\\
   &=& {\sum}_i W^{(i)}\exp[(2e^{(i)} - 1)c] \ ,
\end{eqnarray}
where $e^{(i)} = {1 \over n_i} \sum_j [\kern-0.15em[ (h({\bm b}_{j}^{(i)})\neq g^{(i)})
]\kern-0.15em]$ (computed in Step 3b). Minimization of this expectation actually
corresponds to Step 3d, where numeric optimization techniques such as quasi-Newton
method can be used. Note that in Step 3c if $e^{(i)} \ge 0.5$, the Boosting process
will stop \cite{Zhou:Yu2009}. Finally, the bag-level weights are updated in Step 3f
according to the additive structure of $\mathcal{F}(B)$.

\subsection{\textsc{MimlSvm}}\label{sec:MIMLSVM}\vspace{-5mm}

Now we propose the {\sc MimlSvm} algorithm according to the second solution mentioned
before, that is, identifying the equivalence in the traditional supervised learning
framework using multi-label learning as the bridge. Note that this strategy can also be
used to derive other kinds of MIML algorithms.

Again, given any set $\Omega$, let $|\Omega|$ denote its size, i.e., the number of
elements in $\Omega$; given $(X_i, Y_i)$ and ${\bm z}_i = \phi(X_i)$ where $\phi:
2^\mathcal{X} \to \mathcal{Z}$, for any $y \in \mathcal{Y}$, let $\Phi({\bm z}_i, y) =
+1$ if $y \in Y_i$ and $-1$ otherwise, where $\Phi$ is a function $\Phi: \mathcal{Z}
\times \mathcal{Y} \to \{-1, +1\}$. The basic assumption of \textsc{MimlSvm} is that
the spatial distribution of the bags carries relevant information, and information
helpful for label discrimination can be discovered by measuring the closeness between
each bag and the representative bags identified through clustering. The pseudo-code of
\textsc{MimlSvm} is summarized in Appendix A (Table~\ref{table:mimlsvm}).

In the first step of {\sc MimlSvm}, the $X_u$ of each MIML example $(X_u,Y_u)$ $(u =
1,2,\cdots,m)$ is collected and put into a data set $\Gamma$. Then, in the second step,
$k$-medoids clustering is performed on $\Gamma$. Since each data item in $\Gamma$, i.e.
$X_u$, is an unlabeled multi-instance bag instead of a single instance, Hausdorff
distance \cite{Edgar1990} is employed to measure the distance. The Hausdorff distance
is a famous metric for measuring the distance between two bags of points, which has
often been used in computer vision tasks; other techniques that can measure the
distance between bags of points, such as the \textit{set kernel}
\cite{Gartner:Flach:Kowalczyk2002}, can also be used here. In detail, given two bags $A
= \{ {\bm a}_1, {\bm a}_2, \cdots, {\bm a}_{n_A} \}$ and $B = \{ {\bm b}_1, {\bm b}_2,
\cdots, {\bm b}_{n_B} \}$, the Hausdorff distance between $A$ and $B$ is defined as
\begin{equation}\label{eq:Hausdorff}
d_H (A,B)=\max\{\max\limits_{{\bm a}\in A}\min\limits_{{\bm b}\in B}\|{\bm a}-{\bm
b}\|,\max\limits_{{\bm b}\in B}\min\limits_{{\bm a}\in A}\|{\bm b}-{\bm a}\|\} \ ,
\end{equation}
where $\|{\bm a} - {\bm b}\|$ measures the distance between the instances ${\bm a}$ and
${\bm b}$, which takes the form of Euclidean distance here.

After the clustering process, the data set $\Gamma$ is divided into $k$ partitions,
whose medoids are $M_t$ $(t=1,2,\cdots,k)$, respectively. With the help of these
medoids, the original multi-instance example $X_u$ is transformed into a
$k$-dimensional numerical vector ${\bm z}_u$, where the $i$-th $(i=1,2,\cdots,k)$
component of ${\bm z}_u$ is the distance between $X_u$ and $M_i$, that is, $d_H (X_u,
M_i)$. In other words, ${\bm z}_{ui}$ encodes some structure information of the data,
that is, the relationship between $X_u$ and the $i$-th partition of $\Gamma$. This
process reassembles the {\it constructive clustering} process used by Zhou and Zhang
\cite{Zhou:Zhang2007} in transforming multi-instance examples into single-instance
examples except that in \cite{Zhou:Zhang2007} the clustering is executed at the
instance level while here it is executed at the bag level. Thus, the original MIML
examples $(X_u,Y_u)$ $(u = 1,2,\cdots,m)$ have been transformed into multi-label
examples $({\bm z}_u,Y_u)$ $(u = 1,2,\cdots,m)$, which corresponds to the Step 3 of
{\sc MimlSvm}.

Then, from the data set a multi-label learning function $f_{MLL}$ can be learned, which
can accomplish the desired MIML function because $f_{MIML}(X^*) = f_{MLL}({\bm z}^*)$.
In this paper, the {\sc MlSvm} algorithm \cite{Boutell:Luo:Shen:Brown2004} is used to
implement $f_{MLL}$. Concretely, {\sc MlSvm} decomposes the multi-label learning
problem into multiple independent binary classification problems (one per class), where
each example associated with the label set $Y$ is regarded as a positive example when
building {\sc Svm} for any class $y \in Y$, while regarded as a negative example when
building {\sc Svm} for any class $y \notin Y$, as shown in the Step 4 of {\sc MimlSvm}.
In making predictions, the {\it T-Criterion} \cite{Boutell:Luo:Shen:Brown2004} is used,
which actually corresponds to the Step 5 of the {\sc MimlSvm} algorithm. That is, the
test example is labeled by all the class labels with positive {\sc Svm} {\it scores},
except that when all the {\sc Svm} scores are negative, the test example is labeled by
the class label which is with the {\it top} (least negative) score.

\subsection{Experiments}\label{sec:MIMLexp}\vspace{-5mm}

\subsubsection{Multi-Label Evaluation Criteria}\label{sec:MLLcriteria}

In traditional supervised learning where each object has only one class label,
\textit{accuracy} is often used as the performance evaluation criterion. Typically,
accuracy is defined as the percentage of test examples that are correctly classified.
When learning with complicated objects associated with multiple labels simultaneously,
however, accuracy becomes less meaningful. For example, if approach $A$ missed one
proper label while approach $B$ missed four proper labels for a test example having
five labels, it is obvious that $A$ is better than $B$, but the accuracy of $A$ and $B$
may be identical because both of them incorrectly classified the test example.

Five criteria are often used for evaluating the performance of learning with
multi-label examples \cite{Schapire:Singer2000,Zhou:Zhang2007nips06}; they are
\textit{hamming loss}, \textit{one-error}, \textit{coverage}, \textit{ranking loss} and
\textit{average precision}. Using the same denotation as that in
Sections~\ref{sec:MIML} and \ref{sec:MIMLalgos}, given a test set $S = \{(X_1, Y_1),
(X_2, Y_2), \cdots, (X_p, Y_p)\}$, these five criteria are defined as below. Here,
$h(X_i)$ returns a set of proper labels of $X_i$; $h(X_i, y)$ returns a real-value
indicating the confidence for $y$ to be a proper label of $X_i$; $rank^h (X_i, y)$
returns the rank of $y$ derived from $h(X_i, y)$.

\begin{itemize}

\item $\mbox{hloss}_S(h)=\frac{1}{p}\sum_{i=1}^{p}\frac{1}{|\mathcal{Y}|}|h(X_i)\Delta Y_i|$, where $\Delta$
stands for the symmetric difference between two sets. The {\it hamming loss} evaluates
how many times an object-label pair is misclassified, i.e., a proper label is missed or
a wrong label is predicted. The performance is perfect
when ${\rm hloss}_S(h)=0$; the smaller the value of ${\rm hloss}_S(h)$, the better the performance of $h$. \\


\item $\mbox{one-error}_S(h) = {1 \over p} \sum_{i=1}^{p} [\kern-0.15em[ [\mathop{\arg\max}_{y \in \mathcal{Y}}
h(X_i, y)] \notin Y_i ]\kern-0.15em]$. The {\it one-error} evaluates how many times the
top-ranked label is not a proper label of the object. The performance is perfect when
one-error$_S (h) = 0$; the smaller the value of
one-error$_S (h)$, the better the performance of $h$. \\


\item $\mbox{coverage}_S(h) = {1 \over p} \sum_{i=1}^{p} \mathop{\max}_{y \in Y_i} rank^h(X_i, y) - 1$. The {\it
coverage} evaluates how far it is needed, on the average, to go down the list of labels
in order to cover all the proper labels of the object. It is loosely related to
precision at the level of perfect recall. The smaller the
value of coverage$_S (h)$, the better the performance of $h$. \\


\item $\mbox{rloss}_S(h)=\frac{1}{p}\sum_{i=1}^{p}\frac{1}{|Y_i||\overline Y_i|}|\{(y_1,y_2)|h(X_i,y_1)\leq
h(X_i,y_2),\ (y_1,y_2)\in Y_i\times \overline{Y_i}\}|$, where $\overline{Y_i}$ denotes
the complementary set of $Y_i$ in $\mathcal{Y}$. The {\it ranking loss} evaluates the
average fraction of label pairs that are misordered for the object. The performance is
perfect when ${\rm rloss}_S (h) = 0$; the smaller the value of ${\rm rloss}_S
(h)$, the better the performance  of $h$. \\


\item $\mbox{avgprec}_S (h) = {1 \over p} \sum_{i=1}^{p} {1 \over |Y_i|} \sum_{y \in Y_i} {{ | \{ y^{'} | rank^h
(X_i, y^{'}) \le rank^h(X_i, y), \ y^{'} \in Y_i \}| } \over rank^h (X_i, y)}$. The
{\it average precision} evaluates the average fraction of proper labels ranked above a
particular label $y \in Y_i$. The performance is perfect when avgprec$_S (h) = 1$; the
larger the value of avgprec$_S (h)$, the better the performance of $h$.


\end{itemize}

In addition to the above criteria, we design two new multi-label criteria,
\textit{average recall} and \textit{average F1}, as below.

\begin{itemize}

\item $\mbox{avgrecl}_S (h) = {1 \over p} \sum_{i=1}^{p} {{| \{ y | rank^h (X_i, y) \le | h(X_i)|, \ y \in Y_i
\}|} \over |Y_i|}$. The {\it average recall} evaluates the average fraction of proper
labels that have been predicted. The performance is perfect when avgrecl$_S (h) = 1$;
the larger the value of avgrecl$_S (h)$, the better
the performance of $h$.\\

\item $\mbox{avgF1}_S (h) = {{2 \times \mbox{avgprec}_S (h) \times \mbox{avgrecl}_S (h)} \over {\mbox{avgprec}_S
(h)} + \mbox{avgrecl}_S (h)}$. The {\it average F1} expresses a tradeoff between the
{\it average precision} and the {\it average recall}. The performance is perfect when
avgF1$_S (h) = 1$; the larger the value of avgF1$_S (h)$, the better the performance of
$h$.

\end{itemize}

Note that since the above criteria measure the performance from different
aspects, it is difficult for one algorithm to outperform another on every one of
these criteria.

In the following we study the performance of MIML algorithms on two tasks involving
complicated objects with multiple semantic meanings. We will show that for such tasks, MIML is a good choice, and good performance can be achieved even by
using simple MIML algorithms such as \textsc{MimlBoost} and \textsc{MimlSvm}.

\subsubsection{Scene Classification}\label{sec:MIMLscene}

The scene classification data set consists of 2,000 natural scene images belonging to
the classes \textit{desert}, \textit{mountains}, \textit{sea}, \textit{sunset} and
\textit{trees}. Over 22$\%$ of the images belong to multiple classes simultaneously. Each
image has already been represented as a bag of nine instances generated by the {\sc
Sbn} method \cite{Maron:Ratan1998}, which uses a Gaussian filter to smooth the image
and then subsamples the image to an $8\times 8$ matrix of {\it color blobs} where each
blob is a $2\times 2$ set of pixels within the matrix. An instance corresponds to the
combination of a single blob with its four neighboring blobs (up, down, left, right),
which is described with 15 features. The first three features represent the mean R, G,
B values of the central blob and the remaining twelve features express the differences
in mean color values between the central blob and other four neighboring blobs
respectively.\footnote{The data set is available at
http://lamda.nju.edu.cn/data\_MIMLimage.ashx.}

%
%

We evaluate the performance of the MIML algorithms \textsc{MimlBoost} and
\textsc{MimlSvm}. Note that \textsc{MimlBoost} and \textsc{MimlSvm} are merely proposed
to illustrate the two general degeneration solutions to MIML problems shown in
Fig.~\ref{fig:twosolution}. We do not claim that they are the best algorithms that
can be developed through the degeneration paths. There may exist other processes for
transforming MIML examples into multi-instance single-label (MISL) examples or
single-instance multi-label (SIML) examples. Even by using the same degeneration
process as that used in \textsc{MimlBoost} and \textsc{MimlSvm}, there are also many
alternatives to realize the second step. For example, by using \textsc{mi-Svm}
\cite{Andrews:Tsochantaridis:Hofmann2003} to replace the \textsc{MiBoosting} used in
\textsc{MimlBoost} and by using the two-layer neural network structure
\cite{Zhang:Zhou2007} to replace the \textsc{MlSvm} used in \textsc{MimlSvm}, we get
\textsc{MimlSvm$_{mi}$} and \textsc{MimlNn} respectively. Their performance is also
evaluated in our experiments.

We compare the MIML algorithms with several state-of-the-art algorithms for learning
with multi-label examples, including \textsc{AdtBoost.MH}
\cite{DeComite:Gilleron:Tommasi2003}, \textsc{RankSvm} \cite{Elisseeff:Weston2002},
\textsc{MlSvm} \cite{Boutell:Luo:Shen:Brown2004} and \textsc{Ml-$k$nn}
\cite{Zhang:Zhou2007pr}; these algorithms have been introduced briefly in
Section~\ref{sec:survey}. Note that these are single-instance algorithms that regard
each image as a 135-dimensional feature vector, which is obtained by concatenating the
nine instances in the direction from upper-left to right-bottom.

The parameter configurations of \textsc{RankSvm}, \textsc{MlSvm} and \textsc{Ml-$k$nn} are set by considering the strategies adopted in \cite{Elisseeff:Weston2002}, \cite{Boutell:Luo:Shen:Brown2004} and \cite{Zhang:Zhou2007pr} respectively. For \textsc{RankSvm}, polynomial kernel is used where polynomial degrees of 2 to 9 are considered as in \cite{Elisseeff:Weston2002} and chosen by hold-out tests on training sets. For \textsc{MlSvm}, Gaussian kernel is used. For \textsc{Ml-$k$nn}, the number of nearest neighbors considered is set to 10.

The boosting rounds of \textsc{AdtBoost.MH} and \textsc{MimlBoost} are set to 25 and 50,
respectively; The performance of the two algorithms at different boosting rounds is shown in Appendix B (Fig.~\ref{fig:boostrounds}), it can be observed that at those rounds
the performance of the algorithms have become stable. Gaussian kernel {\sc Libsvm}
\cite{Chang:Lin2001} is used for the Step 3a of \textsc{MimlBoost}. The
\textsc{MimlSvm} and \textsc{MimlSvm$_{mi}$} are also realized with Gaussian kernels.
The parameter $k$ of \textsc{MimlSvm} is set to be 20\% of the number of training
images; The performance of this algorithm with different $k$ values is shown in Appendix B (Fig.~\ref{fig:kvalues}), it can be observed that the setting of $k$ does not
significantly affect the performance of \textsc{MimlSvm}. Note that in Appendix B (Figs.~\ref{fig:boostrounds} and \ref{fig:kvalues}) we plot $1-$\textit{average
precision}, $1-$\textit{average recall} and $1-$\textit{average F1} such that in all
the figures, the lower the curve, the better the performance.

Here in the experiments, 1,500 images are used as training examples while the remaining
500 images are used for testing. Experiments are repeated for thirty runs by using
random training/test partitions, and the average and standard deviation are summarized
in Table~\ref{table:scene},\footnote{For the shared implementation of \textsc{AdtBoost.MH} (http://www.grappa.univ-lille3.fr/ grappa/en\_index.php3?info=software), \textit{ranking loss}, \textit{average recall} and \textit{average F1} are not available in the program's outputs.} where the best
performance on each criterion has been highlighted in boldface.

\renewcommand{\baselinestretch}{.95}
\small\normalsize

\begin{table}[!t]
\caption{Results (mean$\pm$std.) on scene classification data set (`$\downarrow$' indicates `the
smaller the better'; `$\uparrow$' indicates `the larger the
better')}\smallskip\smallskip\label{table:scene}\scriptsize
\begin{center}
\begin{tabular}{lccccccc}
\hline\noalign{\smallskip}
\raisebox{-2mm}{Compared} & \multicolumn{7}{c}{Evaluation Criteria} \\
\cline{2-8}\noalign{\smallskip}
\raisebox{2mm}{Algorithms} & \multicolumn{1}{c}{$hloss$ $^{\downarrow}$} & \multicolumn{1}{c}{$one\mbox{-}error$ $^{\downarrow}$} & \multicolumn{1}{c}{$coverage$ $^{\downarrow}$} & \multicolumn{1}{c}{$rloss$ $^{\downarrow}$} & \multicolumn{1}{c}{$aveprec$ $^{\uparrow}$} & \multicolumn{1}{c}{$averecl$ $^{\uparrow}$} & \multicolumn{1}{c}{$aveF1$ $^{\uparrow}$}\\
\hline\noalign{\smallskip}
\multicolumn{1}{l}{\textsc{MimlBoost}} & \multicolumn{1}{c}{.193$\pm$.007} & \multicolumn{1}{c}{.347$\pm$.019} & \multicolumn{1}{c}{\bf .984$\pm$.049} & \multicolumn{1}{c}{\bf .178$\pm$.011} & \multicolumn{1}{c}{.779$\pm$.012} & \multicolumn{1}{c}{.433$\pm$.027} & \multicolumn{1}{c}{.556$\pm$.023}\\
\multicolumn{1}{l}{\textsc{MimlSvm}} & \multicolumn{1}{c}{189$\pm$.009} & \multicolumn{1}{c}{.354$\pm$.022} & \multicolumn{1}{c}{1.087$\pm$.047}& \multicolumn{1}{c}{.201$\pm$.011}& \multicolumn{1}{c}{.765$\pm$.013}& \multicolumn{1}{c}{.556$\pm$.020}& \multicolumn{1}{c}{.644$\pm$.018}\\
\multicolumn{1}{l}{\textsc{MimlSvm$_{mi}$}} & \multicolumn{1}{c}{.195$\pm$.008} & \multicolumn{1}{c}{\bf .317$\pm$.018} & \multicolumn{1}{c}{1.068$\pm$.052} & \multicolumn{1}{c}{.197$\pm$.011} & \multicolumn{1}{c}{\bf .783$\pm$.011} & \multicolumn{1}{c}{\bf .587$\pm$.019} & \multicolumn{1}{c}{\bf .671$\pm$.015}\\
\multicolumn{1}{l}{\textsc{MimlNn}} & \multicolumn{1}{c}{\bf .185$\pm$.008} & \multicolumn{1}{c}{.351$\pm$.026} & \multicolumn{1}{c}{1.057$\pm$.054} & \multicolumn{1}{c}{.196$\pm$.013} & \multicolumn{1}{c}{.771$\pm$.015} & \multicolumn{1}{c}{.509$\pm$.022} & \multicolumn{1}{c}{.613$\pm$.020}\\

\hline\noalign{\smallskip}

\multicolumn{1}{l}{\textsc{AdtBoost.MH}} & \multicolumn{1}{c}{.211$\pm$.006} & \multicolumn{1}{c}{.436$\pm$.019} & \multicolumn{1}{c}{1.223$\pm$.050} & \multicolumn{1}{c}{N/A} & \multicolumn{1}{c}{.718$\pm$.012} & \multicolumn{1}{c}{N/A} & \multicolumn{1}{c}{N/A}\\
\multicolumn{1}{l}{\textsc{RankSvm}} & \multicolumn{1}{c}{.210$\pm$.024} & \multicolumn{1}{c}{.395$\pm$.075} & \multicolumn{1}{c}{1.161$\pm$.154} & \multicolumn{1}{c}{.221$\pm$.040} & \multicolumn{1}{c}{.746$\pm$.044} & \multicolumn{1}{c}{.529$\pm$.068} & \multicolumn{1}{c}{.620$\pm$.059}\\
\multicolumn{1}{l}{\textsc{MlSvm}} & \multicolumn{1}{c}{.232$\pm$.004} & \multicolumn{1}{c}{.447$\pm$.023} & \multicolumn{1}{c}{1.217$\pm$.054} & \multicolumn{1}{c}{.233$\pm$.012} & \multicolumn{1}{c}{.712$\pm$.013} & \multicolumn{1}{c}{.073$\pm$.010} & \multicolumn{1}{c}{.132$\pm$.017}\\
\multicolumn{1}{l}{\textsc{Ml-$k$nn}} & \multicolumn{1}{c}{.191$\pm$.006} & \multicolumn{1}{c}{.370$\pm$.017} & \multicolumn{1}{c}{1.085$\pm$.048} & \multicolumn{1}{c}{.203$\pm$.010} & \multicolumn{1}{c}{.759$\pm$.011} & \multicolumn{1}{c}{.407$\pm$.026} & \multicolumn{1}{c}{.529$\pm$.023}\\
\hline
\end{tabular}
\end{center}\bigskip\bigskip
\end{table}

\renewcommand{\baselinestretch}{1.4}
\small\normalsize

Pairwise $t$-tests with 95\% significance level disclose that all the MIML algorithms
are significantly better than \textsc{AdtBoost.MH} and \textsc{MlSvm} on all the
seven evaluation criteria. This is impressive since as mentioned before, these
evaluation criteria measure the learning performance from different aspects and one
algorithm rarely outperforms another algorithm on all criteria.  \textsc{MimlSvm} and
\textsc{MimlSvm$_{mi}$} are both significantly better than \textsc{RankSvm} on all the
evaluation criteria, while \textsc{MimlBoost} and \textsc{MimlNn} are both
significantly better than \textsc{RankSvm} on the first five criteria. \textsc{MimlNn}
is significantly better than \textsc{Ml-$k$nn} on all the evaluation criteria. Both
\textsc{MimlBoost} and \textsc{MimlSvm$_{mi}$} are significantly better than
\textsc{Ml-$k$nn} on all criteria except \textit{hamming loss}. \textsc{MimlSvm} is
significantly better than \textsc{Ml-$k$nn} on \textit{one-error}, \textit{average
precision}, \textit{average recall} and \textit{average F1}, while there are ties on
the other criteria. Moreover, note that the best performance on all evaluation
criteria are always attained by MIML algorithms. Overall, comparison on the scene
classification task shows that the MIML algorithms can be significantly better than the
non-MIML algorithms; this validates the powerfulness of the MIML framework.

\subsubsection{Text Categorization}\label{sec:MIMLtext}

The \textsc{Reuters}-21578 data set is used in this experiment. The seven most frequent
categories are considered. After removing documents that do not have labels or main
texts, and randomly removing some documents that have only one label, a data set
containing 2,000 documents is obtained, where over 14.9$\%$ documents have multiple
labels. Each document is represented as a bag of instances according to the method used
in \cite{Andrews:Tsochantaridis:Hofmann2003}. Briefly, the instances are obtained by
splitting each document into passages using overlapping windows of maximal 50 words
each. As a result, there are 2,000 bags and the number of instances in each bag varies
from 2 to 26 (3.6 on average). The instances are represented based on term frequency.
The words with high frequencies are considered, excluding ``function words'' that have
been removed from the vocabulary using the \textsc{Smart} stop-list \cite{Salton1989}.
It has been found that based on document frequency, the dimensionality of the data set
can be reduced to 1-10$\%$ without loss of effectiveness \cite{Yang:Pedersen1997}.
Thus, we use the top 2\% frequent words, and therefore each instance is a
243-dimensional feature vector.\footnote{The data set is available at
http://lamda.nju.edu.cn/data\_MIMLtext.ashx}

The parameter configurations of \textsc{RankSvm}, \textsc{MlSvm} and \textsc{Ml-$k$nn} are set in the same way as in Section \ref{sec:MIMLscene}. The boosting rounds of \textsc{AdtBoost.MH} and \textsc{MimlBoost} are set to 25 and 50, respectively. Linear kernels are used. The parameter $k$ of \textsc{MimlSvm} is set to
be 20\% of the number of training images. The single-instance algorithms regard each
document as a 243-dimensional feature vector which is obtained by aggregating all the
instances in the same bag; this is equivalent to represent the document using a sole
term frequency feature vector.

Here in the experiments, 1,500 documents are used as training examples while the
remaining 500 documents are used for testing. Experiments are repeated for thirty runs
by using random training/test partitions, and the average and standard deviation are
summarized in Table~\ref{table:reuters}, where the best performance on each criterion
has been highlighted in boldface.

\renewcommand{\baselinestretch}{.95}
\small\normalsize

\begin{table}[!t]
\caption{Results (mean$\pm$std.) on text categorization data set (`$\downarrow$' indicates `the
smaller the better'; `$\uparrow$' indicates `the larger the
better')}\smallskip\smallskip\label{table:reuters}\scriptsize
\begin{center}
\begin{tabular}{lccccccc}
\hline\noalign{\smallskip}
\raisebox{-2mm}{Compared} & \multicolumn{7}{c}{Evaluation Criteria} \\
\cline{2-8}\noalign{\smallskip}
\raisebox{2mm}{Algorithms} & \multicolumn{1}{c}{$hloss$ $^{\downarrow}$} & \multicolumn{1}{c}{$one\mbox{-}error$ $^{\downarrow}$} & \multicolumn{1}{c}{$coverage$ $^{\downarrow}$} & \multicolumn{1}{c}{$rloss$ $^{\downarrow}$} & \multicolumn{1}{c}{$aveprec$ $^{\uparrow}$}  & \multicolumn{1}{c}{$averecl$ $^{\uparrow}$} & \multicolumn{1}{c}{$aveF1$ $^{\uparrow}$}\\
\hline\noalign{\smallskip}
\multicolumn{1}{l}{\textsc{MimlBoost}}  & \multicolumn{1}{c}{.053$\pm$.004} & \multicolumn{1}{c}{.094$\pm$.014} & \multicolumn{1}{c}{.387$\pm$.037} & \multicolumn{1}{c}{.035$\pm$.005} & \multicolumn{1}{c}{.937$\pm$.008} & \multicolumn{1}{c}{.792$\pm$.010} & \multicolumn{1}{c}{.858$\pm$.008}\\
\multicolumn{1}{l}{\textsc{MimlSvm}} & \multicolumn{1}{c}{\bf .033$\pm$.003}& \multicolumn{1}{c}{.066$\pm$.011}& \multicolumn{1}{c}{.313$\pm$.035}& \multicolumn{1}{c}{.023$\pm$.004}& \multicolumn{1}{c}{.956$\pm$.006}& \multicolumn{1}{c}{\bf .925$\pm$.010}& \multicolumn{1}{c}{.940$\pm$.008}\\
\multicolumn{1}{l}{\textsc{MimlSvm$_{mi}$}} & \multicolumn{1}{c}{.041$\pm$.004} & \multicolumn{1}{c}{\bf .055$\pm$.009} & \multicolumn{1}{c}{\bf .284$\pm$.030} & \multicolumn{1}{c}{\bf .020$\pm$.003} & \multicolumn{1}{c}{\bf .965$\pm$.005} & \multicolumn{1}{c}{.921$\pm$.012} & \multicolumn{1}{c}{\bf .942$\pm$.007}\\
\multicolumn{1}{l}{\textsc{MimlNn}} & \multicolumn{1}{c}{.038$\pm$.002} & \multicolumn{1}{c}{.080$\pm$.010} & \multicolumn{1}{c}{.320$\pm$.030} & \multicolumn{1}{c}{.025$\pm$.003} & \multicolumn{1}{c}{.950$\pm$.006} & \multicolumn{1}{c}{.834$\pm$.011} & \multicolumn{1}{c}{.888$\pm$.008}\\

\hline\noalign{\smallskip}

\multicolumn{1}{l}{\textsc{AdtBoost.MH}}  & \multicolumn{1}{c}{.055$\pm$.005} & \multicolumn{1}{c}{.120$\pm$.017} & \multicolumn{1}{c}{.409$\pm$.047} & \multicolumn{1}{c}{N/A} & \multicolumn{1}{c}{.926$\pm$.011} & \multicolumn{1}{c}{N/A} & \multicolumn{1}{c}{N/A}\\
\multicolumn{1}{l}{\textsc{RankSvm}} & \multicolumn{1}{c}{.120$\pm$.013} & \multicolumn{1}{c}{.196$\pm$.126} & \multicolumn{1}{c}{.695$\pm$.466} & \multicolumn{1}{c}{.085$\pm$.077} & \multicolumn{1}{c}{.868$\pm$.092} & \multicolumn{1}{c}{.411$\pm$.059} & \multicolumn{1}{c}{.556$\pm$.068}\\
\multicolumn{1}{l}{\textsc{MlSvm}} & \multicolumn{1}{c}{.050$\pm$.003} & \multicolumn{1}{c}{.081$\pm$.011} & \multicolumn{1}{c}{.329$\pm$.029} & \multicolumn{1}{c}{.026$\pm$.003} & \multicolumn{1}{c}{.949$\pm$.006} & \multicolumn{1}{c}{.777$\pm$.016} & \multicolumn{1}{c}{.854$\pm$.011}\\
\multicolumn{1}{l}{\textsc{Ml-$k$nn}} & \multicolumn{1}{c}{.049$\pm$.003} & \multicolumn{1}{c}{.126$\pm$.012} & \multicolumn{1}{c}{.440$\pm$.035} & \multicolumn{1}{c}{.045$\pm$.004} & \multicolumn{1}{c}{.920$\pm$.007} & \multicolumn{1}{c}{.821$\pm$.021} & \multicolumn{1}{c}{.867$\pm$.013}\\
\hline
\end{tabular}
\end{center}\bigskip\bigskip
\end{table}

\renewcommand{\baselinestretch}{1.4}
\small\normalsize

Pairwise $t$-tests with 95\% significance level disclose that, impressively, both
\textsc{MimlSvm} and \textsc{MimlSvm$_{mi}$} are significantly better than all the
non-MIML algorithms. \textsc{MimlNn} is significantly better than \textsc{AdtBoost.MH},
\textsc{RankSvm}, and \textsc{Ml-$k$nn} on all the evaluation criteria; significantly
better than \textsc{MlSvm} on \textit{hamming loss}, \textit{average recall} and
\textit{average F1} while there are ties on the other criteria. \textsc{MimlBoost} is
significantly better than \textsc{AdtBoost.MH} on all criteria except that there is a
tie on \textit{hamming loss}; significantly better than \textsc{RankSvm} on all
criteria; significantly better than \textsc{MlSvm} on \textit{average recall} and there
is a tie on \textit{average F1}; significantly better than \textsc{Ml-$k$nn} on
\textit{one-error}, \textit{coverage}, \textit{ranking loss} and \textit{average
precision}. Moreover, note that the best performance on all evaluation criteria are
always attained by MIML algorithms. Overall, comparison on the text categorization task
shows that the MIML algorithms are better than the non-MIML algorithms; this validates
the powerfulness of the MIML framework.

\section{Solving MIML Problems by Regularization}\label{sec:Dmiml}\vspace{-4mm}

The degeneration methods presented in Section~\ref{sec:MIMLalgos} may lose information
during the degeneration process, and thus a ``direct'' MIML algorithm is desirable. In
this section we propose a regularization method for MIML. In contrast to
\textsc{MimlSvm} and \textsc{MimlSvm$_{mi}$}, this method is developed from the regularization framework directly and so we call it \textsc{D-MimlSvm}. The basic assumption of \textsc{D-MimlSvm} is
that the labels associated to the same example have some relatedness, and the
performance of classifying the bags depends on the loss between the labels and the
predictions on the bags as well as on the constituent instances. Moreover, considering
that for any class label the number of positive examples is smaller than that of
negative examples, this method incorporates a mechanism to deal with class imbalance.
We employ the constrained concave-convex procedure (\textsc{Cccp}) which has
well-studied convergence properties \cite{Smola:Vishwanathan:Hofmann2005} to solve the
resultant non-convex optimization problem. We also present a cutting plane algorithm
that finds the solution efficiently.

\subsection{The Loss Function}\vspace{-5mm}

Given a set of MIML training examples $\{(X_1,Y_1),(X_2,Y_2),\cdots,(X_m,Y_m)\}$, the goal of \textsc{D-MimlSvm} is to learn a mapping ${\bm f}: 2^\mathcal{X}\rightarrow 2^\mathcal{Y}$ where the proper label set for each bag $X\subseteq \mathcal{X}$ corresponds to ${\bm f}(X)\subseteq \mathcal{Y}$. Specifically, \textsc{D-MimlSvm} chooses to instantiate ${\bm f}$ with $T$ functions, i.e. ${\bm f}=(f_1,f_2,\cdots,f_T)$, where $T$ is the number of labels in the label space $\mathcal{Y}=\{l_1,l_2,\cdots,l_T\}$. Here, the $t$-th function $f_t: 2^\mathcal{X}\rightarrow \mathcal{R}$ determines the belongingness of $l_t$ for $X$, i.e. ${\bm f}(X)=\{l_t\mid f_t(X)>0,\,1\leq t\leq T\}$. In addition, each single instance ${\bm x}\in \mathcal{X}$ in a bag $X$ can be viewed as a bag $\{{\bm x}\}$ containing only one instance, such that ${\bm f}(\{{\bm x}\})=(f_1(\{{\bm x}\}),f_2(\{{\bm x}\}),\cdots,f_T(\{{\bm x}\}))$ is also a well-defined function. For convenience, ${\bm f}(\{{\bm x}\})$ and $f_t(\{{\bm x}\})$ are simplified as ${\bm f}({\bm x})$ and $f_t({\bm x})$ in the rest of this section.

To train the component functions $f_t\,(1\leq t\leq T)$ in ${\bm f}$, \textsc{D-MimlSvm} employs the following empirical loss function $V$ involving two terms (balanced by $\lambda$):
\begin{equation}\label{eq:loss0}
    V(\{X_i\}_{i=1}^m,\{Y_i\}_{i=1}^m,{\bm f})=V_1(\{X_i\}_{i=1}^m,\{Y_i\}_{i=1}^m,{\bm f})+
    \lambda\cdot V_2(\{X_i\}_{i=1}^m,{\bm f})
\end{equation}
Here, the first term $V_1$ considers the loss between the ground-truth label set of each training bag $X_i$, i.e. $Y_i$, to its predicted label set, i.e. ${\bm f}(X_i)$. Let $y_{it}=1$ if $l_t\in Y_i$ holds ($1\leq i\leq m,\ 1\leq t\leq T$). Otherwise, $y_{it}=-1$. Furthermore, let $(z)_+=\max(0,z)$ denote the hinge loss function. Accordingly, the first loss term $V_1$ is defined as:
\begin{equation}\label{eq:loss1}
    V_1(\{X_i\}_{i=1}^m,\{Y_i\}_{i=1}^m,{\bm f})=\frac{1}{mT}\sum_{i=1}^m\sum_{t=1}^T(1-y_{it}f_t(X_i))_+
\end{equation}
The second term $V_2$ considers the loss between ${\bm f}(X_i)$ and the predictions of $X_i$'s constituent instances, i.e. $\{{\bm f}({\bm x}_{ij})\mid 1\leq j\leq n_i\}$, which reflects the relationships between the bag $X_i$ and its instances $\{{\bm x}_{i1},{\bm x}_{i2},\cdots,{\bm x}_{i,n_i}\}$. Here, the common assumption in multi-instance learning is that the strength for $X_i$ to hold a label is equal to the maximum strength for its instances to hold the label, i.e. $f_t(X_i)=\max\limits_{j=1,\cdots,n_i}f_t({\bm x}_{ij})$.\footnote{Note that this assumption may be restrictive to some extent. There are many cases where the label of the bag does not rely on the instance with the maximum predictions, as discussed in Section \ref{sec:survey}. In addition, in classification only the sign of prediction is important \cite{Cheung:Kwok2006}, i.e. $sign(f_t(X_i))=sign(\max\limits_{j=1,\cdots,n_i}f_t({\bm x}_{ij}))$. However, in this paper the above common assumption is still adopted due to its popularity and simplicity.} Accordingly, the second loss term $V_2$ is defined as:
\begin{equation}\label{eq:loss2}
     V_2(\{X_i\}_{i=1}^m,{\bm f})=\frac{1}{mT}\sum_{i=1}^m\sum_{t=1}^T l\left(f_t(X_i),\max\limits_{j=1,\cdots,n_i}f_t({\bm x}_{ij})\right)
\end{equation}
Here, $l(v_1,v_2)$ can be defined in various ways and is set to be the $l_1$ loss in this paper, i.e. $l(v_1,v_2)=|v_1-v_2|$. By combining Eq. \ref{eq:loss1} and Eq. \ref{eq:loss2}, the empirical loss function $V$ in Eq. \ref{eq:loss0} is then specified as:\vspace{-5mm}
\begin{eqnarray}\label{eq:loss}
  \nonumber V(\{X_i\}_{i=1}^m,\{Y_i\}_{i=1}^m,{\bm f}) &=& \frac{1}{mT}\sum_{i=1}^m\sum_{t=1}^T(1-y_{it}f_t(X_i))_+ \\
  && +\frac{\lambda}{mT}\sum_{i=1}^m\sum_{t=1}^T l\left(f_t(X_i),\max\limits_{j=1,\cdots,n_i}f_t({\bm x}_{ij})\right)
\end{eqnarray}

\subsection{Representer Theorem for MIML}\vspace{-5mm}

For simplicity, we assume that each function $f_t$ is a linear model, i.e., $f_t(\bm
x)= \langle \bm w_t,\phi(\bm x) \rangle$ where $\phi$ is the feature map induced by a
kernel function $k$ and $\langle\cdot,\cdot\rangle$ denotes the standard inner product
in the Reproducing Kernel Hilbert Space (RKHS) $\mathcal{H}$ induced by the kernel $k$.
We recall that an instance can be regarded as a bag containing only one instance, so the
kernel $k$ can be any kernel defined on a set of instances, such as the \textit{set
kernel} \cite{Gartner:Flach:Kowalczyk2002}. In the case of classification, objects
(bags or instances) are classified according to the sign of $f_t$.

\textsc{D-MimlSvm} assumes that the labels associated with a bag should have some relatedness;
otherwise they should not be associated with the bag simultaneously. To reflect this basic assumption, \textsc{D-MimlSvm} regularizes the empirical loss function in Eq. \ref{eq:loss} with an additional term $\Omega({\bm f})$:
\begin{equation}\label{eq:omega}
    \Omega({\bm f})+\gamma\cdot V(\{X_i\}_{i=1}^m,\{Y_i\}_{i=1}^m,{\bm f})
\end{equation}
Here, $\gamma$ is a regularization parameter balancing the model complexity $\Omega({\bm f})$ and the empirical risk $V$. Inspired by \cite{Evgeniou:Micchelli:Pontil2005}, we assume that the relatedness among the labels can be measured by the mean function $\bm w_0$,
\begin{equation}\label{eq:w0}
\bm w_0 = \frac{1}{T}\sum\limits_{t = 1}^T {\bm w_t } \
\end{equation}
The original idea in \cite{Evgeniou:Micchelli:Pontil2005} is to minimize
$\sum_{t=1}^{T}||\bm w_t - \bm w_0||^{2}$ and meanwhile minimize $||\bm w_0||^{2}$, i.e. to set the regularizer as:
\begin{equation}\label{eq:regularization}
    \Omega({\bm f})=\frac{1}{T}\sum_{t=1}^{T}||\bm w_t - \bm w_0||^{2}+\eta ||\bm w_0||^{2}
\end{equation}
According to Eq.\ref{eq:w0}, the first term in the RHS of Eq. \ref{eq:regularization} can be rewritten as:
\begin{equation}\label{eq:w0_v2}
\frac{1}{T}\sum_{t=1}^{T}\|\bm w_t - \bm w_0\|^{2} = \frac{1}{T}\sum_{t=1}^{T}\|\bm w_t\|^2 -  \|\bm w_0\|^2 \
\end{equation}
Therefore, by substituting Eq. \ref{eq:w0_v2} into Eq. \ref{eq:regularization}, the regularizer can be simplified as:
\begin{equation}\label{eq:regularizaiton1}
\Omega({\bm f})=\frac{1}{T}\sum_{t=1}^{T}||\bm w_t ||^{2}+\mu  ||\bm w_0||^{2}
\end{equation}
Further note that $\|\bm w_t\|^2 =\|f_t\|_{\mathcal{H}}^2$ and $\|\bm w_0\|^2 =
\|\frac{\sum_{t=1}^{T}f_t}{T}\|_{\mathcal{H}}^2$, by substituting Eq. \ref{eq:regularizaiton1} into Eq. \ref{eq:omega}, we have the regularization framework of \textsc{D-MimlSvm} as follows:
\begin{equation}\label{eq:framework}
\min\limits_{\bm f\in \mathcal H}\  \frac{1}{T} \sum_{t=1}^{T}
\|f_t\|_{\mathcal{H}}^2+\mu \|\frac{\sum_{t=1}^{T}f_t}{T}\|_{\mathcal{H}}^2 + \gamma\cdot
V\left( \{X_i\}_{i=1}^m, \{Y_i \}_{i=1}^{m}, {\bm f} \right) \
\end{equation}
Here, $\mu$ is a parameter to trade off the discrepancy and commonness
among the labels, that is, how similar or dissimilar the $\bm{w}_t$'s are. Refer to
Eq.~\ref{eq:w0_v2}, we have $\Omega({\bm f})=\frac{1}{T} \sum_{t=1}^{T} \|f_t\|_{\mathcal{H}}^2+\mu
\|\frac{\sum_{t=1}^{T}f_t}{T}\|_{\mathcal{H}}^2 = \frac{1}{T} \sum_{t=1}^{T}
\|f_t-\frac{\sum_{t=1}^{T}f_t}{T}\|_{\mathcal{H}}^2+(\mu+1)
\|\frac{\sum_{t=1}^{T}f_t}{T}\|_{\mathcal{H}}^2$. Intuitively, when $\mu+1$ (or $\mu$)
is large, minimization of Eq. \ref{eq:framework} will force $\|\frac{\sum_{t=1}^{T}f_t}{T}\|_{\mathcal{H}}^2$ to tend to be zero and the
discrepancy among the labels becomes more important; when $\mu+1$ (or $\mu$) is small,
minimization of Eq. \ref{eq:framework} will force $\|f_t-\frac{\sum_{t=1}^{T}f_t}{T}\|_{\mathcal{H}}^2$ to tend to be zero and the commonness among the labels becomes more important \cite{Evgeniou:Micchelli:Pontil2005}.

Given the above setup, we can prove the following representer theorem.

\begin{theorem}\label{theorem}
The minimizer of the optimization problem \ref{eq:framework} admits an expansion
$$
f_t(\bm x) = \sum\limits_{i=1}^{m} \left(\alpha_{t,i0} k \left(\bm x, X_i \right) +
\sum\limits_{j=1}^{n_i} \alpha_{t,ij} k (\bm x,\bm x_{ij}) \right)
$$
where all $\alpha_{t,i0},\alpha_{t,ij}\in \mathcal R$.
\end{theorem}

\begin{proof} Analogous to \cite{Evgeniou:Micchelli:Pontil2005}, we first introduce a combined feature map
\begin{equation*}
\Psi\left({\bm x}, t \right) = \left(\frac{\phi(\bm x)}{\sqrt{r}},\underbrace{{\bm 0},
\cdots, {\bm 0}}_{t-1},\phi(\bm x),\underbrace{{\bm 0}, \cdots, {\bm 0}}_{T-t}\right)
\end{equation*}
and its decision function, i.e., $\hat{f}(\bm x, t) = \langle \hat{\bm w}, \Psi({\bm
x}, t)\rangle$ where
\begin{equation*}
\hat{\bm w} = (\sqrt{r} {\bm w}_0, {\bm w}_1 - {\bm w}_0, \cdots, {\bm w}_T - {\bm
w}_0).
\end{equation*}
Here $r = \mu T + T$. Let $\hat{k}$ denote the kernel function induced by $\Psi$ and
$\hat{\mathcal{H}}$ is its corresponding RKHS. We have Eqs.~\ref{eq:map} and
\ref{eq:w2}.
\begin{equation}\label{eq:map}
\hat{f}(\bm x, t)= \langle \hat{\bm w}, \Psi({\bm x}, t)\rangle = \langle ({\bm w}_0 +
{\bm w}_t - {\bm w}_0), \phi(\bm x)\rangle = \langle
  {\bm w}_t, \phi(\bm x)\rangle = f_t(\bm x)
\end{equation}
\begin{align}\label{eq:w2}
\|\hat{f}\|_{\hat{\mathcal{H}}}^2 = ||\hat{\bm w}||^2 &= \sum\limits_{i=1}^{T}||\bm w_t
- \bm w_0||^{2}+r||\bm w_0||^2 = \sum\limits_{i=1}^{T}||\bm w_t||^{2}+\mu T||\bm
w_0||^2
\end{align}
Therefore, loss function in Eq.\ref{eq:loss} can be represented by $
\hat{V}(\{X_i\}_{i=1}^m,\{Y_i\}_{i=1}^{m},\hat{f})$, i.e.,
\begin{eqnarray}
\hat{V}(\{X_i\}_{i=1}^m,\{Y_i\}_{i=1}^{m}, \hat{f}) &=& \frac{1}{mT} \ \sum\limits_{i =
1}^m \sum\limits_{t = 1}^T \left( {1 -
y_{it} \hat{f}\left({X_i}, t \right)} \right)_+  \nonumber\\
&+& \frac{\lambda}{mT}  \sum\limits_{i = 1}^m \sum\limits_{t = 1}^T l \left(
{\hat{f}\left( {X_i}, t \right), \mathop{\max}\limits_{j = 1, \cdots, n_i} \hat{f}
\left( \bm x_{ij}, t \right) } \right).
\end{eqnarray}
Thus, Eq.~\ref{eq:framework} is equivalent to
\begin{align}\label{eq:represent}
\min\limits_{\hat{f} \in \hat{\mathcal{H}}}
\frac{1}{T}||\hat{f}||_{\hat{\mathcal{H}}}^2 +\gamma
\hat{V}(\{X_i\}_{i=1}^m,\{Y_i\}_{i=1}^{m},\hat{f}).
\end{align}
Note that $||\hat{f}||_{\hat{\mathcal{H}}}^2$ $:[0,\infty)\rightarrow \mathcal {R}$ is
a strictly monotonically increasing function. According to representer theorem (Theorem
4.2 in \cite{Scholkopf:Smola2002}), each minimizer $\hat{f}$ of the functional risk in
Eq.~\ref{eq:represent} admits a representation of the form
\begin{align}\label{eq:f}
\!\!\!\! \hat{f}({\bm x}, t) = \sum_{t=1}^{T} \sum_{i=1}^{m}\left(\beta_{t,i0}
\hat{k}\left(\left(X_i,t\right),\left({\bm x}, t\right)\right) + \sum_{j=1}^{n_i}
\beta_{t,ij} \hat{k}\left(\left(\bm x_{ij},t\right),\left({\bm x},
t\right)\right)\right) \ ,
\end{align}
where $\beta_{t,ij} \in \mathcal{R}$ and the corresponding weight vector $\hat{\bm w}$
is represented as
\begin{align}\label{eq:w}
\!\!\!\! \hat{\bm w} = \sum_{t=1}^{T} \sum_{i=1}^{m}\left(\beta_{t,i0}
\Psi\left(X_i,t\right) + \sum_{j=1}^{n_i} \beta_{t,ij} \Psi\left(\bm
x_{ij},t\right)\right) \ .
\end{align}
Finally, with Eqs.~\ref{eq:map} and \ref{eq:w}, we have
\begin{align}\label{eq:ftx}
f_t(\bm x) &= \langle \bm w_t, \phi(\bm x)\rangle = \langle \bm w ,\Psi(\bm x,t)\rangle \nonumber\\
&=\sum\limits_{i=1}^{m} \left(\alpha_{t,i0} k \left(\bm x, X_i \right) +
\sum\limits_{j=1}^{n_i} \alpha_{t,ij} k (\bm x,\bm x_{ij}) \right)
\end{align}
where $\alpha_{t,ij} = \frac{1}{\sqrt{r}}(\sum_t \beta_{t,ij}) + \beta_{t,ij}/r$.
\end{proof}

Note that $\bm x$ in Eq.~\ref{eq:ftx} can be regarded not only as a bag $X_i$ but also
an instance $\bm x_{ij}$. In other words, both $f_t(X_i)$ and $f_t(\bm x_{ij})$ can be
obtained by Eq.~\ref{eq:ftx}.


\subsection{Optimization}\vspace{-5mm}

Considering the use of $l_1$ loss for $l(v_1,v_2)$, Eq.\ref{eq:framework} can be
re-written as
\begin{eqnarray}\label{eq:framework2}
&\min\limits_{\bm f \in \mathcal{H},\bm \xi,\bm \delta}&
\frac{1}{T}\sum_{t=1}^{T}\|f_t\|_{\mathcal{H}}^2 + \mu
\|\frac{\sum_{t=1}^{T}f_t}{T}\|_{\mathcal{H}}^2 +
\frac{\gamma}{mT}\bm \xi'\bm 1 + \frac{\gamma \lambda}{mT}\bm \delta'\bm 1 \nonumber\\
&\mbox{s.t.} & y_{it}f_t(X_i) \geq 1 - \xi_{it}, \nonumber\\
&& \bm \xi \geq \bm 0, \nonumber\\
&& -\delta_{it} \leq f_t(X_i) - \max_{j=1,\ldots,n_i} f_t(\bm x_{ij}) \leq \delta_{it}\ \ \forall i=1,\ldots,m,\ t=1,\ldots,T
\end{eqnarray}
where $\bm \xi = [\xi_{11},\xi_{12},\cdots,\xi_{it},\cdots, \xi_{mT}]'$ are slack
variables for the errors on the training bags for each label, $\bm \delta =
[\delta_{11},\delta_{12},\cdots,\delta_{it},\cdots, \delta_{mT}]'$, and $\bm 0$ and
$\bm 1$ are all-zero and all-one vector, respectively.

Without loss of generality, assume that the bags and instances are ordered as $(X_1, \cdots, X_m,$ $\bm{x}_{11}, \cdots, \bm{x}_{1,n_1}, \cdots, \bm{x}_{m,1},
\cdots, \bm{x}_{m,n_m})$. Thus, each object (bag or instance) in the training set can
then be indexed by the following function $\mathcal I$, i.e.,
\[
\left\{ {\begin{array}{*{20}l}
   {\mathcal I(X_i)=i}  \\
   {\mathcal I(\bm x_{ij}) = m + \sum\limits_{l=1}^{i-1}n_l+j}  \\
\end{array}} \right.
\]
for $j = 1, \cdots, n_i$ and $i = 1, \cdots, m$. With this ordering, we can obtain the
$(m+n) \times (m+n)$ kernel matrix $\bm K$ defined on all objects in the training set,
where $n = \sum_{i=1}^{m} n_i$. Denote the $i$-th column of $\bm K$ by $\bm k_i$. According to theorem 1, we have $f_t(X_i)=\bm k_{\mathcal I(X_i)}' \bm \alpha_t+b_t$ and $f_t(\bm x_{ij})=\bm
k_{\mathcal I(\bm x_{ij})}' \bm \alpha_t+b_t$. Here, the bias $b_t$ for each label is
included.


According to definition of $f_t$ in Eq.~\ref{eq:ftx}, Eq.~\ref{eq:framework2} can be
cast as the optimization problem
\begin{eqnarray}\label{eq:opt3}
\min\limits_{\bm A, \bm \xi, \bm \delta, \bm b} && \frac{1}{2T}\sum\limits_{t=1}^{T}\bm
\alpha_t' \bm K \bm
\alpha_t + \frac{\mu}{T^2}\bm 1'\bm A' \bm K \bm A \bm 1 + \frac{\gamma}{mT}\bm \xi' \bm 1 + \frac{\gamma\lambda}{mT} \bm \delta' \bm 1\\
\mbox{s.t.} && y_{it}(\bm k_{\mathcal I(X_i)}'\bm \alpha_t+b_t) \ge 1 - \xi_{it}, \nonumber\\
            && \bm \xi \ge \bm 0, \nonumber\\
            && \bm k_{\mathcal I(\bm x_{ij})}'\bm \alpha_t-\delta_{it} \le \bm k_{\mathcal I(X_i)}'\bm \alpha_t, \nonumber\\
            && \bm k_{\mathcal I(X_i)}'\bm \alpha_t - \max\limits_{j=1,\cdots,n_i}\bm k_{\mathcal
               I(\bm x_{ij})}'\bm \alpha_t \le \delta_{it}, \nonumber
\end{eqnarray}
where $\bm A=[\bm \alpha_1,\bm \alpha_2, \cdots, \bm \alpha_T]$ and $\bm b=[b_1, b_2,
\cdots, b_T]^{'}$.

The above optimization problem is a non-convex optimization problem since the last
constraint is non-convex. Note that this non-convex constraint is a difference between
two convex functions, and thus the optimization problem can be solved by \textsc{Cccp}
\cite{Smola:Vishwanathan:Hofmann2005,Cheung:Kwok2006}, which is one of the most
standard techniques to solve such kind of non-convex optimization problems.
\textsc{Cccp} is guaranteed to converge to a local minimum \cite{Yulle:Rangarajan2003},
and in many cases it can even converge to a global solution \cite{PhamDinh:LeThi1998}.

Here, for solving the optimization problem \ref{eq:opt3}, \textsc{Cccp} works by
solving a sequential convex quadratic problems. Concretely, given the initial
subgradient $\sum\nolimits_{j=1}^{n_i}\rho_{ijt}\bm k_{\mathcal I(\bm x_{ij})}'\bm
\alpha_t$ of $\max\nolimits_{j=1,\cdots,n_i}\bm k_{\mathcal I(\bm x_{ij})}'\bm
\alpha_t$, we solve the following convex quadratic optimization (QP) problem\vspace{-5mm}
\begin{eqnarray}\label{eq:opt4} \min\limits_{\bm A, \bm \xi, \bm \delta, \bm b} &&
\frac{1}{2T}\sum\limits_{t=1}^{T}\bm \alpha_t' \bm K \bm
\alpha_t + \frac{\mu}{T^2}\bm 1'\bm A' \bm K \bm A \bm 1 + \frac{\gamma}{mT}\bm \xi' \bm 1 + \frac{\gamma\lambda}{mT} \bm \delta' \bm 1 \\
\mbox{s.t.} && y_{it}(\bm k_{\mathcal I(X_i)}'\bm \alpha_t+b_t) \ge 1 - \xi_{it}, \nonumber\\
    && \bm \xi \ge \bm 0, \nonumber\\
    && \bm k_{\mathcal I(\bm x_{ij})}'\bm \alpha_t-\delta_{it} \le \bm k_{\mathcal I(X_i)}'\bm \alpha_t, \nonumber\\
    && \bm k_{\mathcal I(X_i)}'\bm \alpha_t - \sum\nolimits_{j=1}^{n_i}\rho_{ijt}\bm k_{\mathcal
      I(\bm x_{ij})}'\bm \alpha_t \le \delta_{it}. \nonumber
\end{eqnarray}
Then, in the next iteration we update $\rho_{ijk}$ according to
\begin{equation*}
\rho _{ijt}  = \left\{
 \begin{array}{l}
  = 0,\quad {\rm if} \ \bm k_{\mathcal I(\bm x_{ij})}'\bm \alpha_t \ne \max\limits_{k=1,\cdots,n_i}
    \left(\bm k_{\mathcal I(\bm x_{ik})}' \bm \alpha_t\right),\\
  = 1/n_d,\quad {\rm otherwise,} \\
 \end{array} \right.
\end{equation*}
where $n_d$ is the number of active $\bm x_{ij}$'s. It holds
$\sum\limits_{j=1}^{n_i}\rho_{ijt}=1$ for any $t$'s. The iteration continues and this
procedure is guaranteed to converge to a local minimum.

\subsection{Handling Class-Imbalance}\vspace{-5mm}

The above solution may be improved further if we explicitly take into account the
instance-level class-imbalance, that is, for any class label the number of
\textit{positive} instances is smaller than the number of \textit{negative}
instances in MIML problems.

We can roughly estimate the \textit{imbalance rate}, which is the ratio of the number
of positive instances to that of negative instances, for each class label using the
strategy adopted by \cite{Kuck:deFreitas2005}. In detail, for a specific label $y \in
\mathcal{Y}$, we can divide the training bags $\{(X_1, Y_1), (X_2, Y_2), \cdots, (X_m,
Y_m)\}$ into two subsets, $A_1 = \{(X_i, Y_i) | y \in Y_i\}$ and $A_2 = \{(X_i, Y_i) |
y \notin Y_i\}$. It is obvious that all the instances in $A_2$ are negative to $y$.
Then, for every $(X_i, Y_i)$ in $A_1$, assuming that the instances of different labels
is roughly equally distributed, the number of positive instances of $y$ in $(X_i, Y_i)$
is roughly $n_i \times {1 \over |Y_i|} $ where $|Y_i|$ returns the number of labels in
$Y_i$. Thus, the imbalance rate of $y$ is:
\begin{equation*}\label{eq:imbalancerate}
ibr\left( y \right) = {\sum\limits_{\scriptstyle i = 1 \hfill \atop
  \scriptstyle y \in Y_i  \hfill}^m {\frac{{n_i }}{{\left| {Y_i } \right|}}} } \times {1 \over
  {\sum\limits_{i = 1}^m {n_i }}} = {\sum\limits_{\scriptstyle i = 1 \hfill \atop
  \scriptstyle y \in Y_i  \hfill}^m {\frac{{n_i }}{ n \times {\left| {Y_i }
  \right|}}} }.
\end{equation*}

There are many class-imbalance learning methods \cite{Weiss2004}. One of the most
popular and effective methods is \textit{rescaling} \cite{Zhou:Liu2006a}, which can be
incorporated into our framework easily. In short, after obtaining the estimated
imbalance rate for every class label, we can use these rates to modulate the loss
caused by different misclassifications.

In detail, $\bm \xi$ in Eq.~\ref{eq:opt4} is directly related to the hinge loss
$\left(1-y_{it}f_t\left(X_i\right)\right)_+$. According to the rescaling method
\cite{Zhou:Liu2006a},
without loss of generality, we can rewrite the loss function into
Eq. \ref{eq:imbalanceloss2}.
\begin{equation}
  \label{eq:imbalanceloss2}
  \left(\frac{y_{it}+1}{2}- y_{it} \times ibr(y_{it})\right)\left( 1-
  y_{it}f_t(X_i)\right).
\end{equation}

Let $\bm \tau=[\tau_{11},\tau_{12},\cdots,\tau_{it},\cdots,\tau_{mT}]$, where
$\tau_{it} = \left(\frac{y_{it}+1}{2}- y_{it} \times ibr(y_{it})\right)$. Then, to
minimize the loss defined in Eq.~\ref{eq:imbalanceloss2}, Eq.~\ref{eq:opt4} becomes
Eq.~\ref{eq:opt5}. Here $\bm \xi' \bm \tau$ indicates the weighted loss after
considering the instance-level class-imbalance. It is evident that the problem in
Eq.~\ref{eq:opt5} is still a standard QP problem.
\begin{eqnarray}
  \label{eq:opt5}
  \min\limits_{\bm A, \bm \xi, \bm \delta, \bm b} && \frac{1}{2T}\sum\limits_{t=1}^{T}\bm \alpha_t' \bm K \bm
\alpha_t + \frac{\mu}{T^2}\bm 1'\bm A' \bm K \bm A \bm 1 + \frac{\gamma}{mT}\bm \xi' \bm \tau + \frac{\gamma\lambda}{mT} \bm \delta' \bm 1 \\
\mbox{s.t.} && y_{it}(\bm k_{\mathcal I(X_i)}'\bm \alpha_t+b_t) \ge 1 - \xi_{it}, \nonumber\\
    && \bm \xi \ge \bm 0, \nonumber\\
    && \bm k_{\mathcal I(\bm x_{ij})}'\bm \alpha_t-\delta_{it} \le \bm k_{\mathcal I(X_i)}'\bm \alpha_t, \nonumber\\
    && \bm k_{\mathcal I(X_i)}'\bm \alpha_t - \sum\limits_{j=1}^{n_i}\rho_{ijt}\bm k_{\mathcal
      I(\bm x_{ij})}'\bm \alpha_t \le \delta_{it}. \nonumber
\end{eqnarray}

%

\subsection{Efficient Algorithm}\vspace{-5mm}

Eq.~\ref{eq:opt5} is a large-scale quadratic programming problem that involves many
constraints and variables. To make it tractable and scalable, and observing that
most of the constraints in Eq.~\ref{eq:opt5} are redundant,
we present an efficient algorithm which constructs a nested sequence of tighter
relaxations of the original problem using the cutting plane method \cite{Kelley1960}.

Similar to its use with structured prediction
\cite{Tsochantaridis:Joachims:Hofmann:Altun2005}, we add a constraint (or a cut) that
is most violated by the current solution, and then find the solution in the updated
feasible region. Such a procedure will converge to an optimal (or
$\varepsilon$-suboptimal) solution of the original problem. Moreover, Eq.~\ref{eq:opt5}
supports a natural problem decomposition since its constraint matrix is a block
diagonal matrix, i.e., each block corresponds to one label.

%
%

The pseudo-code of the algorithm is summarized in Appendix A (Table~\ref{table:cuttingplane}). We first
initialize the working sets $S_t$'s as empty sets and the solutions as all zeros (Line
1). Then, instead of testing all the constraints, which is rather expensive when there
are lots of constraints, we use the speedup heuristic as described in
\cite{Smola:Scholkopf2000}, i.e., we use $p$ constraints to approximate the whole
constraints (Line 4). Smola and Sch{\"o}lkopf \cite{Smola:Scholkopf2000} have shown
that when $p$ is larger than 59, the selected violated constraint is with
probability 0.95 among the $5\%$ most violated constraints among all constraints. The
$Loss_i$ (Line 5) is calculated as $\max \{0, {\bm u}' {\bm x} - d\}$ where ${\bm u}$
and $d$ are the linear coefficients and bias of the $i$-th linear constraint,
respectively.
If the maximal $Loss$ is lower than the given stopping criteria $\varepsilon$ (we
simply set $\varepsilon$ as $10^{-4}$ in our experiments), no update will be taken for
the working set $S_t$; otherwise the constraint with the maximal $Loss$ will be added
into $S_t$ (lines 8 and 9). Once a new constraint is added, the solution will be
re-computed with respect to $S_t$ via solving a smaller quadratic program problem (line
10). The algorithm stops when there is no update for all $S_t$'s.

\subsection{Experiments}\label{sec:Dmimlexp}\vspace{-5mm}

The previous experiments in Section~\ref{sec:MIMLexp} have shown that different MIML
algorithms have different advantages on different performance measures. In this section
we propose the \textsc{D-MimlSvm} algorithm. We do not claim that \textsc{D-MimlSvm} is
the best MIML algorithm. What we want to show is that, in contrast to heuristically
solving the MIML problem by degeneration, developing algorithms from a regularization
framework directly offers a better choice. So the most meaningful comparison is between
the \textsc{D-MimlSvm}, \textsc{MimlSvm} and \textsc{MimlSvm$_{mi}$} algorithms, the latter two not being derived from the regularization framework directly.

To study the behavior of \textsc{D-MimlSvm}, \textsc{MimlSvm} and \textsc{MimlSvm$_{mi}$} under different
amounts of multi-label data, we derive five data sets from the scene data used in
Section~\ref{sec:MIMLscene}. By randomly removing some single-label images, we obtain a
data set where 30\% (or 40\%, or 50\%) images belonging to multiple classes
simultaneously; by randomly removing some multi-label images, we obtain a data set
where 10\% (or 20\%) images belong to multiple classes simultaneously. A similar
process is applied to the text data used in Section~\ref{sec:MIMLtext} to derive five
data sets. On the derived data sets we use 25\% data for training and the remaining
75\% data for testing, and experiments are repeated for thirty runs with random
training/test partitions. The parameters of \textsc{D-MimlSvm}, \textsc{MimlSvm} and
\textsc{MimlSvm$_{mi}$} are all set by hold-out tests on training sets. Since \textsc{D-MimlSvm}
needs to solve a large optimization problem, although we have incorporated advanced
mechanisms such as cutting-plane algorithm, the current \textsc{D-MimlSvm} can only
deal with moderate training set sizes.

The seven criteria introduced in Section~\ref{sec:MLLcriteria} are used to evaluate the
performance.  The average and standard deviation are plotted in
Figs.~\ref{fig:scene-curves} and \ref{fig:text-curves}. Note that in the figures we
plot $1-$\textit{average precision}, $1-$\textit{average recall} and
$1-$\textit{average F1} such that in all the figures, the lower the curve, the better
the performance.

\begin{figure}[!t]
\centering
\begin{minipage}[c]{1.5in}
\centering
\includegraphics[width = 1.5in]{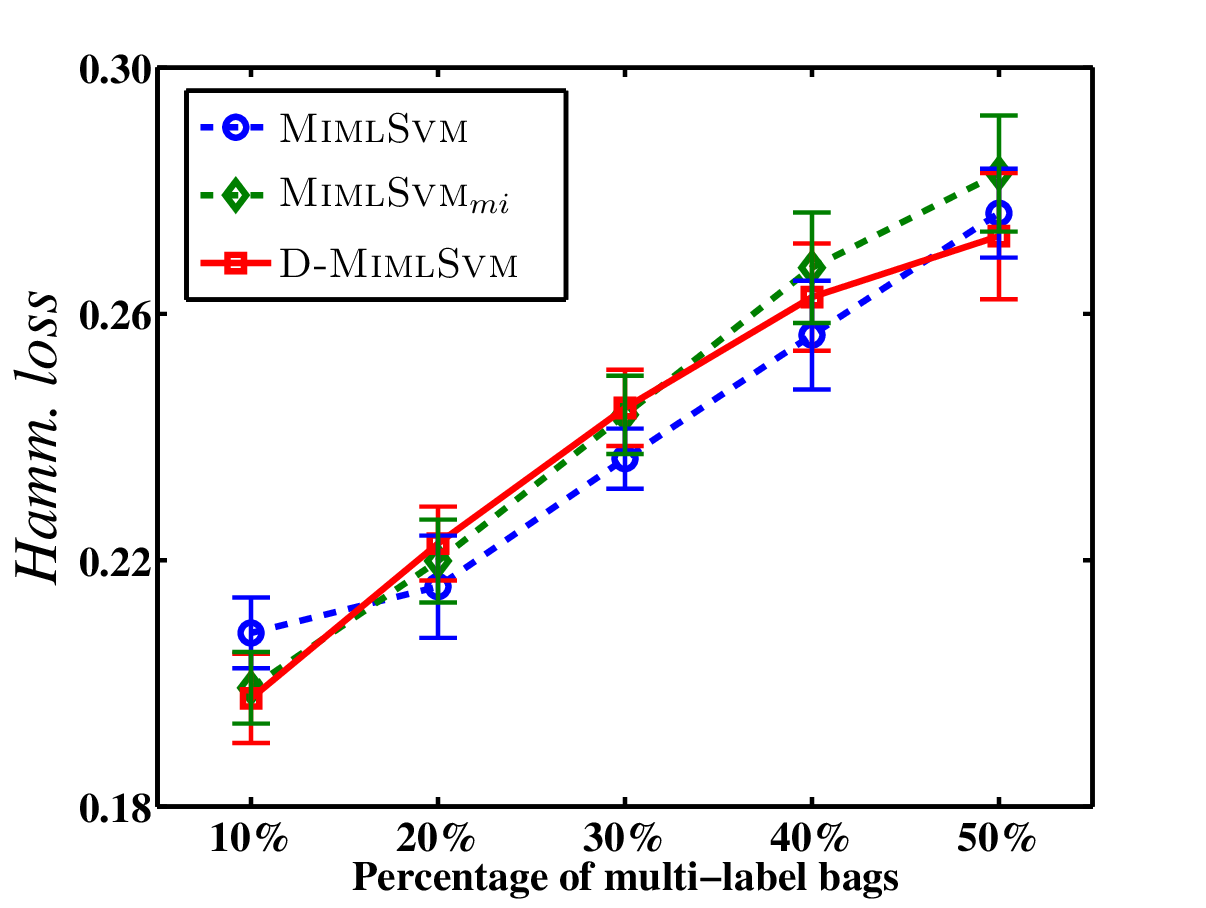}
\end{minipage}%
\begin{minipage}[c]{1.5in}
\centering
\includegraphics[width = 1.5in]{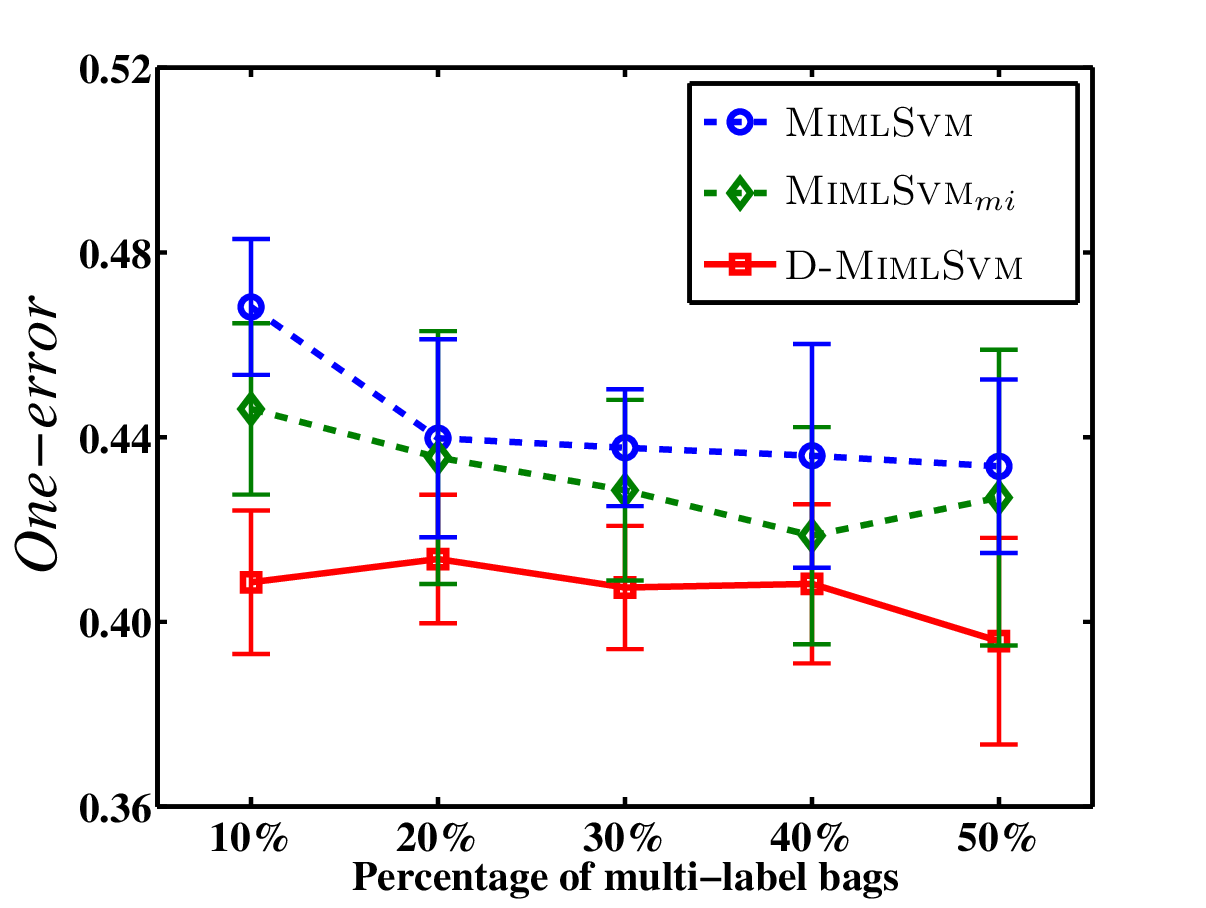}
\end{minipage}%
\begin{minipage}[c]{1.5in}
\centering
\includegraphics[width = 1.5in]{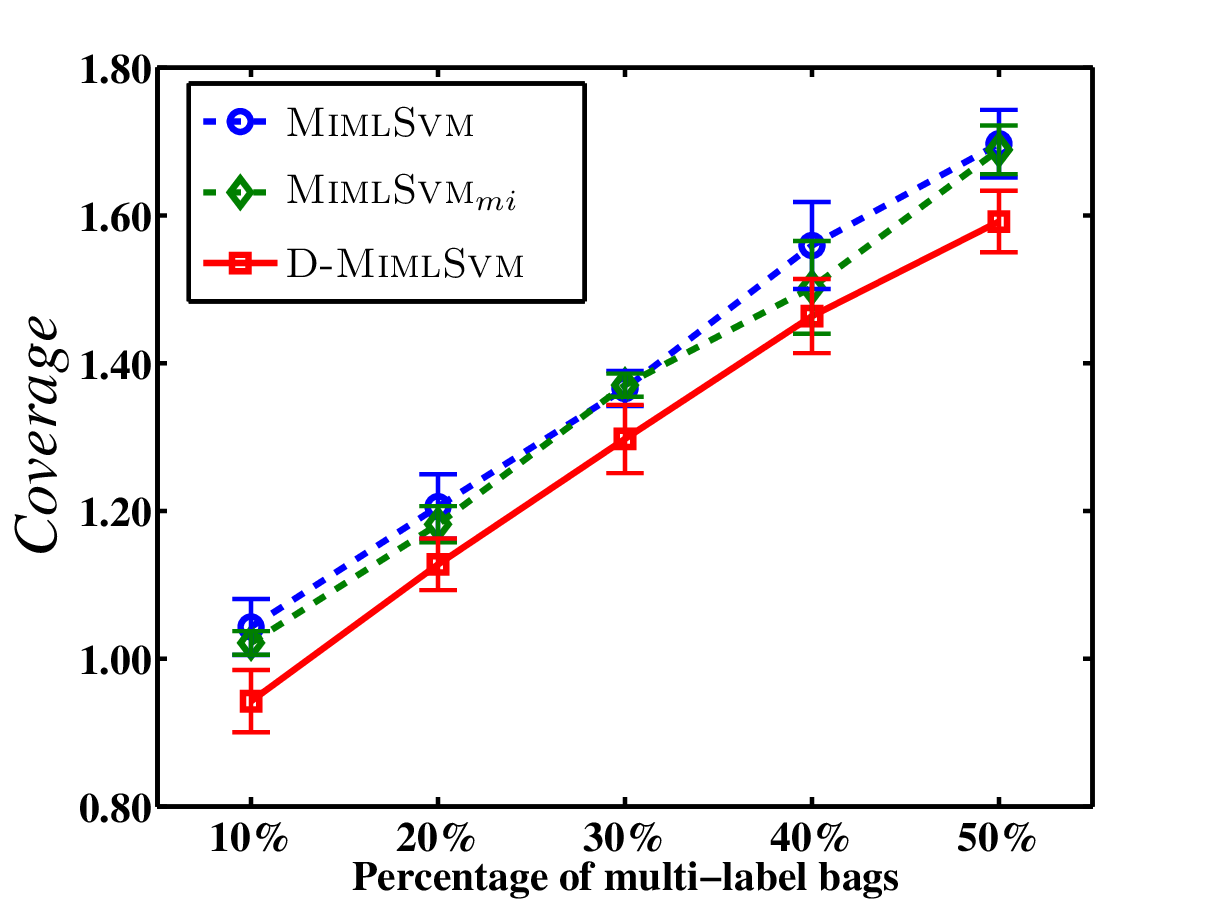}
\end{minipage}%
\begin{minipage}[c]{1.5in}
\centering
\includegraphics[width = 1.5in]{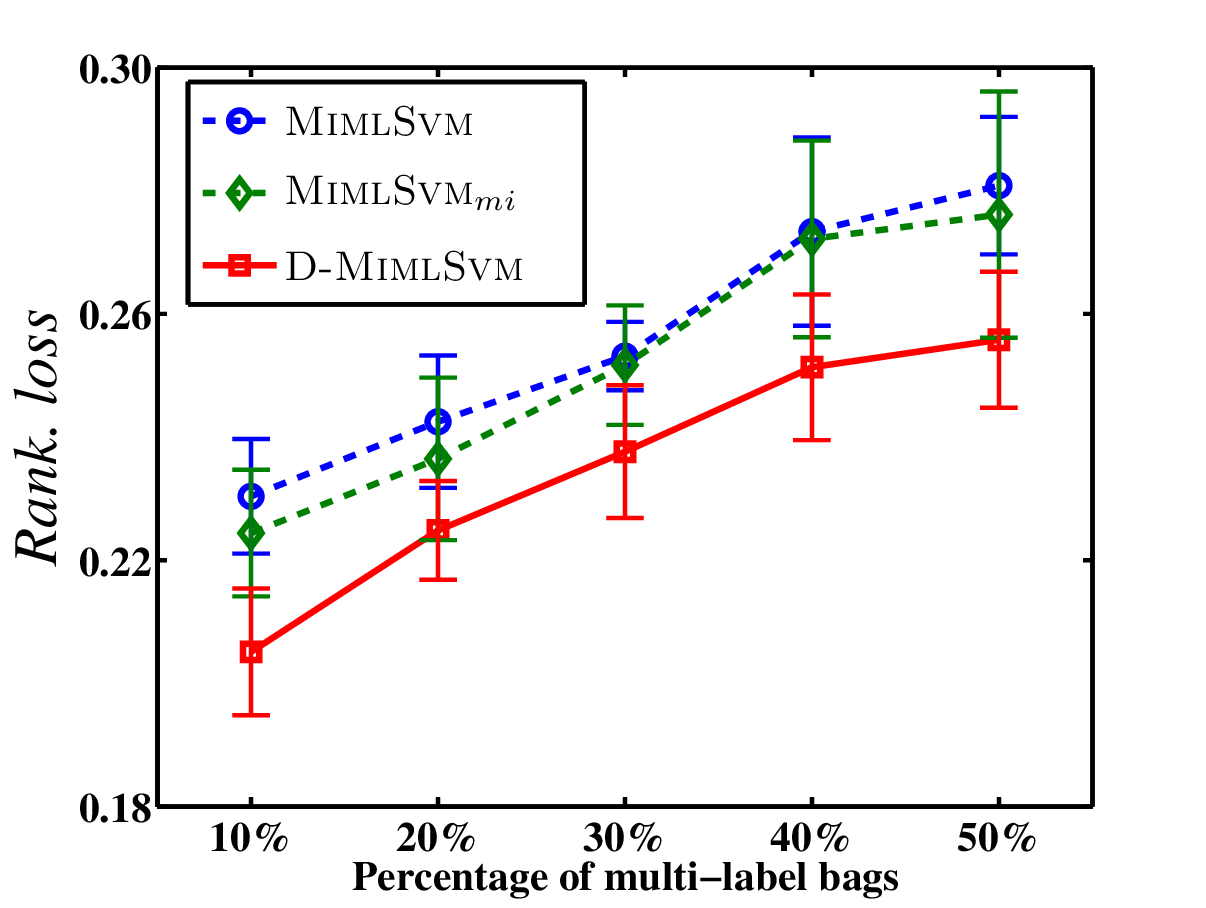}
\end{minipage}\\[+4pt]
\begin{minipage}[c]{1.5in}
\centering \mbox{{\footnotesize (a) \textit{hamming loss}}}
\end{minipage}%
\begin{minipage}[c]{1.5in}
\centering \mbox{{\footnotesize (b) \textit{one-error}}}
\end{minipage}%
\begin{minipage}[c]{1.5in}
\centering \mbox{{\footnotesize (c) \textit{coverage}}}
\end{minipage}%
\begin{minipage}[c]{1.5in}
\centering \mbox{{\footnotesize (d) \textit{ranking loss}}}
\end{minipage}\\[+5pt]
\begin{minipage}[c]{1.8in}
\centering
\includegraphics[width = 1.5in]{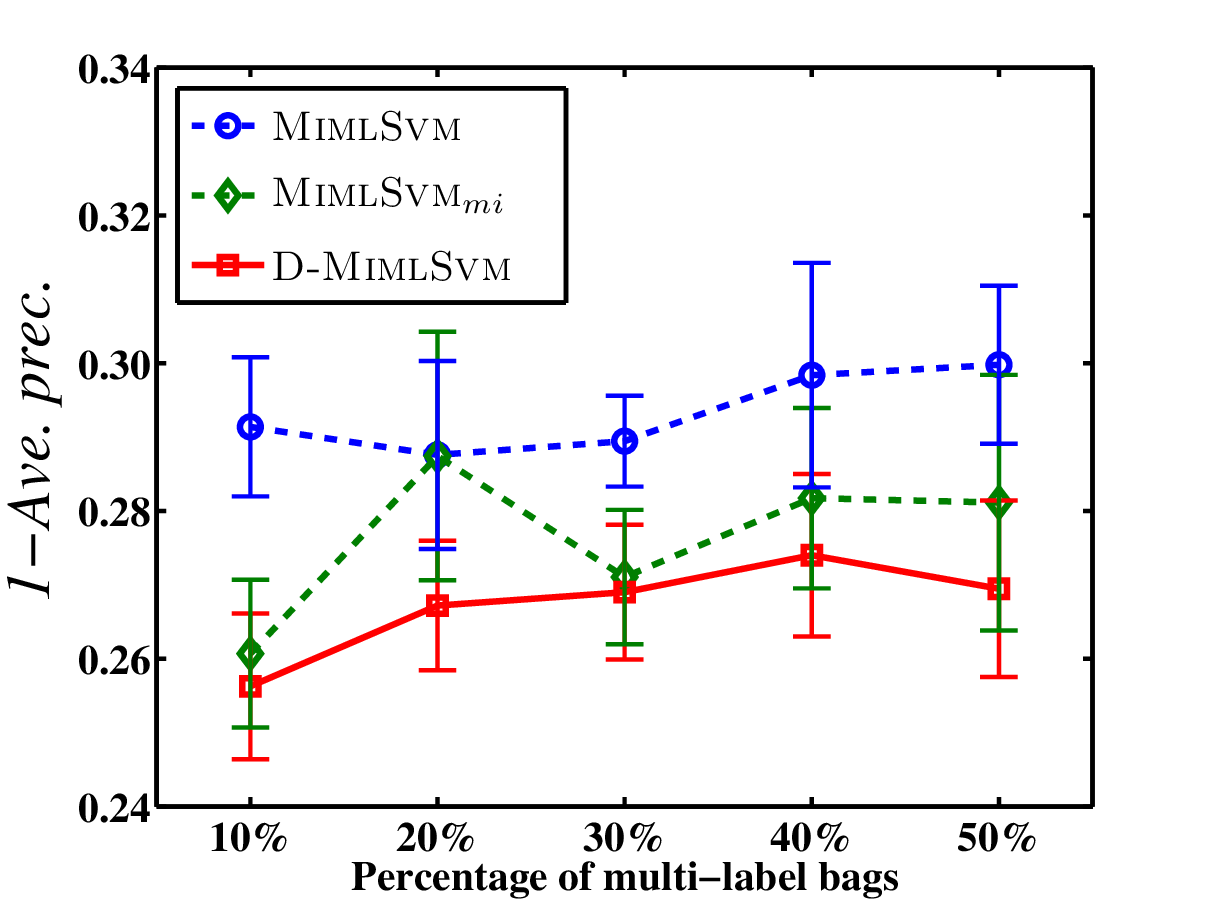}
\end{minipage}%
\begin{minipage}[c]{1.8in}
\centering
\includegraphics[width = 1.5in]{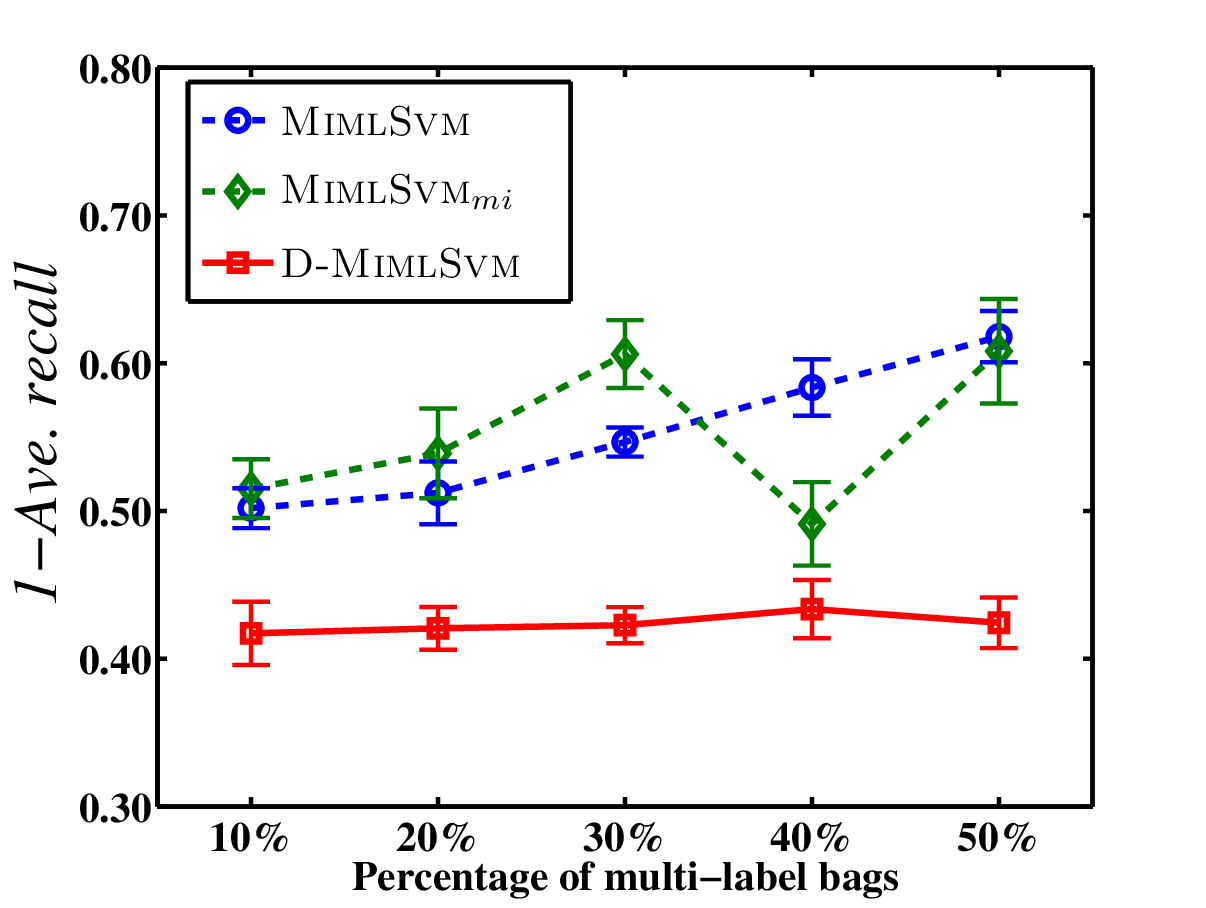}
\end{minipage}%
\begin{minipage}[c]{1.8in}
\centering
\includegraphics[width = 1.5in]{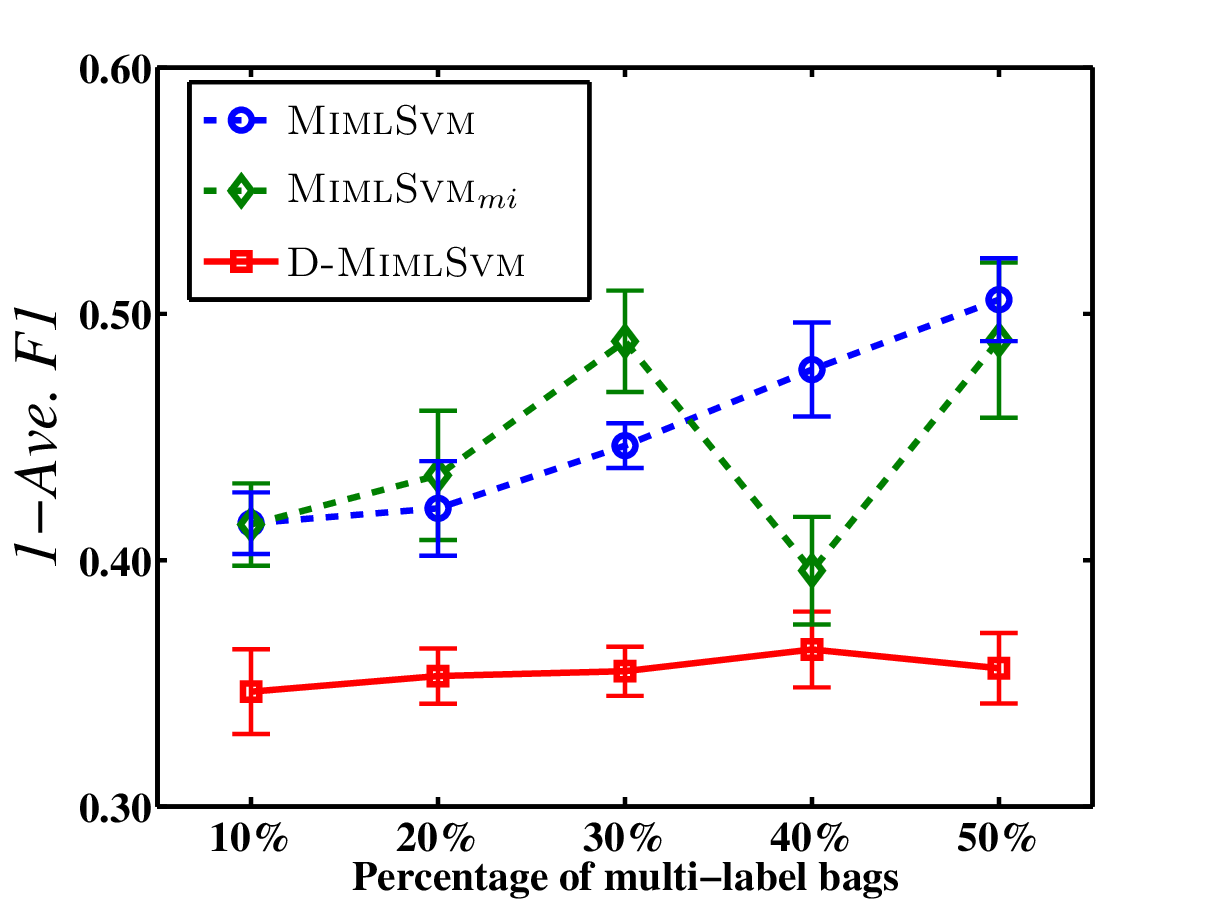}
\end{minipage}\\[+4pt]
\begin{minipage}[c]{1.8in}
\centering \mbox{{\footnotesize (e) $1-$ \textit{average precision}}}
\end{minipage}%
\begin{minipage}[c]{1.8in}
\centering \mbox{{\footnotesize (f) $1-$ \textit{average recall}}}
\end{minipage}%
\begin{minipage}[c]{1.8in}
\centering \mbox{{\footnotesize (g) $1-$ \textit{average F1}}}
\end{minipage}%
\caption{Results on the scene classification data set with different percentage of multi-label data.
The lower the curve, the better the performance.}\label{fig:scene-curves}\bigskip\bigskip
\end{figure}

\begin{figure}[!ht]
\centering
\begin{minipage}[c]{1.5in}
\centering
\includegraphics[width = 1.5in]{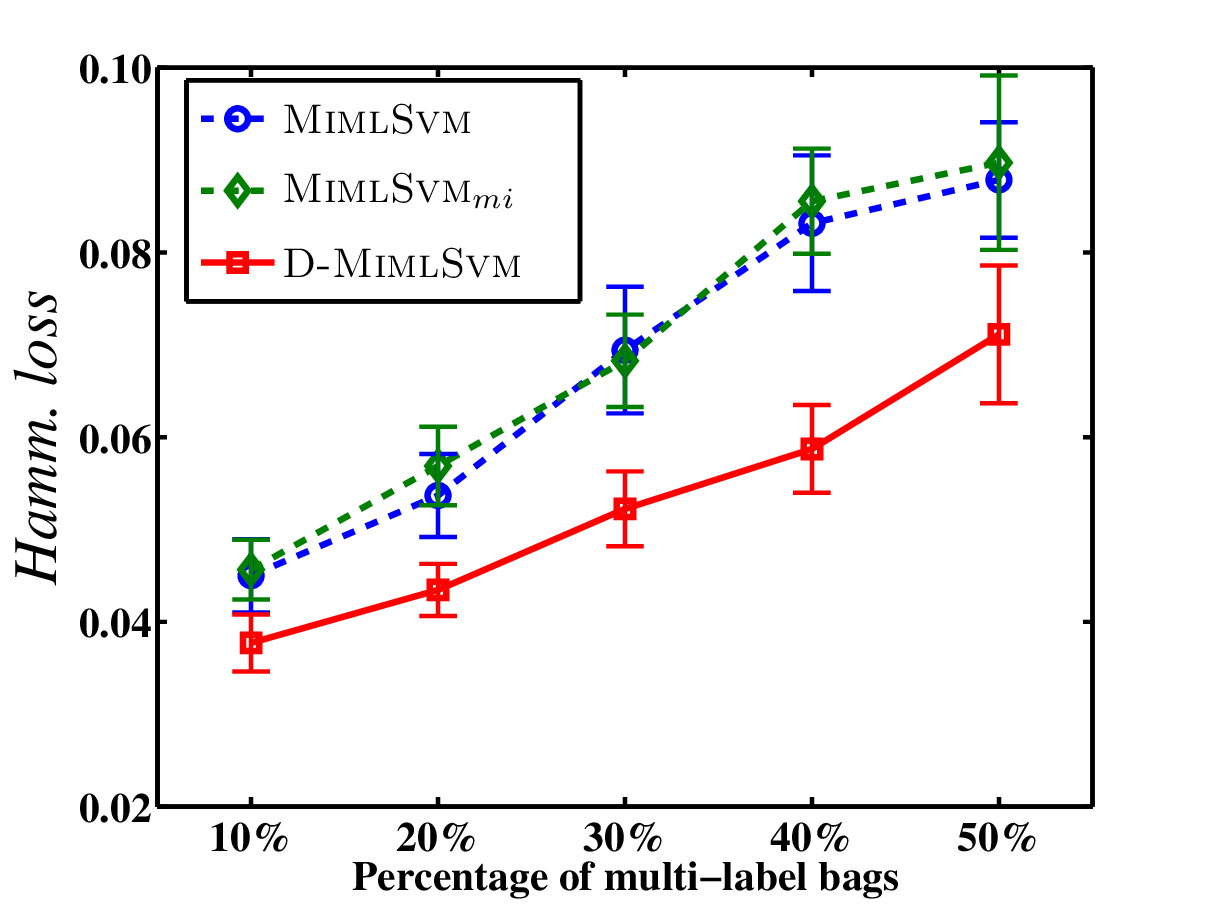}
\end{minipage}%
\begin{minipage}[c]{1.5in}
\centering
\includegraphics[width = 1.5in]{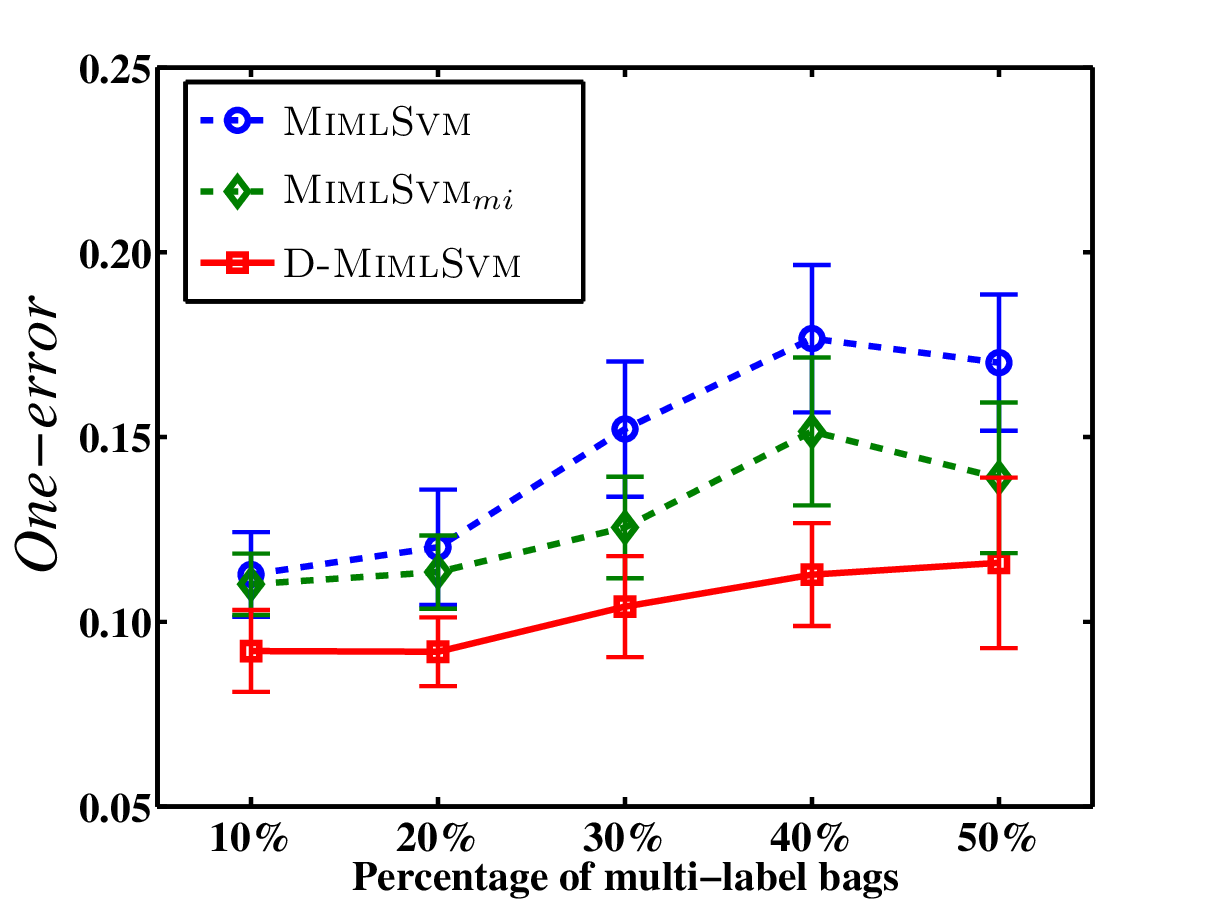}
\end{minipage}%
\begin{minipage}[c]{1.5in}
\centering
\includegraphics[width = 1.5in]{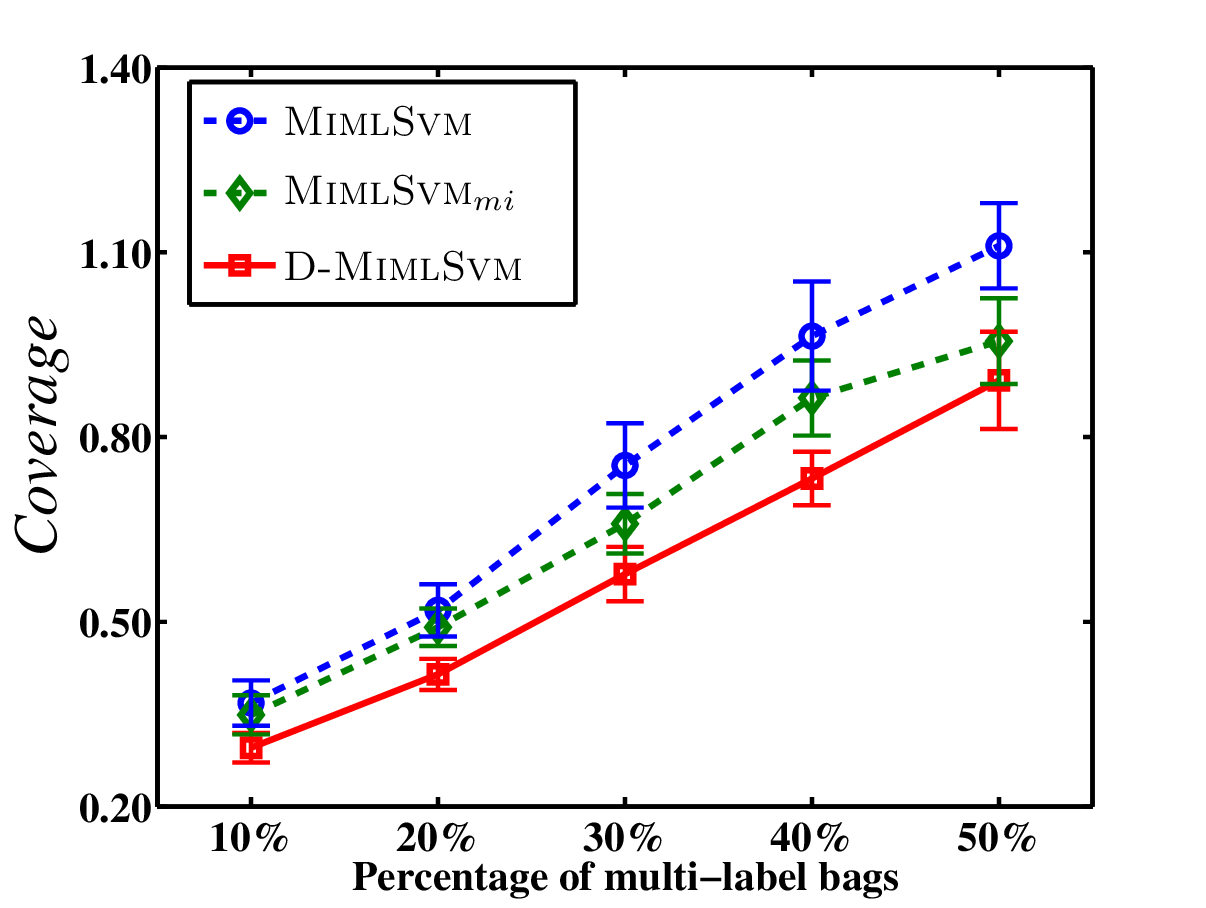}
\end{minipage}%
\begin{minipage}[c]{1.5in}
\centering
\includegraphics[width = 1.5in]{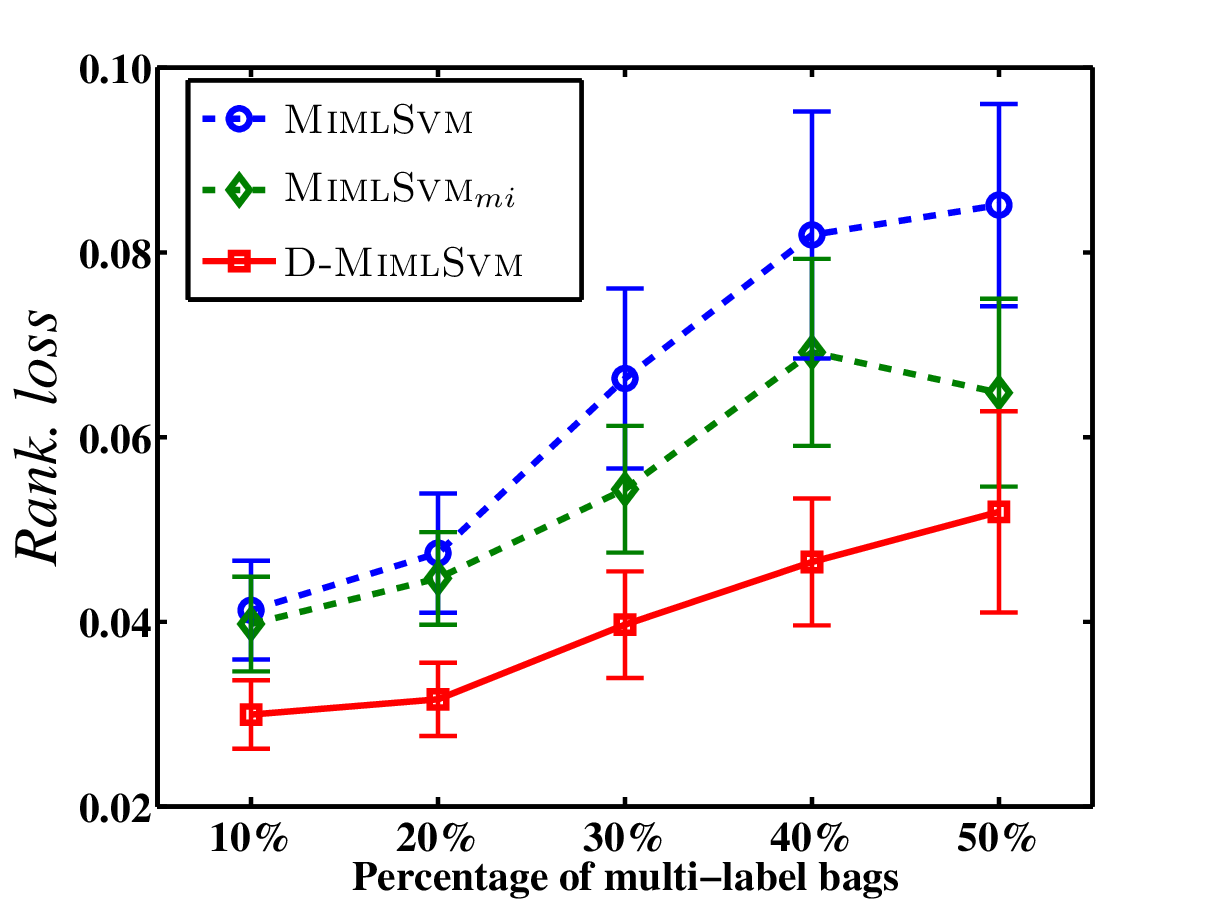}
\end{minipage}\\[+4pt]
\begin{minipage}[c]{1.5in}
\centering \mbox{{\footnotesize (a) \textit{hamming loss}}}
\end{minipage}%
\begin{minipage}[c]{1.5in}
\centering \mbox{{\footnotesize (b) \textit{one-error}}}
\end{minipage}%
\begin{minipage}[c]{1.5in}
\centering \mbox{{\footnotesize (c) \textit{coverage}}}
\end{minipage}%
\begin{minipage}[c]{1.5in}
\centering \mbox{{\footnotesize (d) \textit{ranking loss}}}
\end{minipage}\\[+5pt]
\begin{minipage}[c]{1.8in}
\centering
\includegraphics[width = 1.5in]{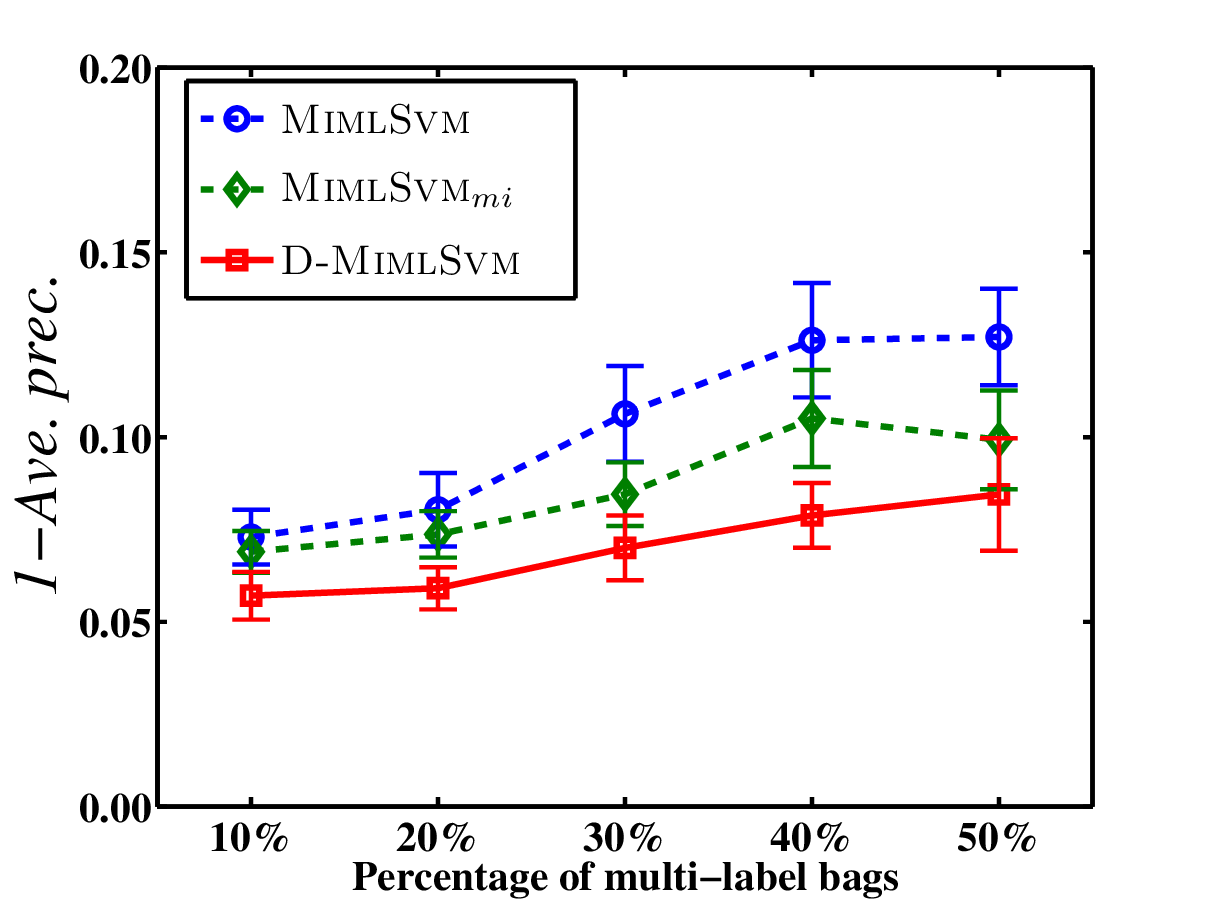}
\end{minipage}%
\begin{minipage}[c]{1.8in}
\centering
\includegraphics[width = 1.5in]{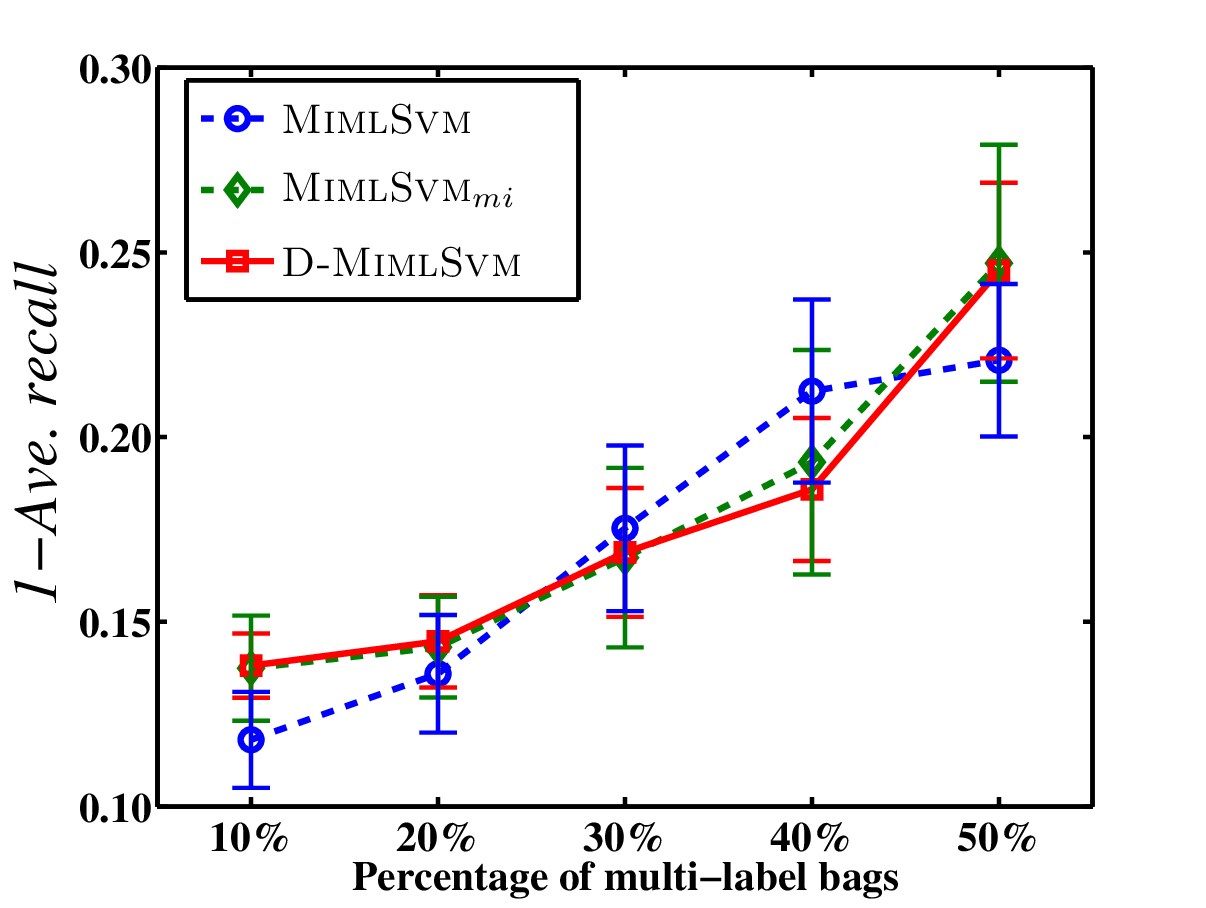}
\end{minipage}%
\begin{minipage}[c]{1.8in}
\centering
\includegraphics[width = 1.5in]{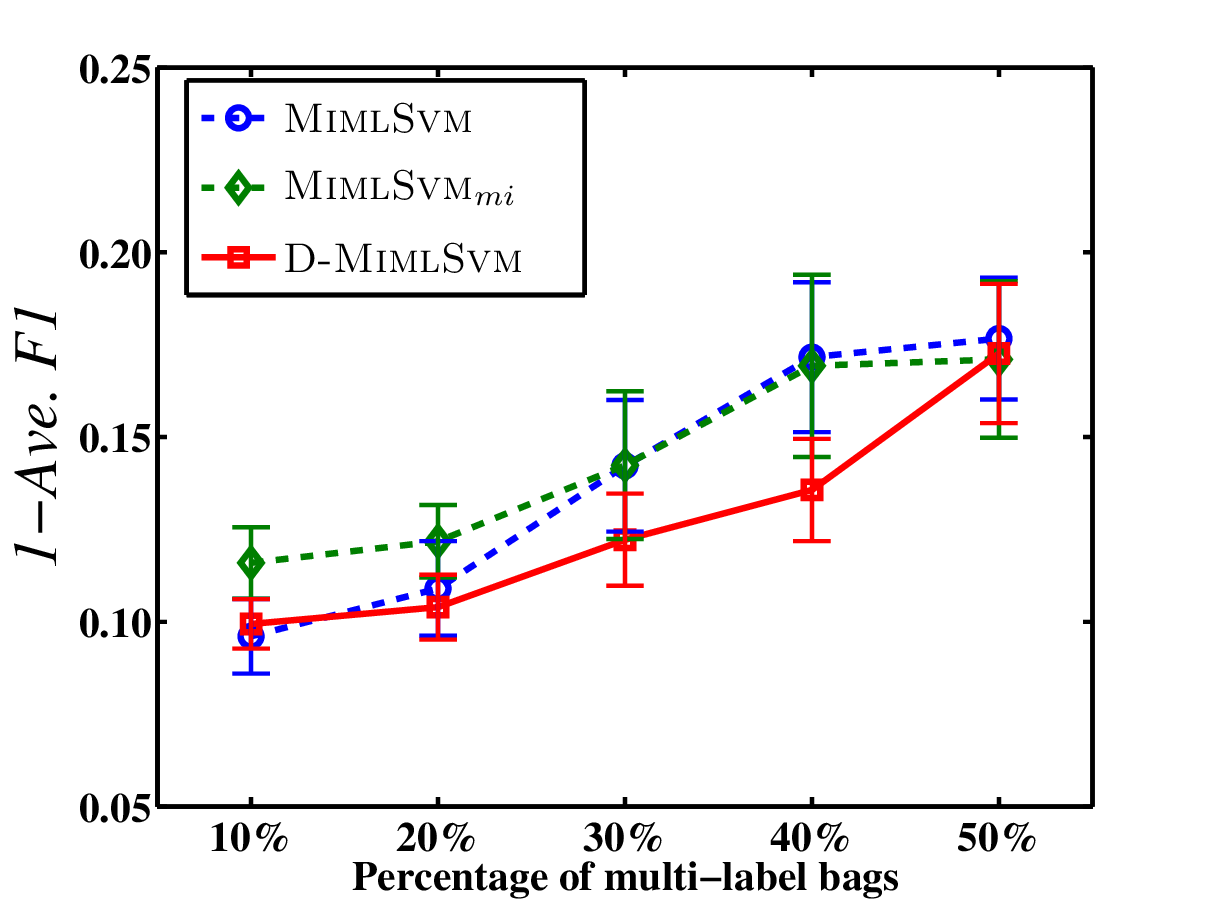}
\end{minipage}\\[+4pt]
\begin{minipage}[c]{1.8in}
\centering \mbox{{\footnotesize (e) $1-$ \textit{average precision}}}
\end{minipage}%
\begin{minipage}[c]{1.8in}
\centering \mbox{{\footnotesize (f) $1-$ \textit{average recall}}}
\end{minipage}%
\begin{minipage}[c]{1.8in}
\centering \mbox{{\footnotesize (g) $1-$ \textit{average F1}}}
\end{minipage}%
\caption{Results on the text categorization data set with different percentage of multi-label data.
The lower the curve, the better the
performance.}\label{fig:text-curves}\bigskip\bigskip
\end{figure}

As shown in Figs.~\ref{fig:scene-curves} and \ref{fig:text-curves}, the
performance of \textsc{D-MimlSvm} is better than those of \textsc{MimlSvm} and \textsc{MimlSvm$_{mi}$} in most cases. Specifically, pairwise $t$-tests with 95\% significance level disclose that: a) On the scene classification task, among all the 35 configurations (7 evaluation criteria $\times$ 5 percentages of multi-label bags), the performance of \textsc{D-MimlSvm} is superior to \textsc{MimlSvm} and \textsc{MimlSvm$_{mi}$} in 88\% and 80\% cases, comparable to them in 6\% and 20\% cases, and inferior to them in only 6\% and none cases; b) On the text categorization task, among all the 35 configurations, the performance of \textsc{D-MimlSvm} is superior to \textsc{MimlSvm} and \textsc{MimlSvm$_{mi}$} in 82\% and 82\% cases, comparable to them in 9\% and 18\% cases, and inferior to them in only 9\% and none cases. The results suggest that \textsc{D-MimlSvm} is a good choice for learning with moderate
number of MIML examples.

\subsection{Discussion}\vspace{-5mm}

The regularization framework presented in this section has an important assumption,
that is, all the class labels share some commonness, i.e., the ${\bm w}_0$ in
Eq.~\ref{eq:w0}. This assumption makes the regularization easier to realize, however,
it over-simplifies the real scenario. In fact, in real applications it is rare that all
class labels share some commonness; it is more typical that some class labels share
some commonness, but the commonness shared by different labels may be different. For
example, class label $y_1$ may share something with class label $y_2$, and $y_2$ may
share something with $y_3$, but maybe $y_1$ shares nothing with $y_3$. So, a more
reasonable assumption is that different pairs of labels share different things (or even
nothing). By considering this assumption, a more powerful method may be developed.

Actually, it is not difficult to modify the framework of Eq.~\ref{eq:framework} by
replacing the role of ${\bm w}_0$ by $\bm W$ whose element ${\bm W}_{ij}$ expresses the
relatedness between the $i$-th and $j$-th class labels, that is,
\begin{equation}\label{eq:newFramework}
\min \frac{1}{2T^2}\sum\limits_{i,j}\parallel {\bm w}_i - {\bm W}_{ij} \parallel ^2 +
\frac{1}{T^2}\sum\limits_{i,j} \mu_{ij}\parallel {\bm W}_{ij}\parallel ^2+\gamma {\bm
V} \ .
\end{equation}

Note that $\bm W$ is a tensor and ${\bm W}_{ij}$ is a vector.

To minimize Eq.~\ref{eq:newFramework}, taking derivative to ${\bm W}_{ij}$, we have
$$
-({\bm w}_i- {\bm W}_{ij})-({\bm w}_j- {\bm W}_{ji})+ 2\mu_{ij} {\bm W}_{ij}+ 2
\mu_{ji} {\bm W}_{ji} = 0 \ .
$$

Considering ${\bm W}_{ij} = {\bm W}_{ji}$ and $\mu_{ij}=\mu_{ji}$, we have
$$
- ({\bm w}_i - {\bm W}_{ij}) - ({\bm w}_j - {\bm W}_{ij})+ 4\mu_{ij} {\bm W}_{ij}=0 \ ,
$$
and so,
\begin{equation}\label{eq:Wij}
{\bm W}_{ij} = \frac{{\bm w}_i + {\bm w}_j} {4\mu_{ji}+2} \ .
\end{equation}

Put Eq.~\ref{eq:Wij} into Eq.~\ref{eq:newFramework}, we have
\begin{equation}\label{eq:newObj}
\min \frac{1}{2T^2} \sum\limits_{i,j} \parallel \frac{(4\mu_{ij}+1) {\bm w}_i - {\bm
w}_j} {4\mu_{ij}+2}
\parallel^2 + \frac{1}{T^2} \sum\limits_{i,j} \mu_{ij} \parallel \frac{{\bm w}_i + {\bm w}_j} {4\mu_{ij}+2} \parallel
^2 + \gamma {\bm V} \ .
\end{equation}

After simplification, Eq.~\ref{eq:newFramework} becomes
$$
\min \frac{1}{8T^2} \sum\limits_{i,j} \left(\frac{16 \mu_{ij}^2 + 10 \mu_{ij} +
1}{(2\mu_{ij}+1)^2} \parallel {\bm w}_i
\parallel^2 + \frac{2 \mu_{ij}+1} {(2\mu_{ij}+1)^2} \parallel {\bm w}_j \parallel^2 \right)
$$
$$
- \frac{1}{4T^2} \sum\limits_{i,j} \frac{2 \mu_{ij} + 1}{(2 \mu_{ij}+1)^2} \langle {\bm
w}_i, {\bm w}_j \rangle + \gamma {\bm V} \ .
$$

So, the new optimization task becomes
\begin{eqnarray}\label{eq:newopt}
\min\limits_{{\bm A}, {\bm \xi}, {\bm \delta}, {\bm b}} &&
\frac{1}{8T^2}\sum\limits_{i=1}^T\sum\limits_{j=1}^T
\left(\frac{16\mu_{ij}^2+10\mu_{ij}+1}{(2\mu_{ij}+1)^2} {\bm \alpha}_i^{'} {\bm K} {\bm
\alpha}_i+\frac{2\mu_{ij}+1}
{(2\mu_{ij}+1)^2}{\bm \alpha}_j^{'} {\bm K} {\bm \alpha}_j\right)\\
&& -\frac{1}{4T^2}\sum\limits_{i=1}^T \sum\limits_{j=1}^T \frac{2 \mu_{ij} +1}{(2
\mu_{ij}+1)^2}{\bm \alpha}^{'}_i {\bm K}{\bm
\alpha}_j+ \frac{\gamma}{mT} {\bm \xi}^{'} {\bm 1}+\frac{\gamma\lambda}{mT}{\bm \delta}^{'}{\bm 1} \nonumber\\
\mbox{s.t.} && y_{it}(\bm k_{\mathcal I(X_i)}'\bm \alpha_t+b_t) \ge 1 - \xi_{it}, \nonumber\\
            && \bm \xi \ge \bm 0, \nonumber\\
            && \bm k_{\mathcal I(\bm x_{ij})}'\bm \alpha_t-\delta_{it} \le \bm k_{\mathcal I(X_i)}'\bm \alpha_t, \nonumber\\
            && \bm k_{\mathcal I(X_i)}'\bm \alpha_t - \max\limits_{j=1,\cdots,n_i}\bm k_{\mathcal
               I(\bm x_{ij})}'\bm \alpha_t \le \delta_{it} . \nonumber
\end{eqnarray}

By solving Eq.~\ref{eq:newopt} we can get not only an MIML learner, but also some
understanding on the relatedness between pairs of labels from ${\bm W}_{ij}$, and some
understanding on the different importance of the ${\bm W}_{ij}$'s in determining the
concerned class label from $\mu_{ij}$'s; this may be very helpful for understanding the
complicated concepts underlying the task. Eq.~\ref{eq:newopt}, however, is difficult to
solve since it involves too many variables. Thus, how to exploit/understand the
pairwise relatedness between different pairs of labels remains an open problem.

\section{Solving Single-Instance Multi-Label Problems through MIML Transformation}\label{sec:MLLtoMIML}\vspace{-4mm}

The previous sections show that when we have access to the real objects and are able to
represent complicated objects as MIML examples, using the MIML framework is beneficial.
However, in many practical tasks we are given observational data where each object has
already been represented by a single instance, and we do not have access to the real
objects. In such case, we cannot capture more information from the real objects using
the MIML representation. Even in this situation, however, MIML is still useful. Here we
propose the \textsc{InsDif} (i.e., INStance DIFferentiation) algorithm which transforms
single-instance multi-label examples into MIML examples to exploit the power of MIML.

\subsection{\textsc{InsDif}}\vspace{-5mm}

For an object associated with multiple class labels, if it is described by only a
single instance, the information corresponding to these labels are mixed and thus
difficult to learn. The basic assumption of \textsc{InsDif} is that the spatial
distribution of instances with different labels encodes information helpful for
discriminating these labels, and such information will become more explicit by breaking
the single-instances into a number of instances each corresponds to one label.

\textsc{InsDif} is a two-stage algorithm, which is based on \textit{instance
differentiation}. In the first stage, \textsc{InsDif} transforms each example into a
bag of instances, by deriving one instance for each class label, in order to explicitly
express the ambiguity of the example in the input space; in the second stage, an MIML
learner is utilized to learn from the transformed data set. For the consistency with
our previous description of the algorithm \cite{Zhang:Zhou2007}, in the current version
of \textsc{InsDif} we use a two-level classification strategy, but note that other MIML
algorithms such as \textsc{D-MimlSvm} can also be applied.

Using the same denotation as that in Sections~\ref{sec:MIML} and \ref{sec:MIMLalgos},
that is, given data set $S = \{({\bm x_1}, Y_1), ({\bm x_2}, Y_2), \cdots, ({\bm x_m},
Y_m)\}$, where ${\bm x_{i}} \in \mathcal{X}$ is an instance and $Y_i \subseteq
\mathcal{Y}$ a set of labels $\{y_{i1}, y_{i2}, \cdots, y_{i,l_i}\}$, $y_{ik} \in
\mathcal{Y}$ $(k=1,2,\cdots,l_i)$. Here $l_i$ denotes the number of labels in $Y_i$.

In the first stage, \textsc{InsDif} derives a prototype vector ${\bm v}_l$ for each
class label $l \in \mathcal{Y}$ by averaging all the training instances belonging to
$l$, i.e.,
\begin{equation}\label{eq:proto}
{\bm v}_l= {1 \over |S_l|} \left(\sum\limits_{{\bm x}_i \in S_l}{\bm x}_i\right),
\end{equation}
where
$$
S_l=\{{\bm x}_i|\{{\bm x}_i,Y_i\}\in S,\ l\in Y_i\}, \ l\in\mathcal{Y}.
$$

Here ${\bm v}_l$ can be approximately regarded as a profile-style vector describing
common characteristics of the class $l$. Actually, this kind of prototype vectors have
already shown their usefulness in solving text categorization problems. For example,
the \textsc{Rocchio} method \cite{Ittner:Lewis:Ahn1995,Sebastiani2002} forms a
prototype vector for each class by averaging all the documents (represented by weight
vectors) of this class, and then classifies the test document by calculating the
dot-products between the weight vector representing the document and each of the
prototype vectors. Here we use such prototype vectors to facilitate bag
generation. After obtaining the prototype vectors, each example ${\bm x}_i$ is
re-represented by a bag of instances $B_i$, where each instance in $B_i$ expresses the
difference between ${\bm x}_i$ and a prototype vector according to
Eq.~\ref{eq:instTobag}. In this way, each example is transformed into a bag whose size
equals to the number of class labels.
\begin{equation}
\centering B_i=\{{\bm x}_i-{\bm v}_l|l\in\mathcal{Y}\} \label{eq:instTobag}
\end{equation}

In fact, such a process attempts to exploit the spatial distribution since ${\bm
x}_i-{\bm v}_l$ in Eq.~\ref{eq:instTobag} is a kind of distance between ${\bm x}_i$ and
${\bm v}_l$. The transformation can also be realized in other ways. For example, other
than referring to the prototype vector of each class, one could also consider the following approach. For each possible class $l$, identify the $k$-nearest neighbors of ${\bm
x}_i$ among training instances that have $l$ as a proper label. Then, the mean vector
of these neighbors can be regarded as an instance in the bag. Note that the
transformation of a single instance into a bag of instances can be realized as a
general pre-processing method which can be plugged into many learning systems.

In the second stage, \textsc{InsDif} learns from the transformed training set $ S^* =
\{(B_1,Y_1),$ $(B_2,Y_2), \cdots, (B_m,Y_m)\}$. This task can be realized by any MIML
learning algorithm. By default we use the \textsc{MimlNn} algorithm introduced in
Section~\ref{sec:MIMLscene}. The use of other MIML algorithms for this stage will also
be studied in the next section.

The pseudo-code of \textsc{InsDif} is summarized in Appendix A (Table~\ref{table:InsDif}). In the
first stage (Steps 1 to 2), \textsc{InsDif} transforms each example into a bag of
instances by querying the class prototype vectors. In the second stage (Step 3), an
MIML algorithm is used to learn from the transformed data set. A test example ${\bm x}^*$
is then transformed into the corresponding bag representation $B^*$ and then fed to the
learned MIML model.

\subsection{Experiments}\label{sec:SIMLexp}\vspace{-5mm}

We compare \textsc{InsDif} with several state-of-the-art multi-label learning
algorithms, including \textsc{AdtBoost.MH} \cite{DeComite:Gilleron:Tommasi2003},
\textsc{RankSvm} \cite{Elisseeff:Weston2002}, \textsc{MlSvm}
\cite{Boutell:Luo:Shen:Brown2004}, \textsc{Ml-$k$nn} \cite{Zhang:Zhou2007pr} and
\textsc{Cnmf } \cite{Liu:Jin:Yang2006}; these algorithms have been introduced briefly
in Section~\ref{sec:survey}. In addition, by using \textsc{MimlBoost}, \textsc{MimlSvm}
and \textsc{MimlSvm}$_{mi}$ respectively to replace \textsc{MimlNn} for realizing the
second stage of \textsc{InsDif}, we get three variants of \textsc{InsDif}, i.e., {\sc
InsDif}$_{\rm MIMLBOOST}$, {\sc InsDif}$_{\rm MIMLSVM}$ and {\sc InsDif}$_{{\rm
MIMLSVM}_{mi}}$. These variants are also evaluated for comparison.

Note that the experiments here are very different from that in
Sections~\ref{sec:MIMLexp} and \ref{sec:Dmimlexp}. In Sections~\ref{sec:MIMLexp} and
\ref{sec:Dmimlexp}, it is assumed that the data are MIML examples; while in this
section, it is assumed that we are given observational data where each real object has
already been represented as a single instance. In other words, in this section we are
trying to learn from single-instance multi-label examples, and therefore the
experimental data sets are different from those used in Sections~\ref{sec:MIMLexp} and
\ref{sec:Dmimlexp}.

\subsubsection{Yeast Gene Functional Analysis}\label{sec:Mllyeast}

The task here is to predict the gene functional classes of the Yeast
\textit{Saccharomyces cerevisiae}, which is one of the best studied organisms.
Specifically, the Yeast data set investigated in
\cite{Elisseeff:Weston2002,Zhang:Zhou2007pr} is studied. Each gene is represented by a
103-dimensional feature vector generated by concatenating a gene expression vector and
the corresponding phylogenetic profile. Each 79-element gene expression vector reflects
the expression levels of a particular gene under two different experimental conditions,
while the phylogenetic profile is a Boolean string, each bit indicating whether the
concerned gene has a close homolog in the corresponding genome. Each gene is associated
with a set of functional classes whose maximum size can be potentially more than 190.
Elisseeff and Weston \cite{Elisseeff:Weston2002} have pre-processed the data set where
only the known structure of the functional classes are used. In fact, the whole set of
functional classes is structured into hierarchies up to 4 levels deep. \footnote{See
http://mips.gsf.de/proj/yeast/catalogues/funcat/ for more details.} Illustrations on the first level of the hierarchy used to generate the Yeast data can be found in \cite{Elisseeff:Weston2002,Zhang:Zhou2006tkde,Zhang:Zhou2007pr}. The resulting multi-label data set
contains 2,417 genes, fourteen possible class labels and the average number of labels
for each gene is $4.24\pm1.57$.


For \textsc{InsDif}, the parameter $M$ is set to be 20\% of the size of training set;
The performance of this algorithm with different $M$ settings is shown in Appendix B (Fig.~\ref{fig:Msettings}), it can be found that its performance is not sensitive to the setting of $M$. The boosting rounds of \textsc{AdtBoost.MH} are
set to 25; The performance of this algorithm at different boosting rounds is shown in Appendix B (Fig. \ref{fig:Roundsettings}), it can be observed that after this round
its performance has become stable. (Similar observations are also
found in Section~\ref{sec:Mllweb}.)
For \textsc{RankSvm}, polynomial kernel with degree 8 is used as suggested in \cite{Elisseeff:Weston2002}. For \textsc{MlSvm}, a Gaussian kernel is used with default \textsc{Libsvm} setting for kernel width (i.e. $\frac{1}{\#\ features}$). For \textsc{Cnmf}, a normalized Gaussian kernel as recommended in \cite{Liu:Jin:Yang2006} is used to compute the pairwise class similarity. For \textsc{Ml-$k$nn}, the number of nearest neighbors considered is set to 10. The criteria introduced in Section~\ref{sec:MLLcriteria} are
used to evaluate the learning performance. Ten-fold cross-validation is conducted on
this data set and the results are summarized in
Table~\ref{table:yeastResult},\footnote{\textit{Hamming loss}, \textit{average recall}
and \textit{average F1} are not available for \textsc{Cnmf}; \textit{ranking loss},
\textit{average recall} and \textit{average F1} are not available for
\textsc{AdtBoost.MH}. The performance of {\sc InsDif}$_{\rm MIMLBOOST}$ is not reported
since this algorithm did not terminate within reasonable time on this data.} where the
best performance on each criterion has been highlighted in boldface.

\renewcommand{\baselinestretch}{.95}
\small\normalsize

\begin{table}[!t]
\caption{Results (mean$\pm$std.) on Yeast gene data set (`$\downarrow$' indicates `the
smaller the better'; `$\uparrow$' indicates `the larger the
better').}\smallskip\smallskip\label{table:yeastResult}\scriptsize \tabcolsep 0.07in
\begin{center}
\begin{tabular}{lccccccc}
\hline\noalign{\smallskip}
\raisebox{-2mm}{Compared} & \multicolumn{7}{c}{Evaluation Criteria} \\
\cline{2-8}\noalign{\smallskip}
\raisebox{2mm}{Algorithms} & \multicolumn{1}{c}{$hloss$ $^{\downarrow}$} & \multicolumn{1}{c}{$one\mbox{-}error$ $^{\downarrow}$} & \multicolumn{1}{c}{$coverage$ $^{\downarrow}$} & \multicolumn{1}{c}{$rloss$ $^{\downarrow}$} & \multicolumn{1}{c}{$aveprec$ $^{\uparrow}$} & \multicolumn{1}{c}{$avgrecl$ $^{\uparrow}$} & \multicolumn{1}{c}{$avgF1$ $^{\uparrow}$} \\
\hline\noalign{\smallskip}
\multicolumn{1}{l}{\textsc{InsDif}} & \multicolumn{1}{c}{\bf .189$\pm$.010} & \multicolumn{1}{c}{\bf .214$\pm$.030} & \multicolumn{1}{c}{6.288$\pm$0.240} & \multicolumn{1}{c}{\bf .163$\pm$.017} & \multicolumn{1}{c}{\bf .774$\pm$.019} & \multicolumn{1}{c}{{.602$\pm$.026}} & \multicolumn{1}{c}{{.677$\pm$.023}}\\
\multicolumn{1}{l}{{\sc InsDif}$_{\rm MIMLSVM}$} & \multicolumn{1}{c}{\bf .189$\pm$.009} & \multicolumn{1}{c}{.232$\pm$.040} & \multicolumn{1}{c}{6.625$\pm$0.261} & \multicolumn{1}{c}{.179$\pm$.015} & \multicolumn{1}{c}{.763$\pm$.021} & \multicolumn{1}{c}{{.591$\pm$.023}} & \multicolumn{1}{c}{{.666$\pm$.022}}\\
\multicolumn{1}{l}{{\sc InsDif}$_{{\rm MIMLSVM}_{mi}}$} &
\multicolumn{1}{c}{.196$\pm$.011} & \multicolumn{1}{c}{.238$\pm$.043} &
\multicolumn{1}{c}{6.396$\pm$0.206} & \multicolumn{1}{c}{.172$\pm$.012} &
\multicolumn{1}{c}{.765$\pm$.019} & \multicolumn{1}{c}{{\bf .655$\pm$.024}} &
\multicolumn{1}{c}{{\bf .706$\pm$.017}}\\\hline
\multicolumn{1}{l}{\textsc{AdtBoost.MH}}  & \multicolumn{1}{c}{.212$\pm$.008} & \multicolumn{1}{c}{.247$\pm$.029} & \multicolumn{1}{c}{6.385$\pm$0.151} & \multicolumn{1}{c}{N/A} & \multicolumn{1}{c}{.739$\pm$.022} & \multicolumn{1}{c}{N/A} & \multicolumn{1}{c}{N/A}\\
\multicolumn{1}{l}{\textsc{RankSvm}} & \multicolumn{1}{c}{.207$\pm$.013} & \multicolumn{1}{c}{.243$\pm$.039} & \multicolumn{1}{c}{7.090$\pm$0.502} & \multicolumn{1}{c}{.195$\pm$.021} & \multicolumn{1}{c}{.750$\pm$.026} & \multicolumn{1}{c}{.500$\pm$.047} & \multicolumn{1}{c}{.600$\pm$.041}\\
\multicolumn{1}{l}{\textsc{MlSvm}} & \multicolumn{1}{c}{.199$\pm$.009} & \multicolumn{1}{c}{.227$\pm$.032} & \multicolumn{1}{c}{7.220$\pm$0.338} & \multicolumn{1}{c}{.201$\pm$.019} & \multicolumn{1}{c}{.749$\pm$.021} & \multicolumn{1}{c}{.572$\pm$.023} & \multicolumn{1}{c}{.649$\pm$.022}\\
\multicolumn{1}{l}{\textsc{Ml-$k$nn}} & \multicolumn{1}{c}{.194$\pm$.010} & \multicolumn{1}{c}{.230$\pm$.030} & \multicolumn{1}{c}{\bf 6.275$\pm$0.240} & \multicolumn{1}{c}{.167$\pm$.016} & \multicolumn{1}{c}{.765$\pm$.021} & \multicolumn{1}{c}{.574$\pm$.022} & \multicolumn{1}{c}{.656$\pm$.021}\\
\multicolumn{1}{l}{\textsc{Cnmf}} & \multicolumn{1}{c}{N/A} & \multicolumn{1}{c}{.354$\pm$.184} & \multicolumn{1}{c}{7.930$\pm$1.089} & \multicolumn{1}{c}{.268$\pm$.062} & \multicolumn{1}{c}{.668$\pm$.093} & \multicolumn{1}{c}{N/A} & \multicolumn{1}{c}{N/A}\\
\hline
\end{tabular}
\end{center}\bigskip
\end{table}

\renewcommand{\baselinestretch}{1.4}
\small\normalsize

Table~\ref{table:yeastResult} shows that {\sc InsDif} and its variants achieve good
performance on the Yeast gene functional data set. Pairwise $t$-tests with 95\%
significance level disclose that: a) {\sc InsDif} is significantly better than all the
compared multi-label learning algorithms (i.e., the second part of
Table~\ref{table:yeastResult}) on all criteria, except that on \textit{coverage} it is
worse than {\sc Ml-$k$nn} but the difference is not statistically significant;\footnote{Note that our implementation of \textsc{RankSvm} was obtained with the help of the authors of \cite{Elisseeff:Weston2002}, yet our results are somewhat worse than the best results reported in \cite{Elisseeff:Weston2002}. We think that the performance gap may be caused by the minor implementation differences and the different experimental data partitions. Nevertheless, it is worth mentioning that the results of \textsc{InsDif} are better than the best results of \textsc{RankSvm} in \cite{Elisseeff:Weston2002} in terms of \emph{hamming loss}, \emph{one-error} and \emph{average precision}, and as same as the best results of \textsc{RankSvm} in \cite{Elisseeff:Weston2002} in terms of \emph{ranking loss}.} b) {\sc
InsDif}$_{\rm MIMLSVM}$ is significantly better than the compared multi-label learning
algorithms for more than 68$\%$ cases, and is significantly inferior to them for less
than 11$\%$ cases; c) {\sc InsDif}$_{{\rm MIMLSVM}_{mi}}$ is significantly better than
the compared multi-label learning algorithms for more than 65$\%$ cases, and is never
significantly inferior to them. Specifically, {\sc InsDif}$_{{\rm MIMLSVM}_{mi}}$
outperforms all the compared algorithms in terms of \emph{average recall} and
\emph{average F1}. It is noteworthy that \textsc{Cnmf} performs quite poorly compared to
other algorithms although it has used test set information. The reason may be that the
key assumption of \textsc{Cnmf}, i.e., two examples with high similarity in the input
space tend to have large overlap in the output space, does not hold on this gene data
since there are some genes whose functions are quite different but the physical
appearances are similar.

Overall, results on the Yeast gene functional analysis task suggest that MIML can be
useful when we are given observational data where each complicated object has already
been represented by a single instance.

\subsubsection{Web Page Categorization}\label{sec:Mllweb}

The web page categorization task has been studied in
\cite{Kazawa:Izumitani:Taira:Maeda2005,Ueda:Saito2003,Zhang:Zhou2007pr}. The web pages
were collected from the ``yahoo.com'' domain and then divided into 11 data sets based
on Yahoo's top-level categories.\footnote{Data set available at
http://www.kecl.ntt.co.jp/as/members/ueda/yahoo.tar.gz.} After that, each page is
classified into a number of Yahoo's second-level subcategories. Each data set contains
2,000 training documents and 3,000 test documents. The simple term selection method
based on \textit{document frequency} (the number of documents containing a specific
term) was applied to each data set to reduce the dimensionality. Actually, only 2\%
words with the highest document frequency were retained in the final vocabulary. Other
term selection methods such as \textit{information gain} and \textit{mutual
information} can also be adopted. After term selection, each document in the data set
is described as a feature vector using the ``\textit{Bag-of-Words}" representation,
i.e., each feature expresses the number of times a vocabulary word appearing in the
document.

Characteristics of the web page data sets are summarized in Appendix C (Table~\ref{table:yahoo}).
Compared to the Yeast data in Section~\ref{sec:Mllyeast}, here the instances are
represented by much higher-dimensional feature vectors and a large portion of them
(about 20-45\%) are multi-labeled. Moreover, here the number of categories
(21-40) are much larger and many of them are \textit{rare} categories (about
20-55\%). So, the web page data sets are more difficult than the Yeast data to
learn.

The parameter settings are similar as those in Section~\ref{sec:Mllyeast}. That is, for \textsc{InsDif}, the parameter $M$ is set to be 20\% of the size of training set; the boosting rounds of \textsc{AdtBoost.MH} are
set to 25; for \textsc{RankSvm}, polynomial kernel is used where polynomial degrees of 2 to 9 are considered as in \cite{Elisseeff:Weston2002} and chosen by hold-out tests on training sets; for \textsc{MlSvm} and \textsc{Cnmf}, linear and Gaussian kernel are used respectively; for \textsc{Ml-$k$nn}, the number of nearest neighbors considered is set to 10.

Results of the eleven data sets are shown in Appendix C (Fig.~\ref{fig:11yahoo}), and the average
results are summarized in Table~\ref{table:yahooResult} where the best performance on
each criterion has been highlighted in boldface.\footnote{The performance of {\sc
InsDif}$_{\rm MIMLBOOST}$ and {\sc InsDif}$_{{\rm MIMLSVM}_{mi}}$ are not reported
since these algorithms did not terminate within reasonable time on this data. Note that though the significant differences between some numbers in the table might be subtle at the first glance (e.g., {\sc InsDif} vs. {\sc RankSvm} in terms of \textit{one-error}), statistical tests based on detailed information (in online supplementary file) justify the significance.}


\renewcommand{\baselinestretch}{.95}
\small\normalsize

\begin{table}[!t]
\caption{Results (mean$\pm$std.) on eleven web page categorization data sets
(`$\downarrow$' indicates `the smaller the better'; `$\uparrow$' indicates `the larger
the better').}\smallskip\smallskip\label{table:yahooResult}\tabcolsep 0.07in\scriptsize
\begin{center}
\begin{tabular}{lccccccc}
\hline\noalign{\smallskip}
\raisebox{-2mm}{Compared} & \multicolumn{7}{c}{Evaluation Criteria} \\
\cline{2-8}\noalign{\smallskip}
\raisebox{2mm}{Algorithms} & \multicolumn{1}{c}{$hloss$ $^{\downarrow}$} & \multicolumn{1}{c}{$one\mbox{-}error$ $^{\downarrow}$} & \multicolumn{1}{c}{$coverage$ $^{\downarrow}$} & \multicolumn{1}{c}{$rloss$ $^{\downarrow}$} & \multicolumn{1}{c}{$aveprec$ $^{\uparrow}$} & \multicolumn{1}{c}{$avgrecl$ $^{\uparrow}$} & \multicolumn{1}{c}{$aveF1$ $^{\uparrow}$} \\
\hline\noalign{\smallskip}
\multicolumn{1}{l}{\textsc{InsDif}} & \multicolumn{1}{c}{\bf .039$\pm$.013} & \multicolumn{1}{c}{.381$\pm$.118} & \multicolumn{1}{c}{4.545$\pm$1.285} & \multicolumn{1}{c}{\bf .102$\pm$.037} & \multicolumn{1}{c}{\bf .686$\pm$.091} & \multicolumn{1}{c}{.377$\pm$.163} & \multicolumn{1}{c}{.479$\pm$.154}\\
\multicolumn{1}{l}{{\sc InsDif}$_{\rm MIMLSVM}$} & \multicolumn{1}{c}{.043$\pm$.015} &
\multicolumn{1}{c}{.395$\pm$.119} & \multicolumn{1}{c}{6.823$\pm$1.623} &
\multicolumn{1}{c}{.166$\pm$.045} & \multicolumn{1}{c}{.653$\pm$.093} &
\multicolumn{1}{c}{\bf .501$\pm$.105} & \multicolumn{1}{c}{\bf .566$\pm$.102}\\\hline
\multicolumn{1}{l}{\textsc{AdtBoost.MH}}  & \multicolumn{1}{c}{.044$\pm$.014} & \multicolumn{1}{c}{.477$\pm$.144} & \multicolumn{1}{c}{4.177$\pm$1.261} & \multicolumn{1}{c}{N/A} & \multicolumn{1}{c}{.621$\pm$.108} & \multicolumn{1}{c}{N/A} & \multicolumn{1}{c}{N/A}\\
\multicolumn{1}{l}{\textsc{RankSvm}} & \multicolumn{1}{c}{.043$\pm$.014} & \multicolumn{1}{c}{.424$\pm$.135} & \multicolumn{1}{c}{7.228$\pm$2.442} & \multicolumn{1}{c}{.182$\pm$.057} & \multicolumn{1}{c}{.621$\pm$.108} & \multicolumn{1}{c}{.252$\pm$.172} & \multicolumn{1}{c}{.345$\pm$.177}\\
\multicolumn{1}{l}{\textsc{MlSvm}} & \multicolumn{1}{c}{.042$\pm$.015} & \multicolumn{1}{c}{{\bf .375$\pm$.119}} & \multicolumn{1}{c}{6.919$\pm$1.767} & \multicolumn{1}{c}{.168$\pm$.047} & \multicolumn{1}{c}{.660$\pm$.093} & \multicolumn{1}{c}{.378$\pm$.167} & \multicolumn{1}{c}{.472$\pm$.156}\\
\multicolumn{1}{l}{\textsc{Ml-$k$nn}} & \multicolumn{1}{c}{.043$\pm$.015} & \multicolumn{1}{c}{.471$\pm$.157} & \multicolumn{1}{c}{\bf 4.097$\pm$1.236} & \multicolumn{1}{c}{\bf .102$\pm$.045} & \multicolumn{1}{c}{.625$\pm$.116} & \multicolumn{1}{c}{.292$\pm$.189} & \multicolumn{1}{c}{.381$\pm$.196}\\
\multicolumn{1}{l}{\textsc{Cnmf}} & \multicolumn{1}{c}{N/A} & \multicolumn{1}{c}{.509$\pm$.142} & \multicolumn{1}{c}{6.717$\pm$1.588} & \multicolumn{1}{c}{.171$\pm$.058} & \multicolumn{1}{c}{.561$\pm$.114} & \multicolumn{1}{c}{N/A} & \multicolumn{1}{c}{N/A}\\
\hline
\end{tabular}
\end{center}\bigskip\bigskip
\end{table}

\renewcommand{\baselinestretch}{1.4}
\small\normalsize

Table~\ref{table:yahooResult} shows that {\sc InsDif} and {\sc InsDif}$_{\rm MIMLSVM}$
perform well on the Yahoo data. Pairwise $t$-tests with 95\% significance level
disclose that: a) {\sc InsDif} is only inferior to {\sc AdtBoost.MH} and {\sc Ml-$k$nn}
in terms of \emph{coverage}, inferior to {\sc MlSvm} in terms of \emph{one-error},
comparable to {\sc Ml-$k$nn} in terms of \emph{ranking loss}, comparable to {\sc MlSvm}
in terms of \emph{average recall} and \emph{average F1}. Under all the other
circumstances (more than 79$\%$ cases), the performance of {\sc InsDif} is
significantly better than the compared multi-label learning algorithms (i.e., the
second part of Table~\ref{table:yahooResult}); b) {\sc InsDif}$_{\rm MIMLSVM}$ is
significantly better than the compared multi-label learning algorithms for more than
44$\%$ cases, and is significantly inferior to them for less than 18$\%$ cases.
Specifically, {\sc InsDif}$_{\rm MIMLSVM}$ achieves the best performance in terms of
\emph{average recall} and \emph{average F1}; on \emph{one-error}, it is only inferior
to {\sc MlSvm} but significantly superior the other compared multi-label learning
algorithms.

Overall, results on the web page categorization task suggest that MIML can be useful
when we are given observational data where each complicated object has already been
represented by a single instance.

\section{Solving Multi-Instance Single-Label Problems through MIML Transformation}\label{sec:MILtoMIML}\vspace{-4mm}

In many tasks we are given observational data where each object has already been
represented as a multi-instance single-label example, and we do not have access to the
real objects. In such case, we cannot capture more information from the real objects
using the MIML representation. Even in this situation, however, MIML is still useful.
Here we propose the \textsc{SubCod} (i.e., SUB-COncept Discovery) algorithm which
transforms multi-instance single-label examples into MIML examples to exploit the power
of MIML.

\subsection{\textsc{SubCod}}\vspace{-5mm}

For an object that has been described by multi-instances, if it is associated with a
label corresponding to a high-level complicated concept such as \textit{Africa} in
Fig.~\ref{fig:Africa}(a), it may be quite difficult to learn this concept directly. The
basic assumption of \textsc{SubCod} is that high-level complicated concepts can be
derived by a number of lower-level sub-concepts which are relatively clearer and easier
for learning, so that we can transform the single-label into a set of labels each
corresponds to one sub-concept. Therefore, we can learn these labels at first and then
derive the high-level complicated label based on them, as illustrated in
Fig.~\ref{fig:Africa}(b).

\textsc{SubCod} is a two-stage algorithm, which is based on \textit{sub-concept
discovery}. In the first stage, \textsc{SubCod} transforms each single-label example
into a multi-label example by discovering and exploiting sub-concepts involved by the
original label; this is realized by constructing multiple labels through unsupervised
clustering all instances and then treating each cluster as a set of instances of a
separate sub-concept. In the second stage, the outputs learned from the transformed
data set are used to derive the original labels that are to be predicted; this is
realized by using a supervised learning algorithm to predict the original labels from
the sub-concepts predicted by an MIML learner.

Using the same denotation as that in Sections~\ref{sec:MIML} and \ref{sec:MIMLalgos},
that is, given data set $\{(X_1, y_1), (X_2, y_2), \cdots, (X_m, y_m)\}$, where $X_i
\subseteq \mathcal{X}$ is a set of instances $\{{\bm x_{i1}}, {\bm x_{i2}}, \cdots,$
${\bm x_{i,n_i}}\}$, ${\bm x_{ij}} \in \mathcal{X}$ $(j=1,2,\cdots,n_i)$, and $y_i \in
\mathcal{Y}$ is the label of $X_i$. Here $n_i$ denotes the number of instances in
$X_i$.

In the first stage, \textsc{SubCod} collects all instances from all the bags to compose
a data set $D = \{{\bm x_{11}}, \cdots, {\bm x_{1,n_1}}, {\bm x_{21}}, \cdots, {\bm
x_{2,n_2}}, \cdots, {\bm x_{m1}}, \cdots, {\bm x_{m,n_m}}\}$. For the ease of
discussion, let $N = \sum\nolimits_{i=1}^{m}n_i$ and re-index the instances in $D$ as
$\{ {\bm x}_1, {\bm x}_2, \cdots, {\bm x}_N\}$. A Gaussian mixture model with $M$
mixture components is to be learned from $D$ by the EM algorithm, and the mixture
components are regarded as sub-concepts. The parameters of the mixture components,
i.e., the means ${\bm \mu}_k$, covariances $\Sigma_k$ and mixing coefficients $\pi_k$
($k = 1, 2, \cdots, M$), are randomly initialized and the initial value of the
log-likelihood is evaluated. In the E-step, the responsibilities are measured according
to
\begin{equation}\label{eq:responsibility}
\gamma_{ik} = {{\pi_k \mathcal{N}({\bm x}_i | {\bm \mu}_k,\Sigma_k)} \over
{\sum\limits_{j=1}^M \pi_j \mathcal{N}({\bm x}_i | {\bm \mu}_j,\Sigma_j)}} ~~~~~~ (i =
1, 2, \cdots, N) \ .
\end{equation}

In the M-step, the parameters are re-estimated according to
\begin{equation}\label{eq:updateMu}
{\bm \mu}_k^{new} = \frac{\sum\limits_{i=1}^{N} \gamma_{ik}{\bm
x}_i}{\sum\limits_{i=1}^{N} \gamma_{ik}} \ ,
\end{equation}
\begin{equation}\label{eq:updateSigma}
\Sigma_k^{new} = \frac{\sum\limits_{i=1}^{N} \gamma_{ik}({\bm x}_i - {\bm
\mu}_k^{new})({\bm x_i} - {\bm \mu}_k^{new})^{\rm T}}{\sum\limits_{i=1}^{N}
\gamma_{ik}} \ ,
\end{equation}
\begin{equation}\label{eq:updatePi}
\pi_k^{new} = \frac{\sum\limits_{i=1}^{N} \gamma_{ik}}{N} \ ,
\end{equation}
and the log-likelihood is evaluated according to
\begin{equation}\label{eq:loglikelihood}
\ln p \left(D|{\bm \mu}, \Sigma, \pi \right) = \sum\limits_{i=1}^{N} \ln \left(
\sum\limits_{k=1}^{M} \pi_k^{new} \mathcal{N} \left({\bm x}_i |{\bm
\mu}_k^{new},\Sigma_k^{new} \right) \right) \ .
\end{equation}

After the convergence of the EM process (or after a pre-specified number of
iterations), we can estimate the associated sub-concept for every instance ${\bm x}_i
\in D$ ($i=1, 2, \cdots, N$) by
\begin{equation}\label{eq:sc}
sc({\bm x}_i) = \mathop{\arg\max}_{k}\gamma_{ik} ~~~~~ (k = 1, 2, \cdots, M) \ .
\end{equation}

Then, we can derive the multi-label for each $X_i$ ($i= 1, 2, \cdots, m$) by
considering the sub-concept belongingness. Let ${\bm c}_i$ denote an $M$-dimensional
binary vector where each element is either $+1$ or $-1$. For $j=1, 2, \cdots, M$,
$c_{ij} = +1$ means that the sub-concept corresponding to the $j$-th Gaussian mixture
component appears in $X_i$, while $c_{ij} = -1$ means that this sub-concept does not
appear in $X_i$. Here the value of $c_{ij}$ can be determined according to a simple
rule that $c_{ij} = +1$ if $X_i$ has at least one instance which takes the $j$-th
sub-concept (i.e., satisfying Eq.~\ref{eq:sc}); otherwise $c_{ij} = -1$. Note that for
examples with identical single-label, the derived multi-labels for them may be
different.

The above process works in an unsupervised way which does not consider the original
labels of the bags $X_i$'s. Thus, the derived multi-labels ${\bm c}_i$ need to be
polished by incorporating the relation between the sub-concepts and the original label
of $X_i$. Here the maximum margin criterion is used. In detail, consider a vector ${\bm
z}_i$ with elements $z_{ij} \in [-1.0, +1.0]$ $(j=1, 2, \cdots, M)$;  $z_{ij} = +1$
means that the label $c_{ij}$ should not be modified while $z_{ij} = -1$ means that the
label $c_{ij}$ should be inverted. Denote ${\bm q}_i = {\bm c}_i \odot {\bm z}_i$ as
that for $j=1, 2, \cdots, M$, $q_{ij} = c_{ij} z_{ij}$. Let $\theta$ denote the
smallest number of labels that cannot be inverted. \textsc{SubCod} attempts to optimize
the objective\vspace{-10pt}
\begin{eqnarray}\label{eq:SubCodop}
\min\limits_{\bm{w}, b, \bm{\xi}, \bm{Z}} && \frac{1}{2}\|\bm{w}\|_2^2 + C \sum\limits_{i=1}^{m}{\xi_i}\\
\mbox{s.t.} && y_i({\bm w}^{\rm T} ({\bm c}_i \odot {\bm z}_i) +b) \geq 1- \xi_i, \nonumber\\
    && \bm \xi \ge \bm 0,\ -1\leq z_{ij}\leq 1 \nonumber\\
    && \sum\limits_{i,j} z_{ij} \geq 2\theta-mM \nonumber \ ,
\end{eqnarray}
where $\bm Z = [\bm z_1, \bm z_2, \cdots, \bm z_m]$.

By solving Eq.~\ref{eq:SubCodop} we will get the vector ${\bm z}_i$ which maximizes the
margin of the prediction of the proper labels of $X_i$. Here we solve
Eq.~\ref{eq:SubCodop} iteratively. We initialize ${\bm Z}$ with all 1's. First, we fix
${\bm Z}$ to get the optimal ${\bm w}$ and $b$; this is a standard QP problem. Then, we
fix ${\bm w}$ and $b$ to get the optimal ${\bm Z}$; this is a standard LP problem.
These two steps are iterated till convergence. Finally, we set the multi-label vector's
elements which correspond to positive $c_{ij} z_{ij}$'s ($i= 1, 2, \cdots, m; j = 1, 2,
\cdots, M$) to $+1$, and set the remaining ones to $-1$. Thus, we get all the polished
multi-label vectors $\tilde{{\bm c}}_i$ for the bags $X_i$. Thus, the original data set
$\{(X_1, y_1), (X_2, y_2), \cdots, (X_m, y_m)\}$ is transformed to an MIML data set
$\{(X_1, \tilde{{\bm c}}_1), (X_2, \tilde{{\bm c}}_2), \cdots, (X_m, \tilde{{\bm
c}}_m)\}$, and any MIML algorithms can be applied.

To map the multi-labels predicted by the MIML classifier for a test example to the
original single-labels $y \in \mathcal{Y}$, in the second stage of \textsc{SubCod}, a
traditional classifier $f: {\{+1,-1\}}^{M} \to \mathcal{Y}$ is generated from the data
set $\{(\tilde{{\bm c}}_1, y_1), (\tilde{{\bm c}}_2, y_2), \cdots, (\tilde{{\bm c}}_m,
y_m)\}$. This is relatively simple and traditional supervised learning algorithms can
be applied.

The pseudo-code of \textsc{SubCod} is summarized in Appendix A (Table~\ref{table:SubCod}). In the
first stage (Steps 1 to 3), \textsc{SubCod} derives multi-labels via sub-concept
discovery and transforms single-label examples into MIML examples, from which an MIML
learner is generated. In the second stage (Step 4), a traditional classifier is trained
to map the derived multi-labels to the original single-labels. Test example $X^*$ is
fed to the MIML learner to get its multi-labels, and the multi-labels are then fed to
the supervised classifier to get the label $y^*$ predicted for $X^*$.

\subsection{Experiments}\vspace{-5mm}

We compare \textsc{SubCod} with several state-of-the-art multi-instance learning
algorithms, including \textsc{Diverse Density} \cite{Maron:Lozano1998}, \textsc{Em-dd}
\cite{Zhang:Goldman2002}, \textsc{mi-Svm} and \textsc{Mi-Svm}
\cite{Andrews:Tsochantaridis:Hofmann2003}, and \textsc{Ch-Fd}
\cite{Fung:Dundar:Krishnappuram2007nips06}; these algorithms have been introduced
briefly in Section~\ref{sec:survey}.For {\sc SubCod}, the MIML learner in Step 3 is
realized by {\sc MimlSvm} and the classifier in Step 4 is realized by {\sc Smo} with
default parameters. In addition, by using {\sc MimlNn} and {\sc MimlSvm}$_{mi}$
respectively to replace {\sc MimlSvm} for realizing Step 3 of {\sc SubCod}, we get two
variants of {\sc SubCod}, i.e., {\sc SubCod}$_{\rm MIMLNN}$ and {\sc SubCod}$_{{\rm
MIMLSVM}_{mi}}$. They are also evaluated for comparison.\footnote{We have also
evaluated the variant {\sc SubCod}$_{\rm MIMLBOOST}$ which is obtained by employing
{\sc MimlBoost} to replace {\sc MimlSvm}, however, it did not terminate within
reasonable time and so its performance is not reported in this section.}

Note that the experiments here are very different from that in
Sections~\ref{sec:MIMLexp}, \ref{sec:Dmimlexp} and \ref{sec:SIMLexp}. Both
Sections~\ref{sec:MIMLexp} and \ref{sec:Dmimlexp} deal with learning from MIML
examples, Section~\ref{sec:SIMLexp} deals with learning from single-instance
multi-label examples, while this section deals with learning from multi-instance
single-label examples, and therefore the experimental data sets in this section are
different from those used in Sections~\ref{sec:MIMLexp}, \ref{sec:Dmimlexp} and
\ref{sec:SIMLexp}.

Five benchmark multi-instance learning data sets are used, including \textit{Musk1},
\textit{Musk2}, \textit{Elephant}, \textit{Tiger} and \textit{Fox}. Both \textit{Musk1}
and \textit{Musk2} are drug activity prediction data sets, publicly available at the
UCI machine learning repository \cite{Blake:Keogh:Merz1998}. Here every bag corresponds
to a molecule, while every instance corresponds to a low-energy shape of the molecule
\cite{Dietterich:Lathrop:Lozano1997}. \textit{Musk1} contains 47 positive bags and 45
negative bags, and the number of instances contained in each bag ranges from 2 to 40.
\textit{Musk2} contains 39 positive bags and 63 negative bags, and the number of
instances contained in each bag ranges from 1 to 1,044. Each instance is a
166-dimensional feature vector. \textit{Elephant}, \textit{Tiger} and \textit{Fox} are
three image annotation data sets generated by \cite{Andrews:Tsochantaridis:Hofmann2003}
for multi-instance learning. Here every bag is an image, while every instance
corresponds to a segmented region in the image
\cite{Andrews:Tsochantaridis:Hofmann2003}. Each data set contains 100 positive and 100
negative bags, and each instance is a 230-dimensional feature vector. These data sets
are popularly used in evaluating the performance of multi-instance learning algorithms.

\renewcommand{\baselinestretch}{.95}
\small\normalsize

\begin{table}[!t]
\caption{Predictive accuracy on \emph{Musk1}, \emph{Musk2}, \emph{Elephant}, \emph{Tiger} and \emph{Fox} data sets}\smallskip\smallskip \label{table:MILresult}\tabcolsep 0.1in\scriptsize
\begin{center}
\begin{tabular}{lccccc}
\hline\noalign{\smallskip}
\raisebox{-2mm}{Compared} & \multicolumn{5}{c}{Data sets} \\
\cline{2-6}\noalign{\smallskip}
\raisebox{2mm}{Algorithms} & \multicolumn{1}{c}{\textit{Musk1}} & \multicolumn{1}{c}{\textit{Musk2}} & \multicolumn{1}{c}{\textit{Elephant}} & \multicolumn{1}{c}{\textit{Tiger}} & \multicolumn{1}{c}{\textit{Fox}} \\
\hline\noalign{\smallskip}
\multicolumn{1}{l}{\textsc{SubCod}} & \multicolumn{1}{c}{0.850$\pm$0.035} & \multicolumn{1}{c}{\bf 0.921$\pm$0.014} & \multicolumn{1}{c}{\bf 0.836$\pm$0.010} & \multicolumn{1}{c}{0.808$\pm$0.013} & \multicolumn{1}{c}{\bf 0.616$\pm$0.020} \\
\multicolumn{1}{l}{\textsc{SubCod}$_{\rm MIMLNN}$} & \multicolumn{1}{c}{0.859$\pm$0.025} & \multicolumn{1}{c}{0.888$\pm$0.022} & \multicolumn{1}{c}{0.815$\pm$0.023} & \multicolumn{1}{c}{0.795$\pm$0.018} & \multicolumn{1}{c}{0.599$\pm$0.032} \\
\multicolumn{1}{l}{\textsc{SubCod}$_{{\rm MIMLSVM}_{mi}}$} & \multicolumn{1}{c}{0.870$\pm$0.023} & \multicolumn{1}{c}{0.869$\pm$0.020} & \multicolumn{1}{c}{0.805$\pm$0.017} & \multicolumn{1}{c}{0.787$\pm$0.016} & \multicolumn{1}{c}{0.590$\pm$0.015} \\
\hline
\multicolumn{1}{l}{\textsc{Diverse Density}}  & \multicolumn{1}{c}{0.880} & \multicolumn{1}{c}{0.840} & \multicolumn{1}{c}{N/A} & \multicolumn{1}{c}{N/A} & \multicolumn{1}{c}{N/A} \\
\multicolumn{1}{l}{\textsc{Em-dd}} & \multicolumn{1}{c}{0.848} & \multicolumn{1}{c}{0.849} & \multicolumn{1}{c}{0.783} & \multicolumn{1}{c}{0.721} & \multicolumn{1}{c}{0.561} \\
\multicolumn{1}{l}{\textsc{mi-Svm}} & \multicolumn{1}{c}{0.874} & \multicolumn{1}{c}{0.836} & \multicolumn{1}{c}{0.820} & \multicolumn{1}{c}{0.789} & \multicolumn{1}{c}{0.582} \\
\multicolumn{1}{l}{\textsc{Mi-Svm}} & \multicolumn{1}{c}{0.779} & \multicolumn{1}{c}{0.843} & \multicolumn{1}{c}{0.814} & \multicolumn{1}{c}{\bf 0.840} & \multicolumn{1}{c}{0.594} \\
\multicolumn{1}{l}{\textsc{Ch-Fd}} & \multicolumn{1}{c}{\bf 0.888} & \multicolumn{1}{c}{0.857} & \multicolumn{1}{c}{0.824} & \multicolumn{1}{c}{0.822} & \multicolumn{1}{c}{0.604} \\
\hline
\end{tabular}
\end{center}\bigskip\bigskip
\end{table}

\renewcommand{\baselinestretch}{1.4}
\small\normalsize

Parameters of {\sc SubCod} are determined by hold-out tests on training sets.
Specifically, candidate values of $M$ (the number of Gaussian mixture components) range
between $[10, 70]$, while candidate values of $\theta$ (the smallest number of labels
that cannot be inverted) range between $[mM\times 10\%, mM\times 70\%]$. Ten runs of
ten-fold cross validation are performed and the results are summarized in
Table~\ref{table:MILresult}, where the best performance on each data set has been
highlighted in boldface. Note that the results of the compared algorithms (second part
of Table~\ref{table:MILresult}) are the best performance reported in literatures
\cite{Andrews:Tsochantaridis:Hofmann2003,Fung:Dundar:Krishnappuram2007nips06}.\footnote{The tradition of the multi-instance learning community is to compare with the best performance reported in literature. Since the detailed results are not available \cite{Andrews:Tsochantaridis:Hofmann2003,Chen:Bi:Wang2006,Chen:Wang2004,Fung:Dundar:Krishnappuram2007nips06,Gartner:Flach:Kowalczyk2002,Maron:Lozano1998,Wang:Zucker2000,Zhang:Goldman2002}, we do not perform statistical significance tests at here.}

Table~\ref{table:MILresult} shows that \textsc{SubCod} and its variants are very
competitive to state-of-the-art multi-instance learning algorithms. In particular, on
\textit{Musk2} their performance are much better than other algorithms. This is
expectable because \textit{Musk2} is a complicated data set which has the largest
number of instances, while on such data set the sub-concept discovery process of
\textsc{SubCod} may be more effective.

Overall, the experimental results suggest that MIML could be useful when we are given
observational data where each object has already been represented as a multi-instance
single-label example.

\section{Conclusion}\label{sec:conclusion}\vspace{-4mm}

This paper extends our preliminary work ~\cite{Zhou:Zhang2007nips06,Zhang:Zhou2007} to
formalize the MIML \textit{Multi-Instance Multi-Label learning} framework, where an
example is described by multiple instances and associated with multiple class labels.
It was inspired by the recognition that when solving real-world problems, having a good
representation is often more important than having a strong learning algorithm because
a good representation may capture more meaningful information and make the learning
task easier to tackle. Since many real objects are inherited with input ambiguity as
well as output ambiguity, MIML is more natural and convenient for tasks involving such
objects.

To exploit the advantages of the MIML representation, we propose the \textsc{MimlBoost}
algorithm and the \textsc{MimlSvm} algorithm based on a simple degeneration strategy.
Experiments on scene classification and text categorization show that solving problems
involving complicated objects with multiple semantic meanings under the MIML framework
can lead to good performance. Considering that the degeneration process may lose
information, we also propose the \textsc{D-MimlSvm} algorithm which tackles MIML
problems directly in a regularization framework. Experiments show that this ``direct''
\textsc{Svm} algorithm outperforms the ``indirect'' \textsc{MimlSvm} algorithm.

In some practical tasks we are given observational data where each complicated object
has already been represented by a single instance, and we do not have access to the
real objects such that we cannot capture more information from the real objects using
the MIML representation. For such scenario, we propose the \textsc{InsDif} algorithm
which transforms single-instances into MIML examples to learn. Experiments on Yeast
gene functional analysis and web page categorization show that such algorithm is able
to achieve a better performance than learning the single-instances directly. This is
not difficult to understand. Actually, by representing the multi-label object using
multi-instances, the structure information collapsed in traditional single-instance
representation may become easier to exploit, and for each label the number of training
instances can be significantly increased. So, transforming multi-label examples to MIML
examples for learning may be beneficial in some tasks.

MIML can also be helpful for learning single-label examples involving complicated
high-level concepts. Usually it may be quite difficult to learn such concepts directly
since many different lower-level concepts are mixed together. If we can transform the
single-label into a set of labels corresponding to some sub-concepts, which are
relatively clearer and easier to learn, we can learn these labels at first and then
derive the high-level complicated label based on them. Inspired by this recognition, we
propose the \textsc{SubCod} algorithm which works by discovering sub-concepts of the
target concept at first and then transforming the data into MIML examples to learn.
Experiments show that this algorithm is able to achieve better performance than
learning the single-label examples directly in some tasks.

We believe that semantics exist in the connections between atomic input patterns and
atomic output patterns; while a prominent usefulness of MIML, which has not been
realized in this paper, is the possibility of identifying such connection. As
stated in Section \ref{sec:MIML}, in the MIML framework it is possible to
understand why a concerned object has a certain class label; this may be more important
than simply making an accurate prediction, because the results could be helpful for
understanding the source of ambiguous semantics.

\section*{Acknowledgements}

The authors want to thank the anonymous reviewers for helpful comments and suggestions.
The authors also want to thank De-Chuan Zhan and James Kwok for help on
\textsc{D-MimlSvm}, Yang Yu for help on \textsc{SubCod}, and Andr\'e Elisseeff and Jason Weston for providing the Yeast data and the implementation details of \textsc{RankSvm}. A preliminary Chinese version
has been presented at the Chinese Workshop on Machine Learning and Applications 2009.


\newpage

\section*{Appendix}
\appendix
\section{Pseudo-codes of the Learning Algorithms}

\renewcommand{\baselinestretch}{.85}

\begin{table}[!ht]
\small \caption{The {\sc MimlBoost} algorithm}
\label{table:mimlboost}\smallskip\smallskip
\begin{center}
\begin{tabular}{lll}
\hline\noalign{\smallskip\smallskip}

%

 1 & \multicolumn{2}{l}{Transform each MIML example $(X_u, Y_u)$ $(u=1,2,\cdots,m)$ into $|\mathcal{Y}|$ number of multi-}\\
   & \multicolumn{2}{l}{instance bags $\{ [(X_u, y_1), \Psi(X_u, y_1)], \cdots, [(X_u, y_{|\mathcal{Y}|}), \Psi(X_u, y_{|\mathcal{Y}|})]\}$. Thus, the original}\\
   & \multicolumn{2}{l}{data set is transformed into a multi-instance data set containing $m \times |\mathcal{Y}|$ number of}\\
   & \multicolumn{2}{l}{ multi-instance bags, denoted by $\{[(X^{(i)},y^{(i)}),\Psi(X^{(i)},y^{(i)})]\}$ ($i = 1, 2, \cdots, m \times |\mathcal{Y}|$).}\smallskip\smallskip\\

 2 & \multicolumn{2}{l}{Initialize weight of each bag to $W^{(i)} = {1 \over {m \times |\mathcal{Y}|}}$ ($i = 1, 2, \cdots, m \times |\mathcal{Y}|$).}\smallskip\smallskip\\

 3 & \multicolumn{2}{l}{Repeat for $t = 1,2,\cdots,T$ iterations:}\smallskip\\
 & 3a & Assign the bag's label $\Psi(X^{(i)},y^{(i)})$ to each of its instances $({\bm x}_{j}^{(i)},y^{(i)})$ ($i = 1, 2,$\\
 &    & $\cdots, m \times |\mathcal{Y}|$; $j = 1, 2, \cdots, n_i$), set the weight of the $j$-th instance of the $i$-th bag\\
 &    & $W_{j}^{(i)} = W^{(i)} / n_i$, and build an instance-level predictor $h_t[({\bm x}_{j}^{(i)}, y^{(i)})] \in \{-1, +1\}$.\smallskip\\

 & 3b & For the $i$-th bag, compute the error rate $e^{(i)} \in [0, 1]$ by counting the number of\\
 &    & misclassified instances within the bag, i.e. $e^{(i)} = {{\sum_{j=1}^{n_i} [\kern-0.15em[ h_t[({\bm x}_{j}^{(i)}, y^{(i)})] \neq \Psi(X^{(i)}, y^{(i)}) ]\kern-0.15em]} \over n_i}$.\smallskip\\

 & 3c & If $e^{(i)} < 0.5$ for all $i \in \{1, 2, \cdots, m \times |\mathcal{Y}| \}$, go to Step 4.\smallskip\\

 & 3d & Compute $c_t = \mathop{\arg\min}_{c_t} \sum_{i=1}^{m \times |\mathcal{Y}|} W^{(i)} \exp[(2e^{(i)} - 1)c_t]$.\smallskip\\

 & 3e & If $c_t \le 0$, go to Step 4.\smallskip\\

 & 3f & Set $W^{(i)} = W^{(i)} \exp[(2e^{(i)} -1)c_t]$ ($i = 1, 2, \cdots, m \times |\mathcal{Y}|$) and re-normalize such\\
 &    &  that $0 \le W^{(i)} \le 1$ and $\sum_{i=1}^{m \times |\mathcal{Y}|} W^{(i)} = 1$.\smallskip\smallskip\\

 4 & \multicolumn{2}{l}{Return $Y^* = \{ y | sign\left(\sum_{j} \sum_{t} c_t h_t[({\bm x}^*_j, y)]\right) = + 1\}$ (${\bm x}^*_j$ is $X^*$'s $j$-th instance).}\smallskip\\

\hline
\end{tabular}\bigskip\bigskip
\end{center}
\end{table}

\renewcommand{\baselinestretch}{.85}
\small\normalsize

\begin{table}[!ht]
\small \caption{The {\sc MimlSvm} algorithm} \label{table:mimlsvm}\smallskip\smallskip
\begin{center}
\begin{tabular}{lll}
\hline\noalign{\smallskip\smallskip}

%

 1 & \multicolumn{2}{l}{For MIML examples $(X_u, Y_u)$ $(u=1,2,\cdots,m)$, $\Gamma = \{X_u | u=1,2,\cdots,m \}$.}\smallskip\smallskip\\

 2 & \multicolumn{2}{l}{Randomly select $k$ elements from $\Gamma$ to initialize the medoids $M_t$ $(t=1,2,\cdots,k)$,}\\
   & \multicolumn{2}{l}{repeat until all $M_t$ do not change:}\smallskip\\
   & 2a & $\Gamma_t = \{ M_t\}$ $(t=1,2,\cdots,k)$.\smallskip\\
   & 2b & Repeat for each $X_u \in \left(\Gamma - \{M_t | t = 1,2,\cdots,k \}\right)$:\smallskip\\
   &    & \hspace{+3mm} $index = \mathop{\arg\min}\limits_{t\in \{1,\cdots,k\}} d_H(X_u,M_t)$, $\Gamma_{index} = \Gamma_{index} \cup \{ X_u\}$.\smallskip\\
   & 2c & $M_t = \mathop{\arg\min}\limits_{A \in \Gamma_t} \sum\limits_{B \in \Gamma_t} d_H(A,B)$ $(t=1,2,\cdots,k)$.\smallskip\smallskip\\

 3 & \multicolumn{2}{l}{Transform $(X_u, Y_u)$ into a multi-label example $({\bm z}_u, Y_u)$ $(u=1,2,\cdots,m)$, where }\\
   & \multicolumn{2}{l}{${\bm z}_u = \left({\bm z}_{u1}, {\bm z}_{u2}, \cdots, {\bm z}_{uk}\right) = \left( d_H(X_u, M_1), d_H(X_u, M_2), \cdots, d_H(X_u, M_k) \right)$.}\smallskip\smallskip\\

 4 & \multicolumn{2}{l}{For each $y \in \mathcal{Y}$, derive a data set $\mathcal{D}_y = \{ \left({\bm z}_u, \Phi\left({\bm z}_u, y\right)\right) | u = 1,2,\cdots,m \}$, and then}\\
   & \multicolumn{2}{l}{train an {\sc Svm} $h_y = SVMTrain(\mathcal{D}_y)$.}\smallskip\smallskip\\

 5 & \multicolumn{2}{l}{Return $Y^* = \{ \mathop{\arg\max}\limits_{y \in \mathcal{Y}} h_y ({\bm z}^*) \} \cup \{y| h_y({\bm z}^*) \ge 0, y \in \mathcal{Y}\}$, where ${\bm z}^* = (d_H(X^*, M_1),$}\\
   & \multicolumn{2}{l}{$d_H(X^*, M_2),\cdots, d_H(X^*, M_k) )$.}\smallskip\\

\hline
\end{tabular}\bigskip\bigskip
\end{center}
\end{table}

\renewcommand{\baselinestretch}{1.4}
\small\normalsize

\renewcommand{\baselinestretch}{.85}
\small\normalsize

\begin{table}[!ht]
\small \caption{Efficient Algorithm for Eq.~\ref{eq:opt5}}
\label{table:cuttingplane}\smallskip\smallskip
\begin{center}
\begin{tabular}{rl}
\hline\noalign{\smallskip} \multicolumn{2}{l}{\textbf{Input:} $K$, $\lambda$, $\mu$,
$\gamma$, $\varepsilon$,
$\{X_i,Y_i\}_{i=1}^m$}\vspace{+1mm}\\
1 & $\forall{t}$, $S_{t}=\emptyset$, $\bm v_{t}=(\bm \alpha_t^{T},\bm \xi_{t1},\cdots,\bm \xi_{tm},\bm \delta_{t1},\cdots,\bm \delta_{tm},b_t)=\bm 0$\\
2 & \textbf{Repeat}\\
3 & ~~ \textbf{For} $t = 1, \cdots, T$\\
4 & ~~~~ Pick $p$ indexes of constraints that are not in $S_t$ randomly, denoted by $I$;\\
5 & ~~~~ Compute $Loss_i$ for every constraint in $I$;\\
6 & ~~~~ \% find out the cutting plane\\
7 & ~~~~ $q = \mathop{\arg\max}_{i \in I} Loss_i$\\
8 & ~~~~ \textbf{If} $Loss_q > \varepsilon$\\
9 & ~~~~~~ $S_t = S_t \cup \{q\}$;\\
10& ~~~~~~ $\bm v_t$ $\leftarrow$ optimized over $S_t$; \\
11& ~~~~ \textbf{End If}\\
12& ~~ \textbf{End For}\\
13& \textbf{Until} no $S_{t}$ changes\\
\noalign{\smallskip}\hline
\end{tabular}\bigskip
\end{center}
\end{table}

\renewcommand{\baselinestretch}{1.4}
\small\normalsize

\renewcommand{\baselinestretch}{.85}
\small\normalsize

\begin{table}[!ht]
\small \caption{The \textsc{InsDif} algorithm} \label{table:InsDif}\smallskip\smallskip
\begin{center}
\begin{tabular}{lll}
\hline\noalign{\smallskip\smallskip}

%

 1 & \multicolumn{2}{l}{For single-instance multi-label examples $(x_u, Y_u)$ $(u=1,2,\cdots,m)$, compute the}\\
   & \multicolumn{2}{l}{prototype vectors ${\bm v}_l\ (l\in \mathcal{Y})$ using Eq.~\ref{eq:proto}.}\smallskip\smallskip\\

 2 & \multicolumn{2}{l}{Derive the new training set $S^*$ by transforming each ${\bm x}_i$ into a bag of instances $B_i$}\\
   & \multicolumn{2}{l}{using Eq.~\ref{eq:instTobag}.}\smallskip\smallskip\\

 3 & \multicolumn{2}{l}{Learning from $ S^* = \{(B_1,Y_1),$ $(B_2,Y_2), \cdots, (B_m,Y_m)\}$ by using an MIML algorithm.}\smallskip\\

%
%
%

\hline
\end{tabular}\vspace{60pt}
\end{center}

\small \caption{The \textsc{SubCod} algorithm} \label{table:SubCod}\smallskip\smallskip
\begin{center}
\begin{tabular}{lll}
\hline\noalign{\smallskip\smallskip}

%

 1 & \multicolumn{2}{l}{For multi-instance single-label examples $(X_u, y_u)$ $(u=1,2,\cdots,m)$, collect all the}\\
   & \multicolumn{2}{l}{instances ${\bm x} \in X_u$ together and identify the Gaussian mixture components through}\\
   & \multicolumn{2}{l}{the EM process detailed in Eqs.~\ref{eq:responsibility} to \ref{eq:loglikelihood}.}\smallskip\smallskip\\

 2 & \multicolumn{2}{l}{Determine the sub-concept for every instance ${\bm x}\in X_u$ according to Eq.~\ref{eq:sc}, and}\\
   & \multicolumn{2}{l}{then derive the label vector ${\bm c}_u$ for $X_u$.}\smallskip\smallskip\\

 3 & \multicolumn{2}{l}{Make corrections to ${\bm c}_u$ by optimizing Eq.~\ref{eq:SubCodop}, which results in $\tilde{{\bm c}}_u$ for $X_u$, and then}\\
   & \multicolumn{2}{l}{train an MIML learner $h_t(X)$ on $\{(X_u, \tilde{{\bm c}}_u)\}$ $(u=1,2,\cdots,m)$.}\smallskip\smallskip\\

 4 & \multicolumn{2}{l}{Train a classifier $h_y(\tilde{{\bm c}})$ on $\{(\tilde{{\bm c}}_u, y_u)\}$ $(u=1,2,\cdots,m)$, which maps the derived}\\
   & \multicolumn{2}{l}{multi-labels to the original single-labels.}\smallskip\smallskip\\

 5 & \multicolumn{2}{l}{Return $y^*=h_y \left( h_t \left( X^* \right) \right)$.}\smallskip\\

\hline
\end{tabular}\bigskip\bigskip\vspace{180pt}
\end{center}
\end{table}

\renewcommand{\baselinestretch}{1.4}
\small\normalsize

\clearpage

\section{Parameter Settings of the Learning Algorithms}\vspace{20pt}

\begin{figure}[!ht]
\centering
\begin{minipage}[c]{1.5in}
\centering
\includegraphics[width = 1.5in]{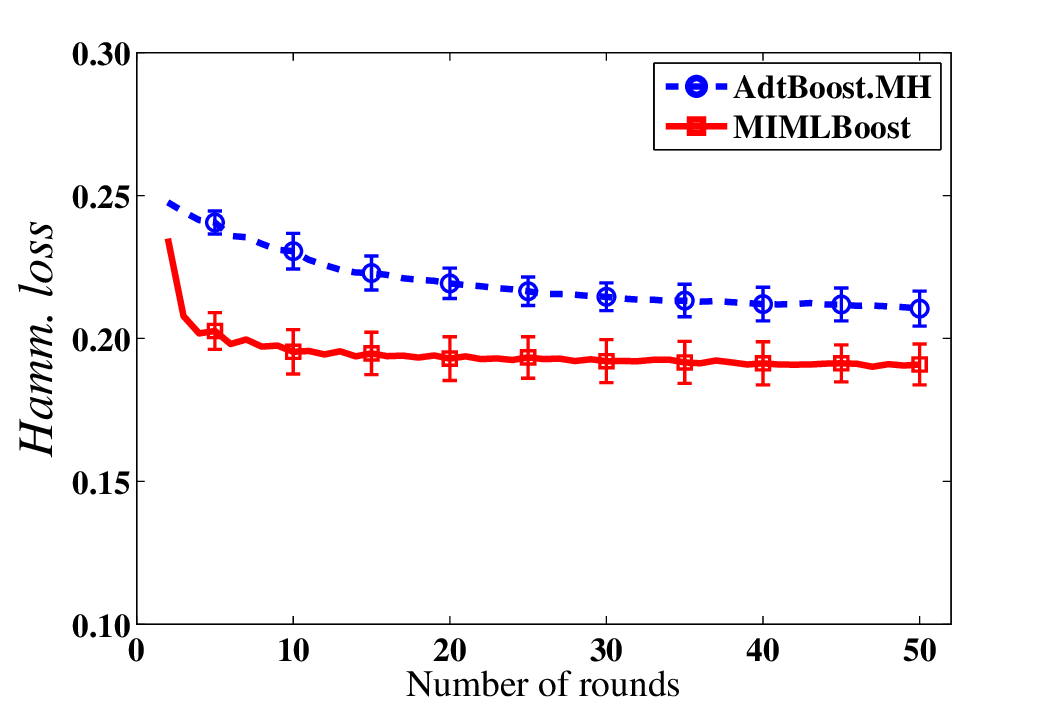}
\end{minipage}%
\begin{minipage}[c]{1.5in}
\centering
\includegraphics[width = 1.5in]{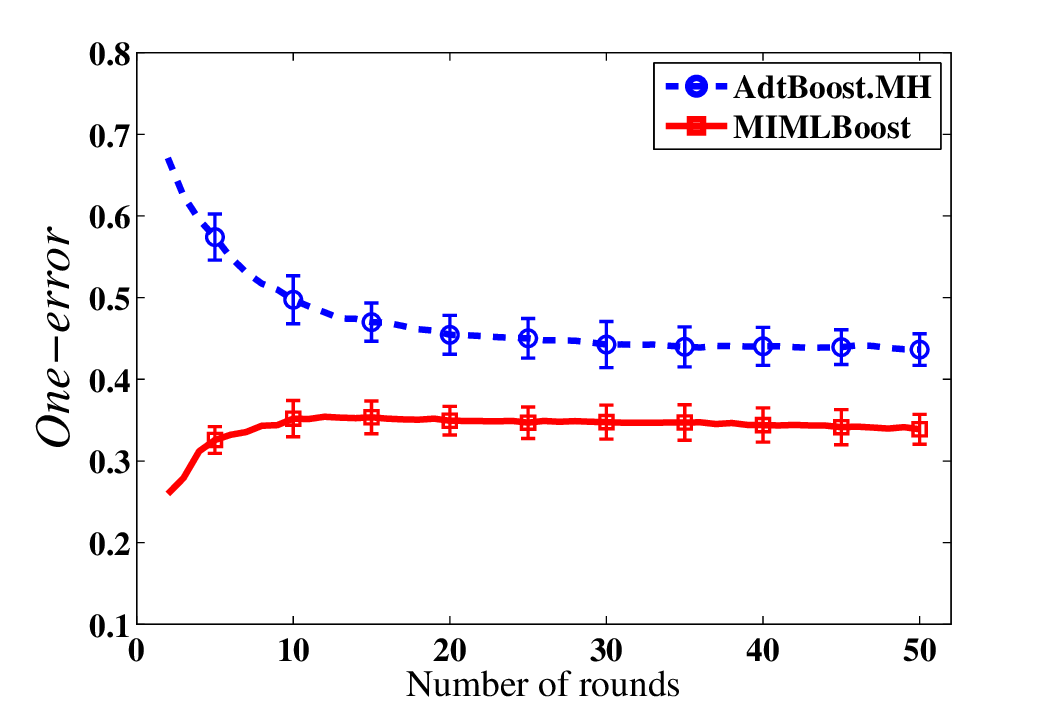}
\end{minipage}%
\begin{minipage}[c]{1.5in}
\centering
\includegraphics[width = 1.5in]{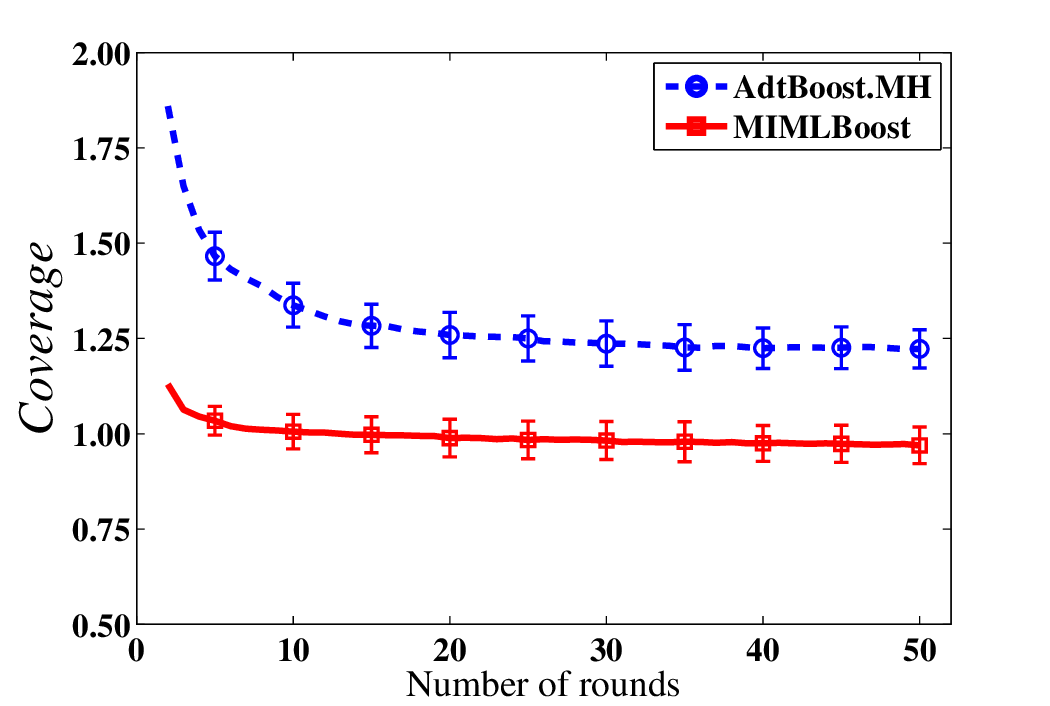}
\end{minipage}%
\begin{minipage}[c]{1.5in}
\centering
\includegraphics[width = 1.5in]{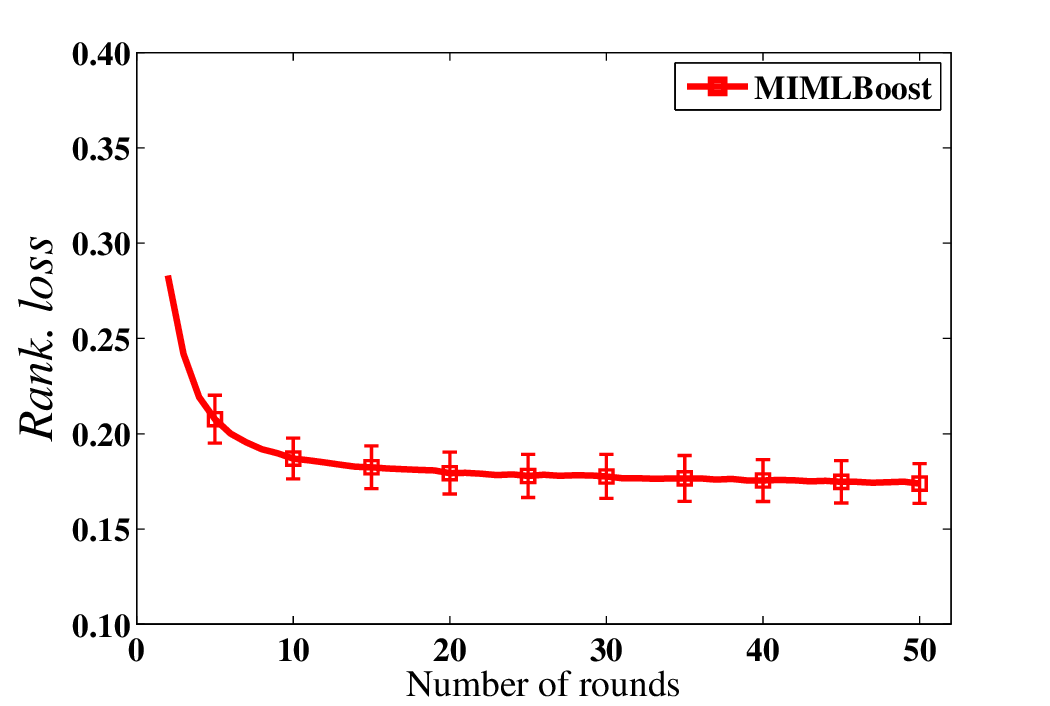}
\end{minipage}\\[+4pt]
\begin{minipage}[c]{1.5in}
\centering \mbox{{\footnotesize (a) \textit{hamming loss}}}
\end{minipage}%
\begin{minipage}[c]{1.5in}
\centering \mbox{{\footnotesize (b) \textit{one-error}}}
\end{minipage}%
\begin{minipage}[c]{1.5in}
\centering \mbox{{\footnotesize (c) \textit{coverage}}}
\end{minipage}%
\begin{minipage}[c]{1.5in}
\centering \mbox{{\footnotesize (d) \textit{ranking loss}}}
\end{minipage}\\[+5pt]
\begin{minipage}[c]{1.8in}
\centering
\includegraphics[width = 1.5in]{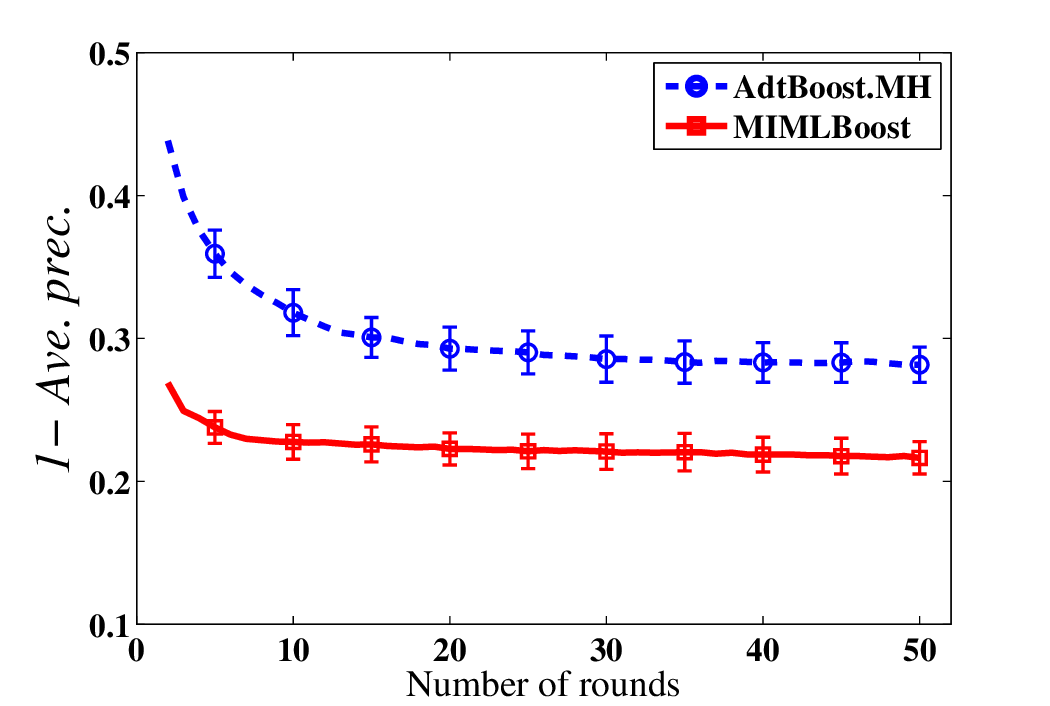}
\end{minipage}%
\begin{minipage}[c]{1.8in}
\centering
\includegraphics[width = 1.5in]{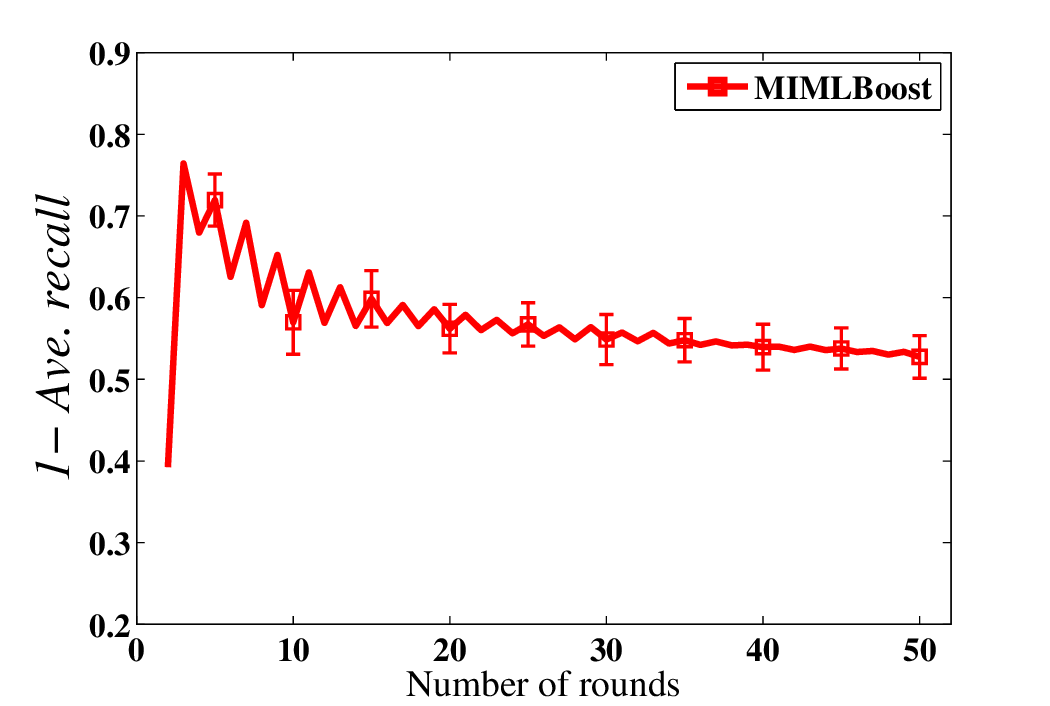}
\end{minipage}%
\begin{minipage}[c]{1.8in}
\centering
\includegraphics[width = 1.5in]{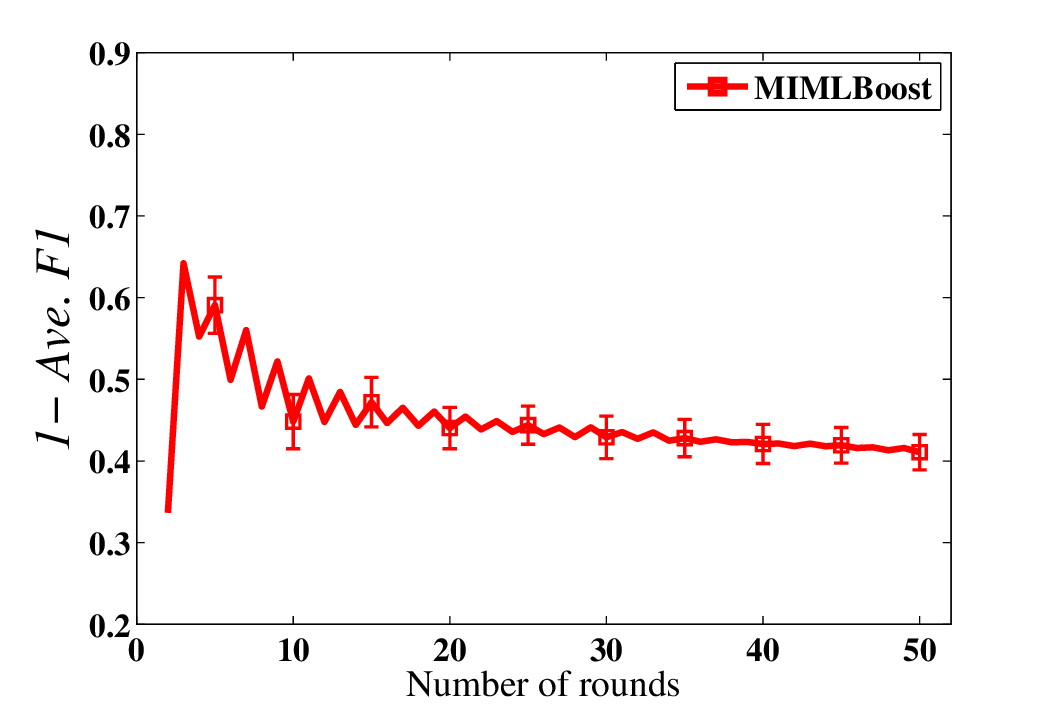}
\end{minipage}\\[+4pt]
\begin{minipage}[c]{1.8in}
\centering \mbox{{\footnotesize (e) $1-$ \textit{average precision}}}
\end{minipage}%
\begin{minipage}[c]{1.8in}
\centering \mbox{{\footnotesize (f) $1-$ \textit{average recall}}}
\end{minipage}%
\begin{minipage}[c]{1.8in}
\centering \mbox{{\footnotesize (g) $1-$ \textit{average F1}}}
\end{minipage}%
\caption{Performance of \textsc{MimlBoost} and \textsc{AdtBoost.MH} at different
rounds on scene classification data set.}\label{fig:boostrounds}\bigskip\vspace{40pt}
\end{figure}

\begin{figure}[!ht]
\centering
\begin{minipage}[c]{1.5in}
\centering
\includegraphics[width = 1.5in]{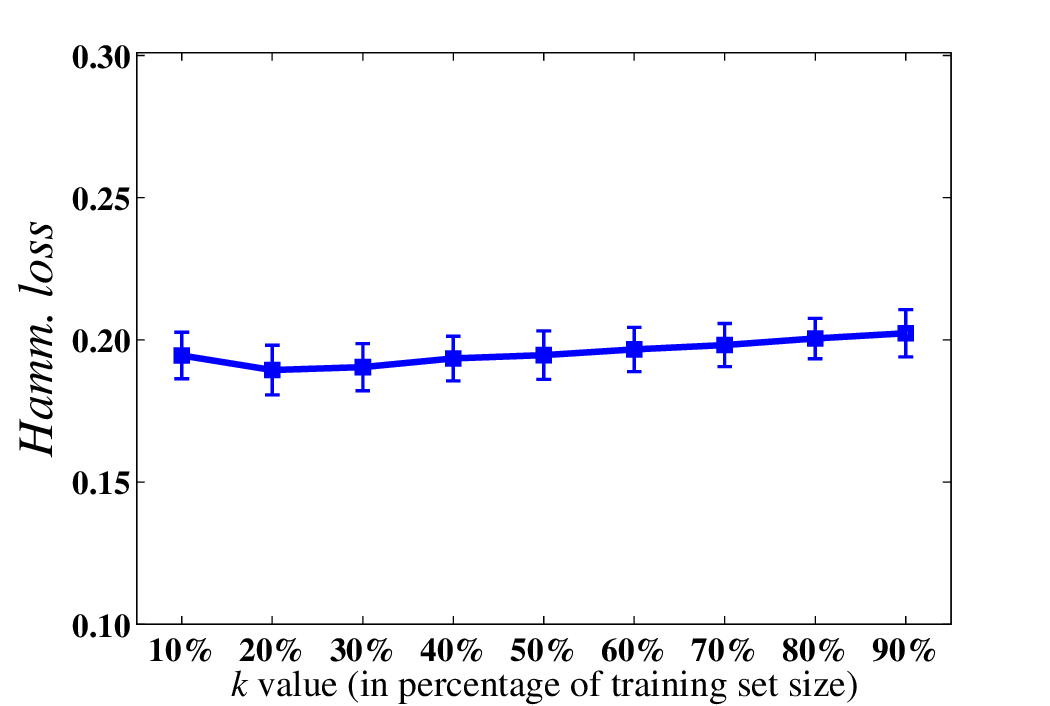}
\end{minipage}%
\begin{minipage}[c]{1.5in}
\centering
\includegraphics[width = 1.5in]{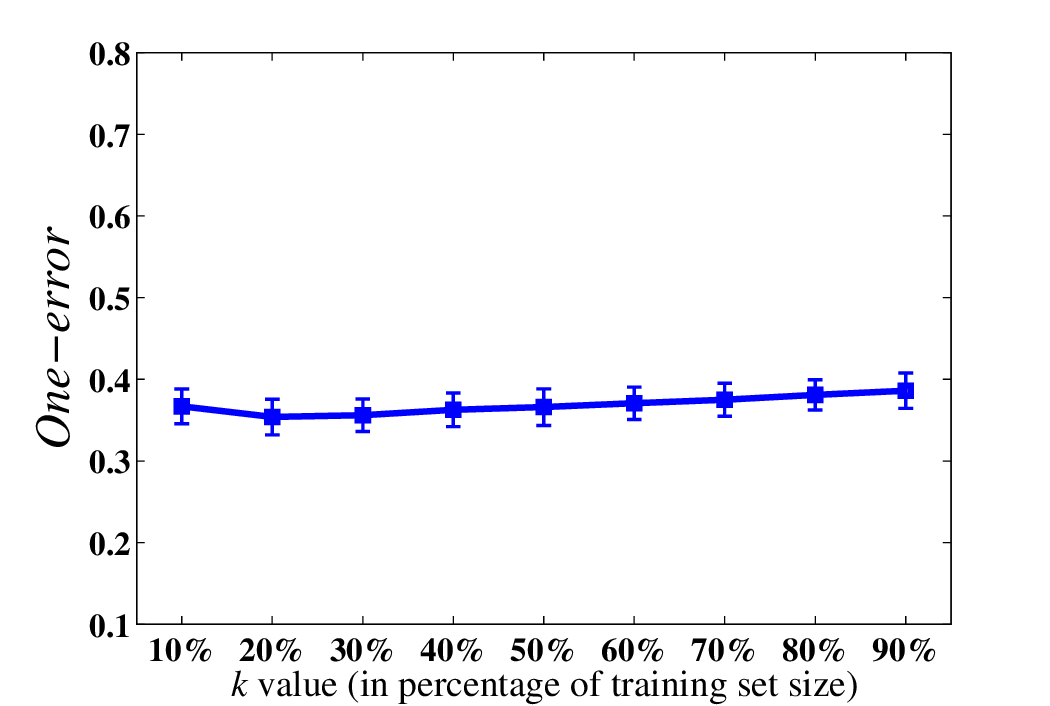}
\end{minipage}%
\begin{minipage}[c]{1.5in}
\centering
\includegraphics[width = 1.5in]{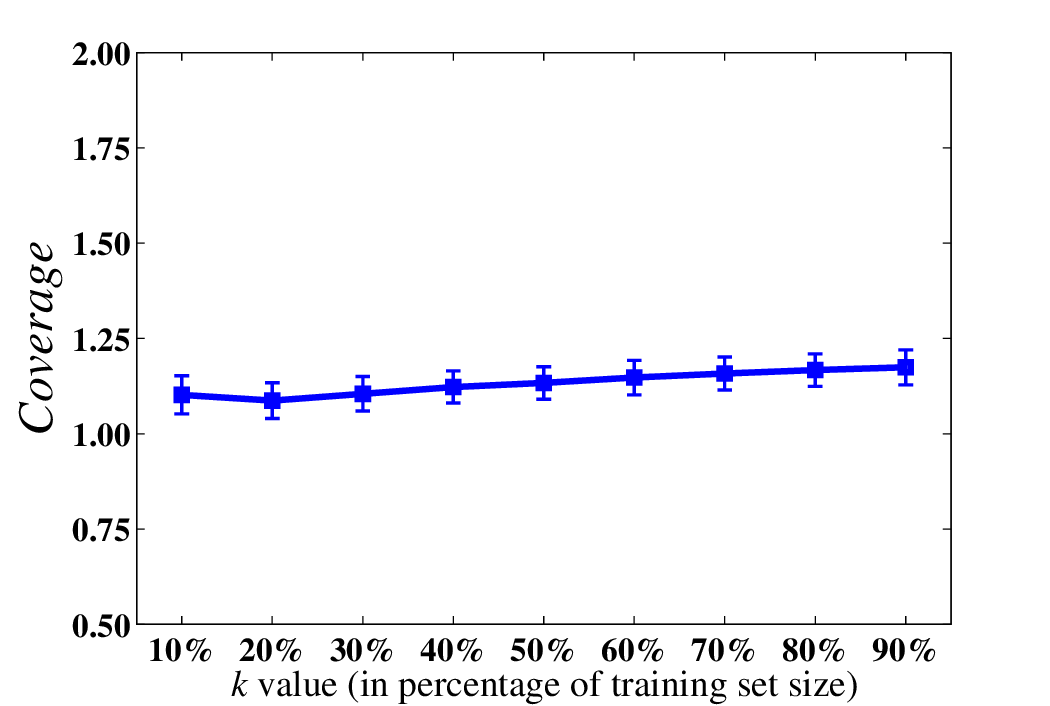}
\end{minipage}%
\begin{minipage}[c]{1.5in}
\centering
\includegraphics[width = 1.5in]{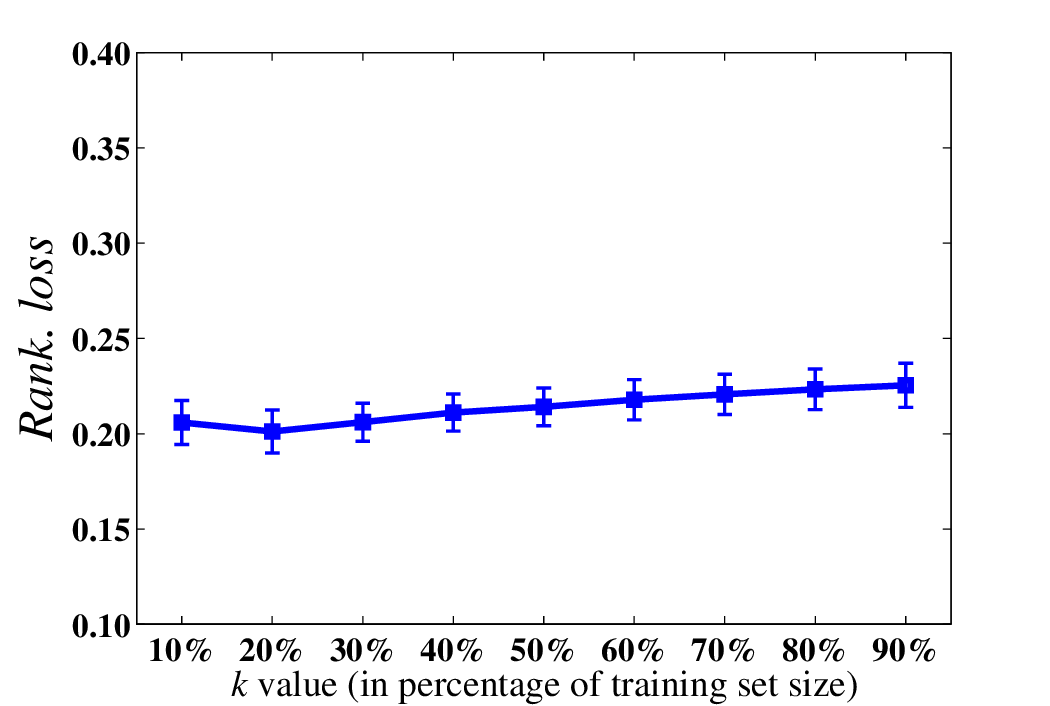}
\end{minipage}\\[+4pt]
\begin{minipage}[c]{1.5in}
\centering \mbox{{\footnotesize (a) \textit{hamming loss}}}
\end{minipage}%
\begin{minipage}[c]{1.5in}
\centering \mbox{{\footnotesize (b) \textit{one-error}}}
\end{minipage}%
\begin{minipage}[c]{1.5in}
\centering \mbox{{\footnotesize (c) \textit{coverage}}}
\end{minipage}%
\begin{minipage}[c]{1.5in}
\centering \mbox{{\footnotesize (d) \textit{ranking loss}}}
\end{minipage}\\[+5pt]
\begin{minipage}[c]{1.8in}
\centering
\includegraphics[width = 1.5in]{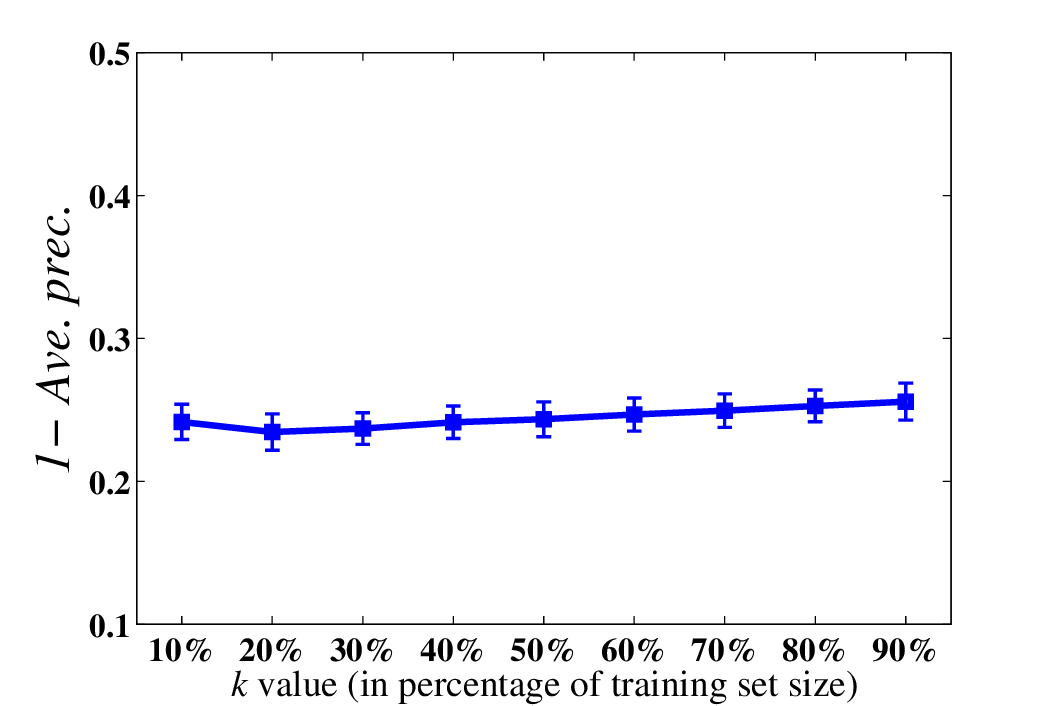}
\end{minipage}%
\begin{minipage}[c]{1.8in}
\centering
\includegraphics[width = 1.5in]{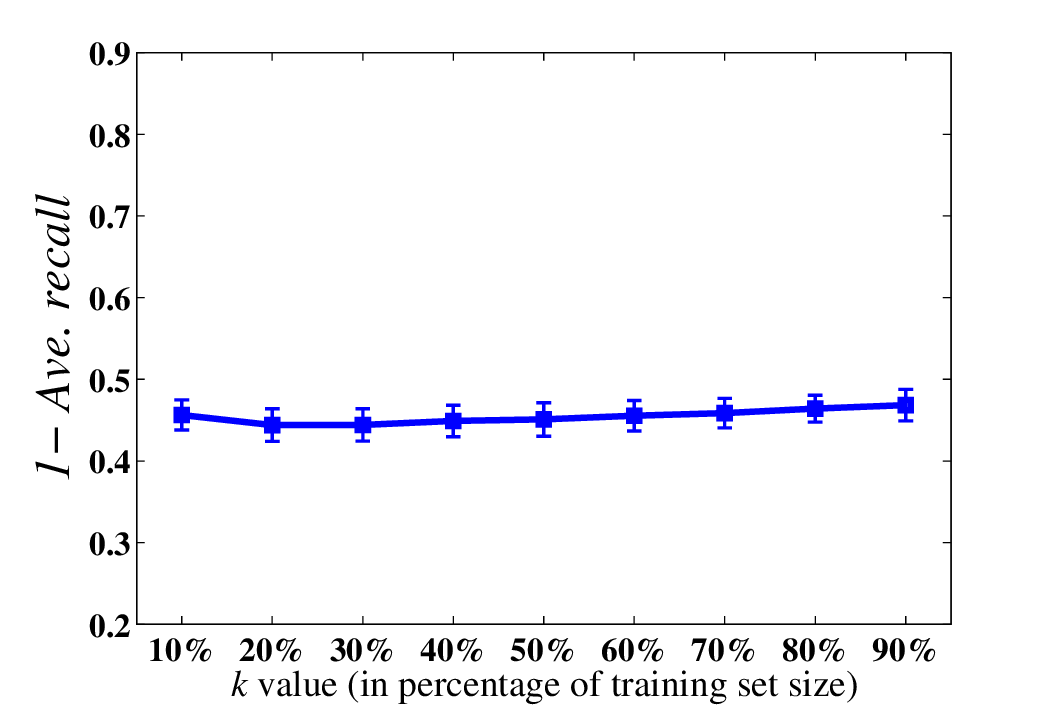}
\end{minipage}%
\begin{minipage}[c]{1.8in}
\centering
\includegraphics[width = 1.5in]{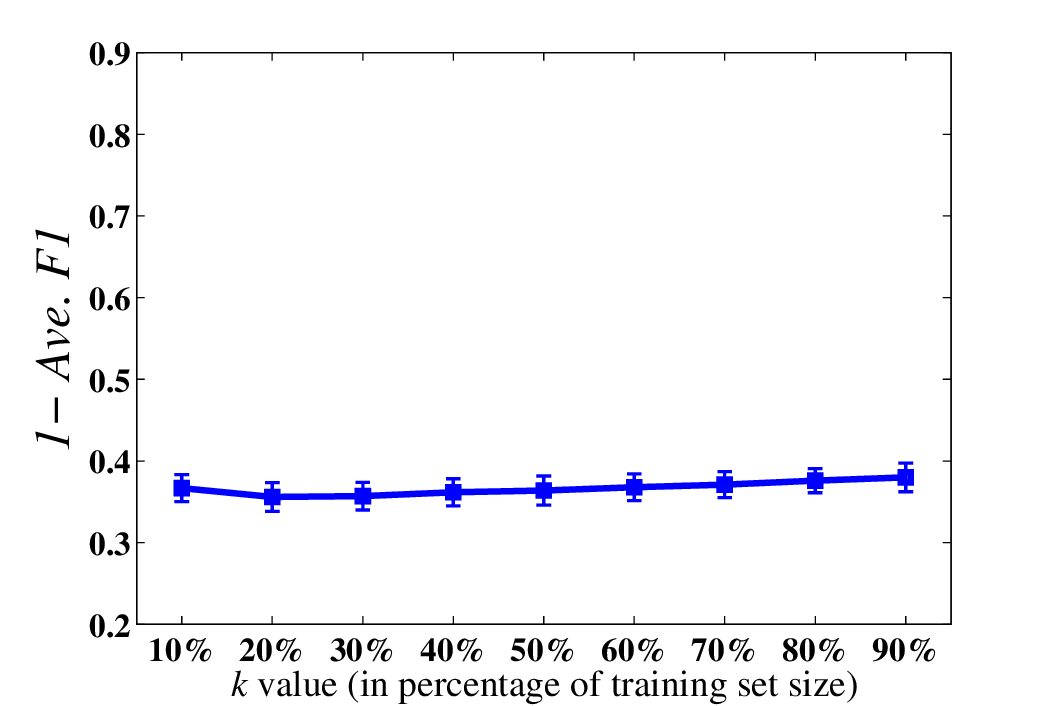}
\end{minipage}\\[+4pt]
\begin{minipage}[c]{1.8in}
\centering \mbox{{\footnotesize (e) $1-$ \textit{average precision}}}
\end{minipage}%
\begin{minipage}[c]{1.8in}
\centering \mbox{{\footnotesize (f) $1-$ \textit{average recall}}}
\end{minipage}%
\begin{minipage}[c]{1.8in}
\centering \mbox{{\footnotesize (g) $1-$ \textit{average F1}}}
\end{minipage}%
\caption{Performance of \textsc{MimlSvm} with different $k$
values on scene classification data set.}\label{fig:kvalues}
\end{figure}

\newpage

\renewcommand{\baselinestretch}{1.4}
\small\normalsize

\begin{figure}[!ht]
\centering
\begin{minipage}[c]{1.5in}
\centering
\includegraphics[width = 1.5in]{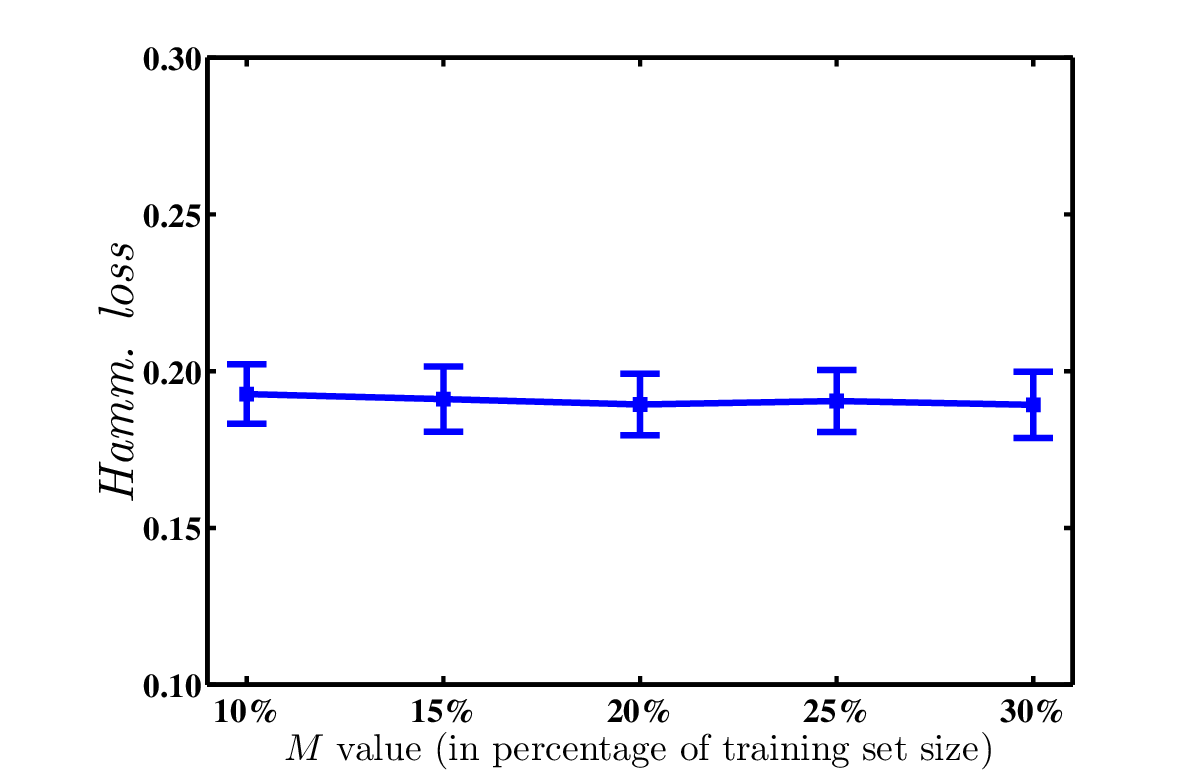}
\end{minipage}%
\begin{minipage}[c]{1.5in}
\centering
\includegraphics[width = 1.5in]{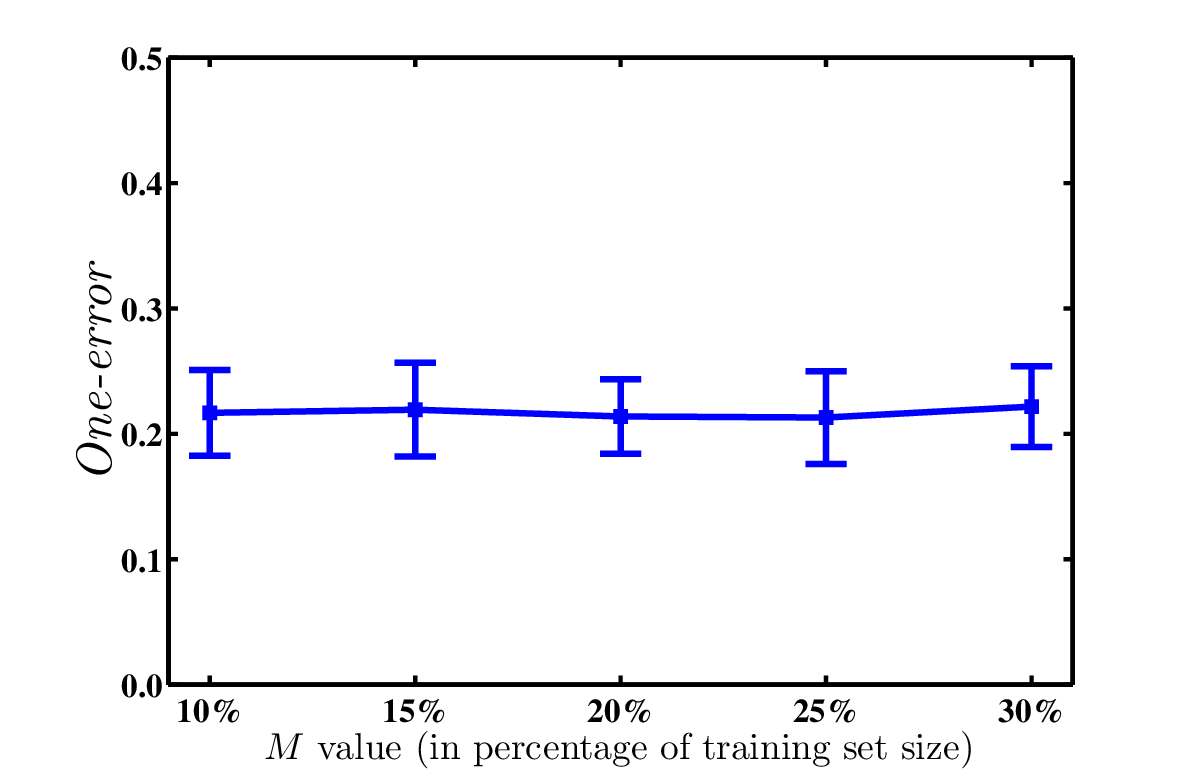}
\end{minipage}%
\begin{minipage}[c]{1.5in}
\centering
\includegraphics[width = 1.5in]{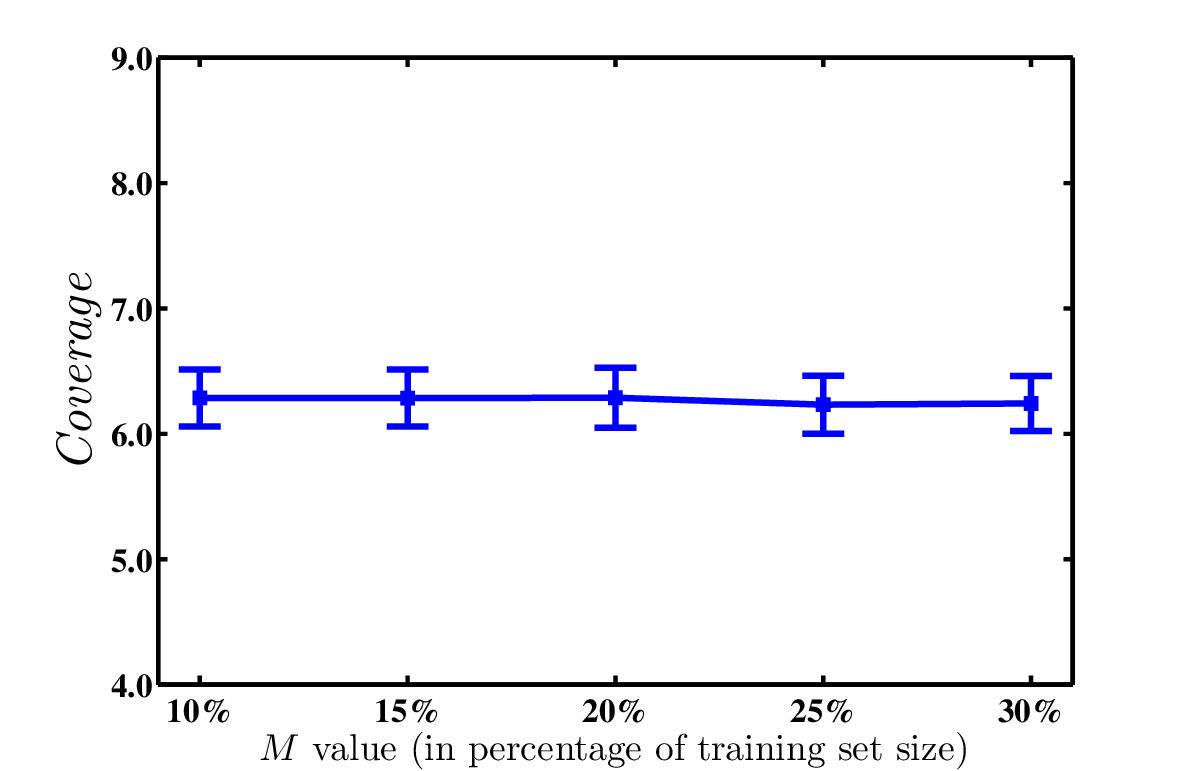}
\end{minipage}%
\begin{minipage}[c]{1.5in}
\centering
\includegraphics[width = 1.5in]{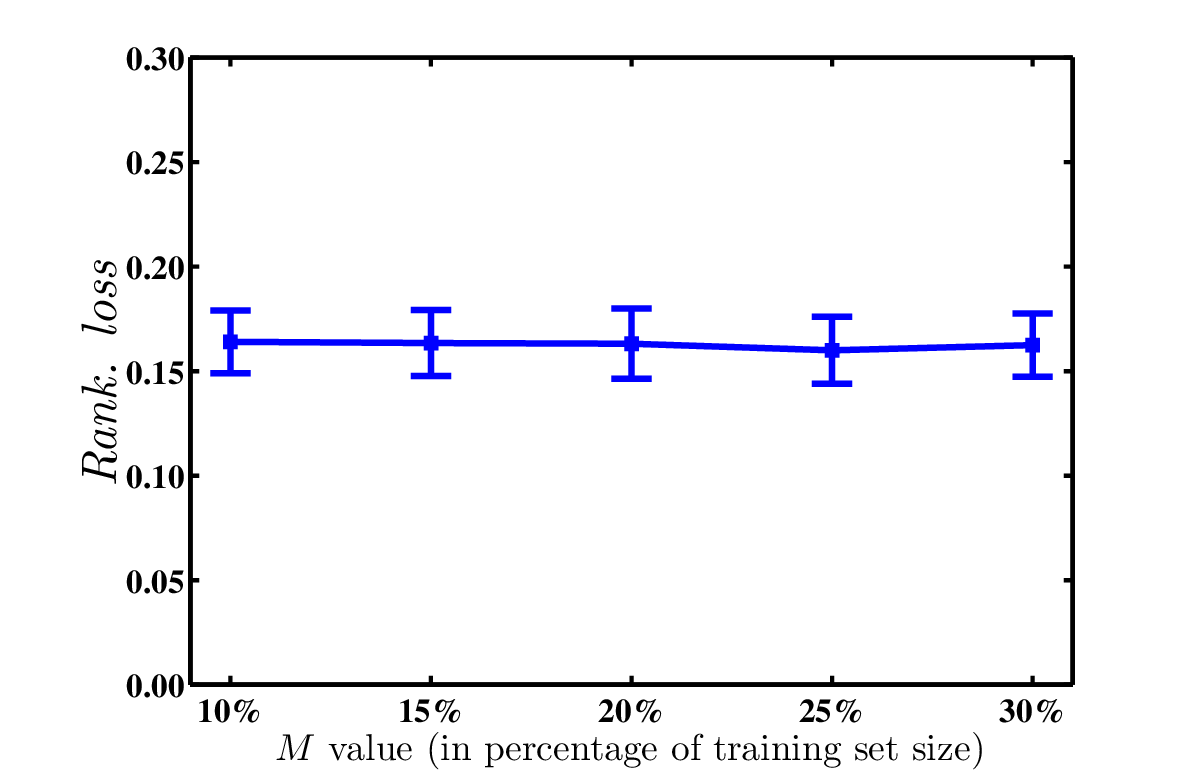}
\end{minipage}\\[+4pt]
\begin{minipage}[c]{1.5in}
\centering \mbox{{\footnotesize (a) \textit{hamming loss}}}
\end{minipage}%
\begin{minipage}[c]{1.5in}
\centering \mbox{{\footnotesize (b) \textit{one-error}}}
\end{minipage}%
\begin{minipage}[c]{1.5in}
\centering \mbox{{\footnotesize (c) \textit{coverage}}}
\end{minipage}%
\begin{minipage}[c]{1.5in}
\centering \mbox{{\footnotesize (d) \textit{ranking loss}}}
\end{minipage}\\[+5pt]
\begin{minipage}[c]{1.8in}
\centering
\includegraphics[width = 1.5in]{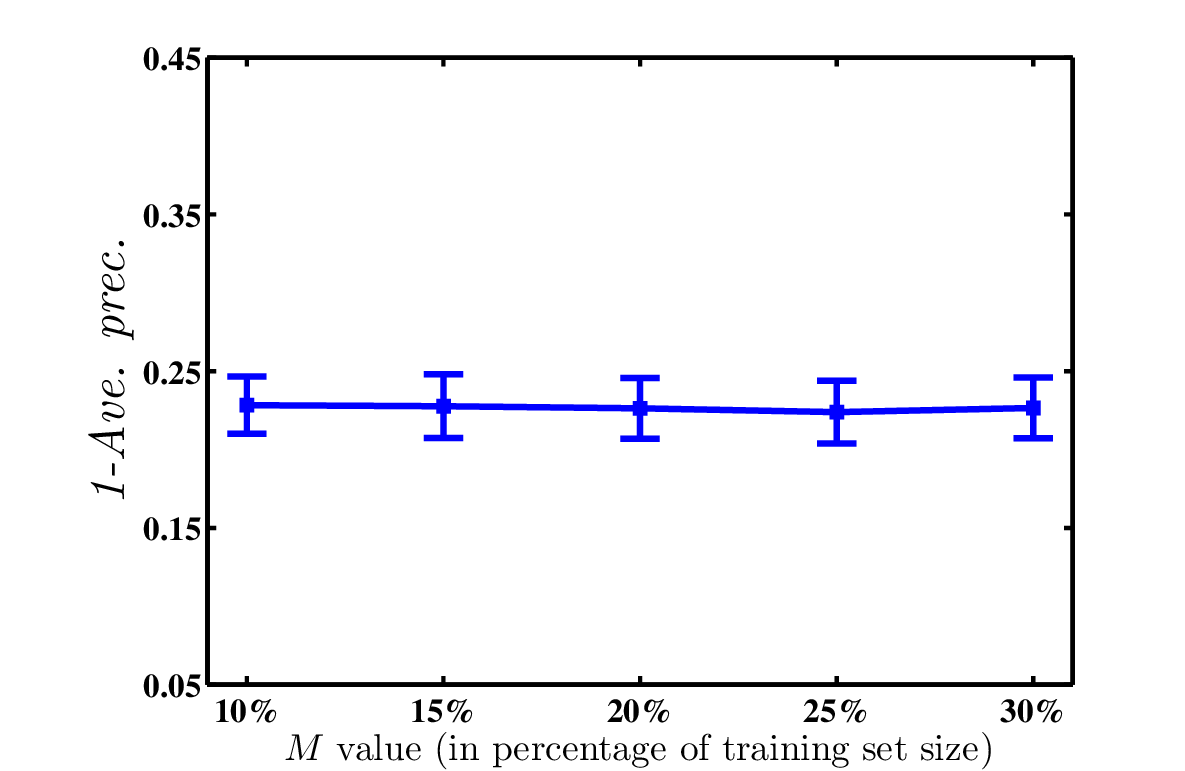}
\end{minipage}%
\begin{minipage}[c]{1.8in}
\centering
\includegraphics[width = 1.5in]{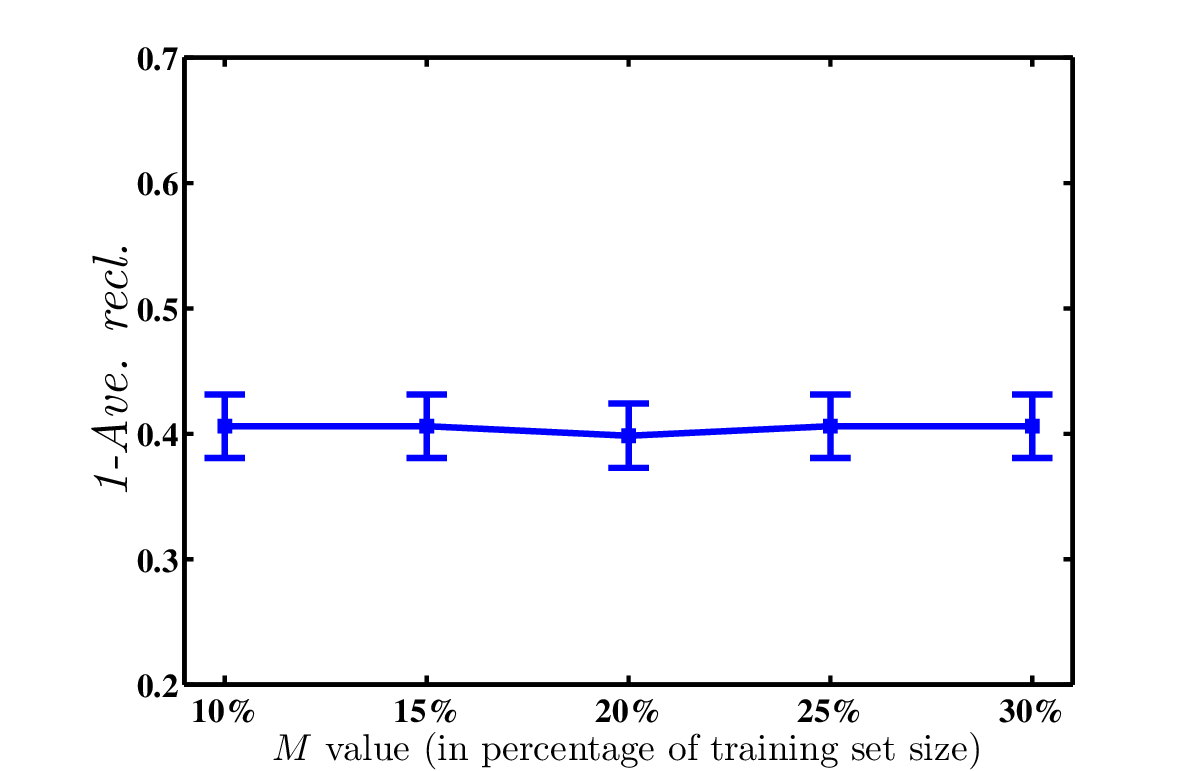}
\end{minipage}%
\begin{minipage}[c]{1.8in}
\centering
\includegraphics[width = 1.5in]{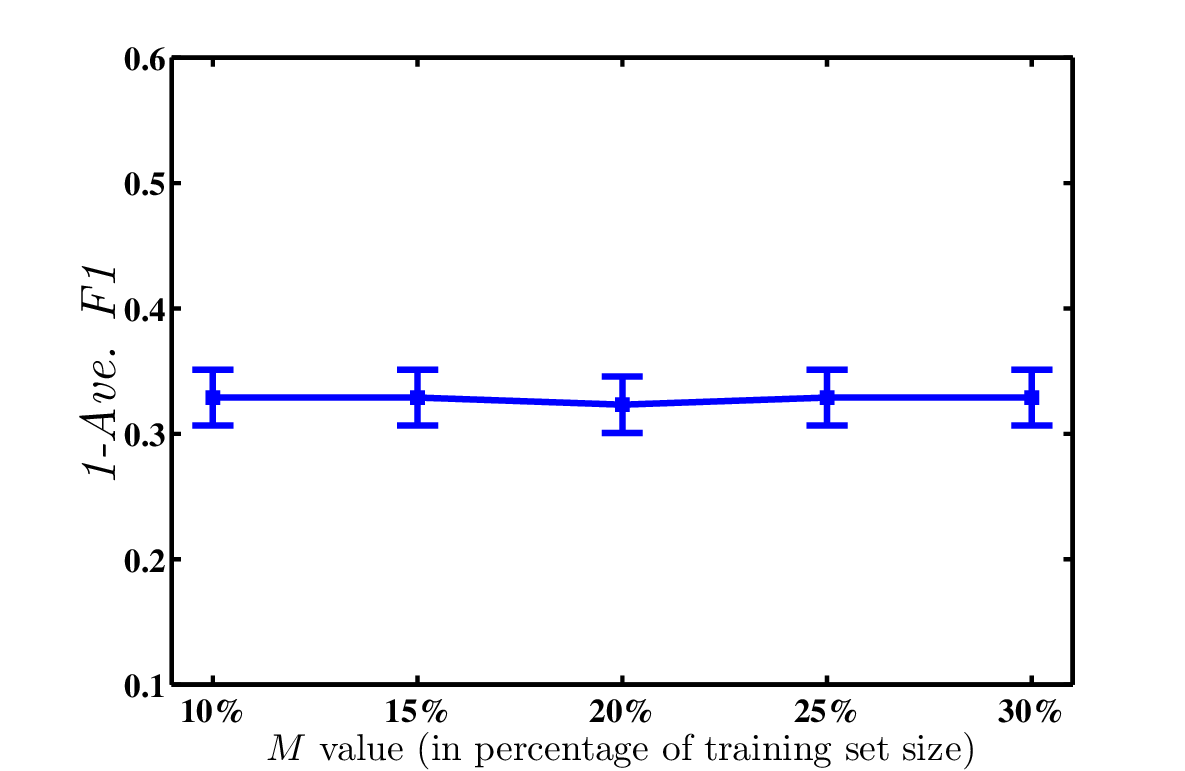}
\end{minipage}\\[+4pt]
\begin{minipage}[c]{1.8in}
\centering \mbox{{\footnotesize (e) $1-$ \textit{average precision}}}
\end{minipage}%
\begin{minipage}[c]{1.8in}
\centering \mbox{{\footnotesize (f) $1-$ \textit{average recall}}}
\end{minipage}%
\begin{minipage}[c]{1.8in}
\centering \mbox{{\footnotesize (g) $1-$ \textit{average F1}}}
\end{minipage}%
\caption{Performance of \textsc{InsDif} with different $M$
settings on Yeast gene data set.}\label{fig:Msettings}\bigskip
\end{figure}

\begin{figure}[!ht]
\centering
\begin{minipage}[c]{1.5in}
\centering
\includegraphics[width = 1.5in]{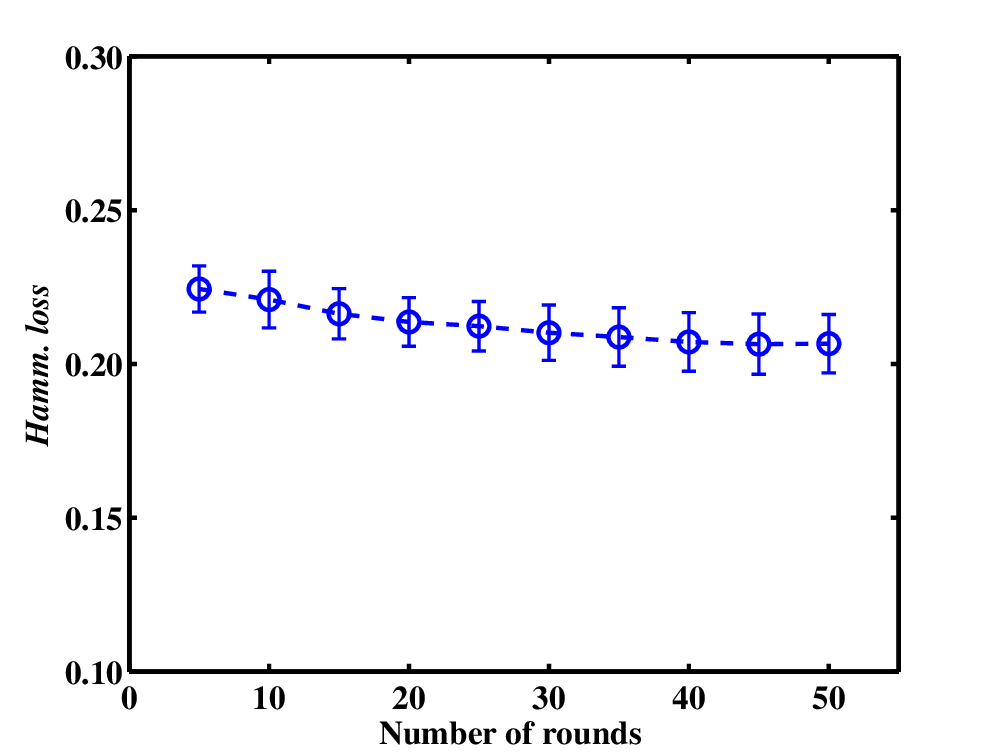}
\end{minipage}%
\begin{minipage}[c]{1.5in}
\centering
\includegraphics[width = 1.5in]{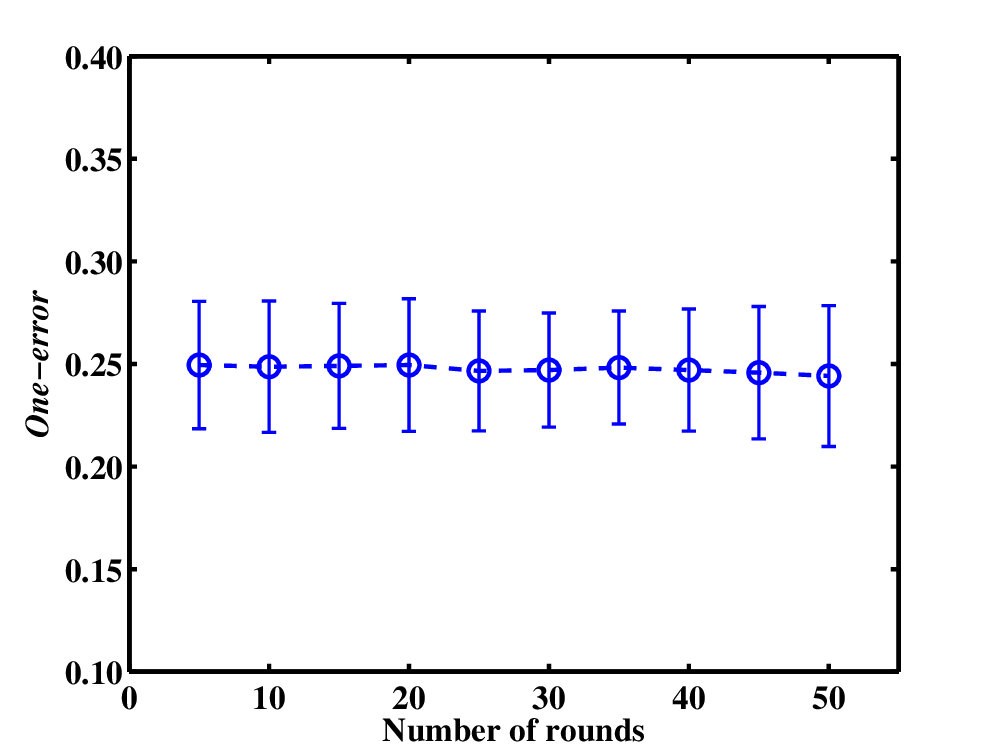}
\end{minipage}%
\begin{minipage}[c]{1.5in}
\centering
\includegraphics[width = 1.5in]{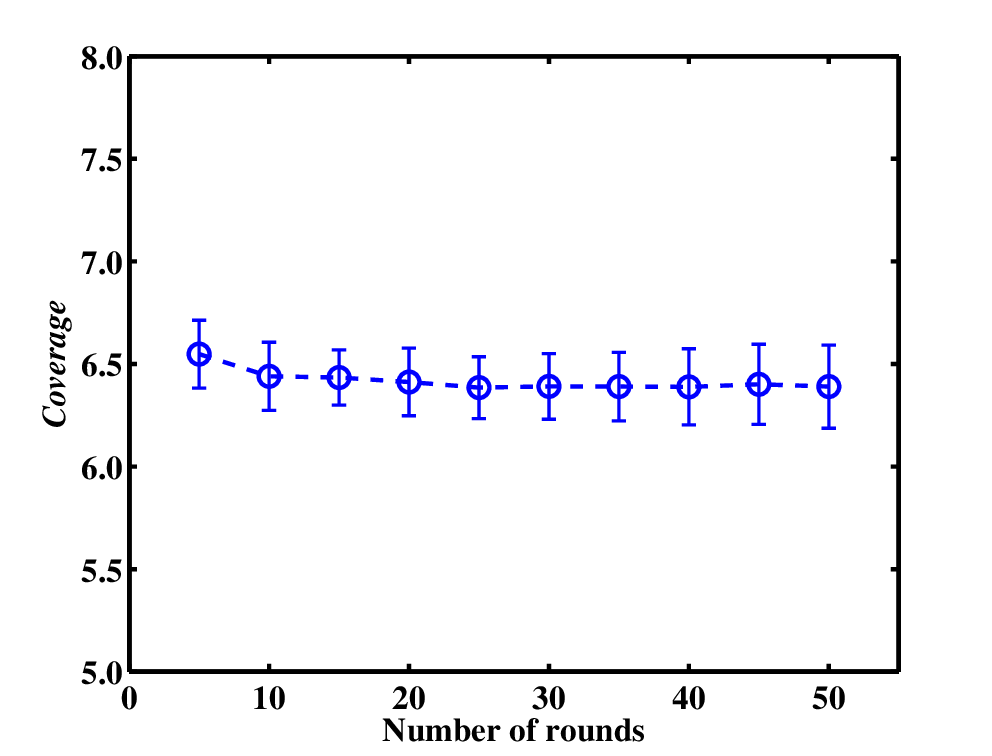}
\end{minipage}%
\begin{minipage}[c]{1.5in}
\centering
\includegraphics[width = 1.5in]{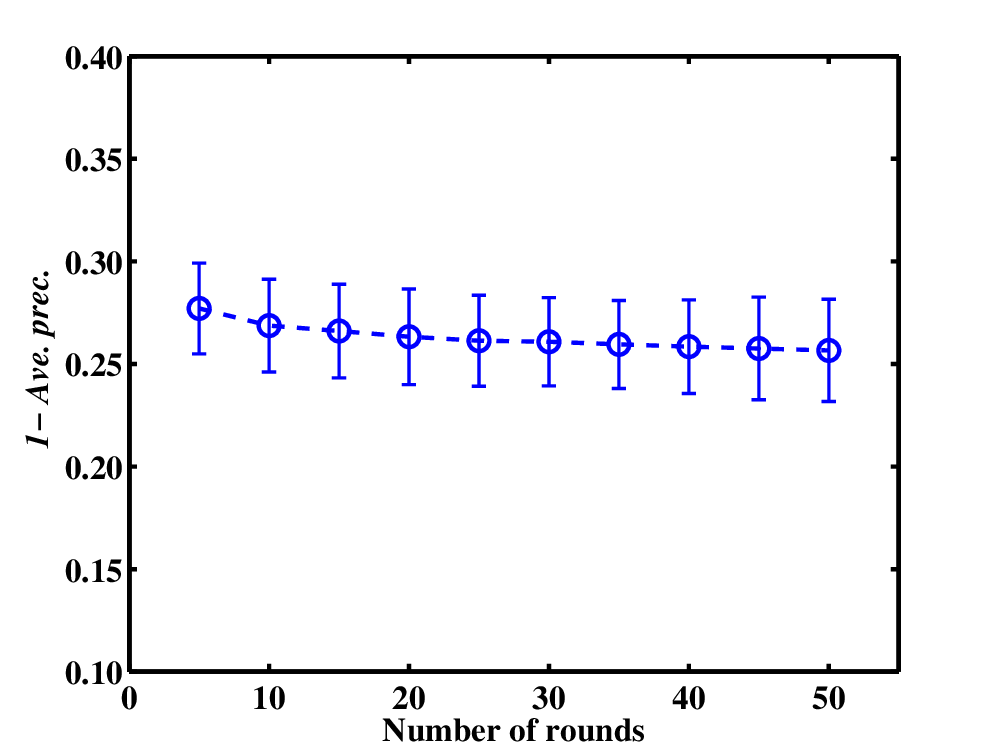}
\end{minipage}\\[+4pt]
\begin{minipage}[c]{1.5in}
\centering \mbox{{\footnotesize (a) \textit{hamming loss}}}
\end{minipage}%
\begin{minipage}[c]{1.5in}
\centering \mbox{{\footnotesize (b) \textit{one-error}}}
\end{minipage}%
\begin{minipage}[c]{1.5in}
\centering \mbox{{\footnotesize (c) \textit{coverage}}}
\end{minipage}%
\begin{minipage}[c]{1.5in}
\centering \mbox{{\footnotesize (d) $1-$\textit{average precision}}}
\end{minipage}%
\caption{Performance of \textsc{AdtBoost.MH} at different
rounds on Yeast gene data set.}\label{fig:Roundsettings}
\end{figure}

\section{Web Page Data Sets}

\renewcommand{\baselinestretch}{.95}
\small\normalsize

\begin{table}[!ht]
\caption{Characteristics of the web page data sets (after term selection). \textit{PMC}
denotes the percentage of documents belonging to more than one category; \textit{ANL}
denotes the average number of labels for each document; \textit{PRC} denotes the
percentage of \textit{rare} categories, i.e., the kind of category where only less than
$1\%$ instances in the data set belong to
it.}\smallskip\smallskip\label{table:yahoo}\scriptsize
\begin{center}
\begin{tabular}{llllllllll}
\hline\noalign{\smallskip}
& \multicolumn{1}{c}{\raisebox{-.5mm}{Number of}} & \multicolumn{1}{c}{\raisebox{-.5mm}{Vocabulary}} & \multicolumn{3}{c}{Training Set} & \multicolumn{1}{c}{} & \multicolumn{3}{c}{Test Set} \\
\cline{4-6}\cline{8-10}
\raisebox{2mm}{Data Set} & \multicolumn{1}{c}{\raisebox{.5mm}{Categories}} & \multicolumn{1}{c}{\raisebox{.5mm}{Size}} & \multicolumn{1}{c}{\emph{PMC}} & \multicolumn{1}{c}{\emph{ANL}} & \multicolumn{1}{c}{\emph{PRC}} & \multicolumn{1}{c}{} & \multicolumn{1}{c}{\emph{PMC}} & \multicolumn{1}{c}{\emph{ANL}} & \multicolumn{1}{c}{\emph{PRC}} \\
\hline
Arts\&Humanities & \multicolumn{1}{c}{26} & \multicolumn{1}{c}{462} & \multicolumn{1}{c}{44.50\%} & \multicolumn{1}{c}{1.627} & \multicolumn{1}{c}{19.23\%} & \multicolumn{1}{c}{} & \multicolumn{1}{c}{43.63\%} & \multicolumn{1}{c}{1.642} & \multicolumn{1}{c}{19.23\%} \\
Business\&Economy & \multicolumn{1}{c}{30} & \multicolumn{1}{c}{438} & \multicolumn{1}{c}{42.20\%} & \multicolumn{1}{c}{1.590} & \multicolumn{1}{c}{50.00\%} & \multicolumn{1}{c}{} & \multicolumn{1}{c}{41.93\%} & \multicolumn{1}{c}{1.586} & \multicolumn{1}{c}{43.33\%} \\
Computers\&Internet & \multicolumn{1}{c}{33} & \multicolumn{1}{c}{681} & \multicolumn{1}{c}{29.60\%} & \multicolumn{1}{c}{1.487} & \multicolumn{1}{c}{39.39\%} & \multicolumn{1}{c}{} & \multicolumn{1}{c}{31.27\%} & \multicolumn{1}{c}{1.522} & \multicolumn{1}{c}{36.36\%} \\
Education & \multicolumn{1}{c}{33} & \multicolumn{1}{c}{550} & \multicolumn{1}{c}{33.50\%} & \multicolumn{1}{c}{1.465} & \multicolumn{1}{c}{57.58\%} & \multicolumn{1}{c}{} & \multicolumn{1}{c}{33.73\%} & \multicolumn{1}{c}{1.458} & \multicolumn{1}{c}{57.58\%} \\
Entertainment & \multicolumn{1}{c}{21} & \multicolumn{1}{c}{640} & \multicolumn{1}{c}{29.30\%} & \multicolumn{1}{c}{1.426} & \multicolumn{1}{c}{28.57\%} & \multicolumn{1}{c}{} & \multicolumn{1}{c}{28.20\%} & \multicolumn{1}{c}{1.417} & \multicolumn{1}{c}{33.33\%} \\
Health & \multicolumn{1}{c}{32} & \multicolumn{1}{c}{612} & \multicolumn{1}{c}{48.05\%} & \multicolumn{1}{c}{1.667} & \multicolumn{1}{c}{53.13\%} & \multicolumn{1}{c}{} & \multicolumn{1}{c}{47.20\%} & \multicolumn{1}{c}{1.659} & \multicolumn{1}{c}{53.13\%} \\
Recreation\&Sports & \multicolumn{1}{c}{22} & \multicolumn{1}{c}{606} & \multicolumn{1}{c}{30.20\%} & \multicolumn{1}{c}{1.414} & \multicolumn{1}{c}{18.18\%} & \multicolumn{1}{c}{} & \multicolumn{1}{c}{31.20\%} & \multicolumn{1}{c}{1.429} & \multicolumn{1}{c}{18.18\%} \\
Reference & \multicolumn{1}{c}{33} & \multicolumn{1}{c}{793} & \multicolumn{1}{c}{13.75\%} & \multicolumn{1}{c}{1.159} & \multicolumn{1}{c}{51.52\%} & \multicolumn{1}{c}{} & \multicolumn{1}{c}{14.60\%} & \multicolumn{1}{c}{1.177} & \multicolumn{1}{c}{54.55\%} \\
Science & \multicolumn{1}{c}{40} & \multicolumn{1}{c}{743} & \multicolumn{1}{c}{34.85\%} & \multicolumn{1}{c}{1.489} & \multicolumn{1}{c}{35.00\%} & \multicolumn{1}{c}{} & \multicolumn{1}{c}{30.57\%} & \multicolumn{1}{c}{1.425} & \multicolumn{1}{c}{40.00\%} \\
Social\&Science & \multicolumn{1}{c}{39} & \multicolumn{1}{c}{1 047} & \multicolumn{1}{c}{20.95\%} & \multicolumn{1}{c}{1.274} & \multicolumn{1}{c}{56.41\%} & \multicolumn{1}{c}{} & \multicolumn{1}{c}{22.83\%} & \multicolumn{1}{c}{1.290} & \multicolumn{1}{c}{58.97\%} \\
Society\&Culture & \multicolumn{1}{c}{27} & \multicolumn{1}{c}{636} & \multicolumn{1}{c}{41.90\%} & \multicolumn{1}{c}{1.705} & \multicolumn{1}{c}{25.93\%} & \multicolumn{1}{c}{} & \multicolumn{1}{c}{39.97\%} & \multicolumn{1}{c}{1.684} & \multicolumn{1}{c}{22.22\%} \\
\hline
\end{tabular}
\end{center}\bigskip
\end{table}

\renewcommand{\baselinestretch}{1.4}
\small\normalsize

\begin{figure}[!t]
\centering
\includegraphics[width=14cm]{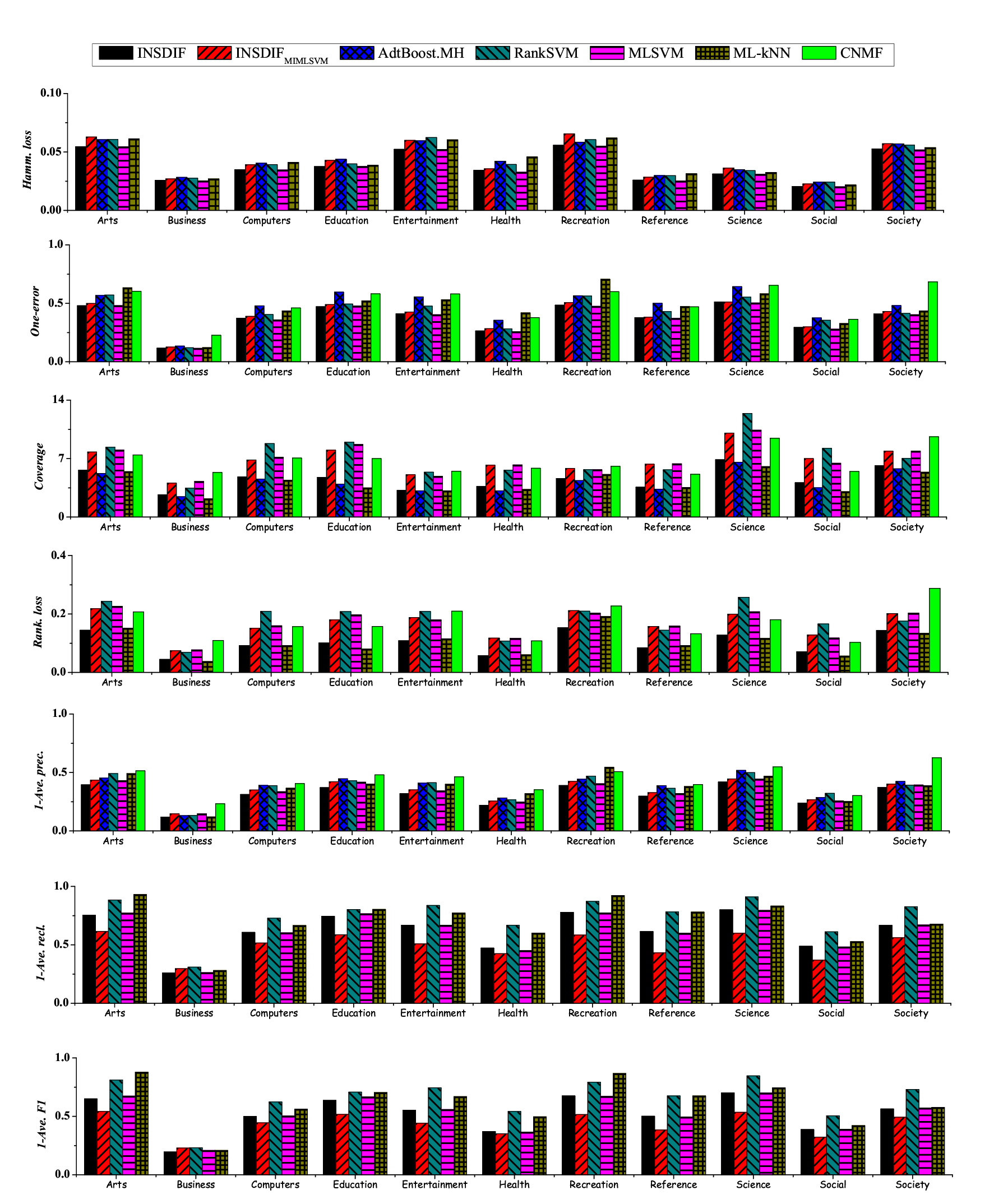}
\caption{Results on the eleven Yahoo data sets.}\bigskip\label{fig:11yahoo}\vspace{150pt}
\end{figure}

\clearpage

\bibliographystyle{plain}
\bibliography{miml}

\end{document}